\documentclass[11pt,letterpaper]{article}
\usepackage[margin=1in]{geometry}
\usepackage[utf8]{inputenc}
\usepackage[T1]{fontenc} %
\usepackage{times}
\usepackage{cancel}
\usepackage{mathtools}
\usepackage{mathpazo}

\usepackage{manfnt} %
\usepackage{figchild}
\usepackage{fontawesome5}

\usepackage{color-edits} 

\usepackage{accents}

\usepackage{xcolor} 
\usepackage{color} 
\definecolor{niceRed}{RGB}{190,38,38}
\definecolor{niceYellow}{HTML}{f5b400}
\definecolor{blueGrotto}{HTML}{059DC0}
\definecolor{royalBlue}{HTML}{057DCD}
\definecolor{navyBlue}{HTML}{0B579C}
\definecolor{yaleBlue}{HTML}{00356b}
\definecolor{limeGreen}{HTML}{81B622}
\definecolor{nicePurple}{HTML}{9c27b0}
\definecolor{lightRoyalBlue}{HTML}{def2ff}  
\definecolor{gold}{HTML}{ffa300}

\usepackage{color-edits} 
\addauthor[Steve]{sh}{Orchid}
\addauthor[Anay]{am}{gold}
\addauthor[Grigoris]{gv}{blue}
\addauthor[Manolis]{mz}{teal}

\usepackage{hyperref}
\hypersetup{
  colorlinks = true, 
  urlcolor = {blueGrotto},
  linkcolor = {royalBlue},
  citecolor = {navyBlue}
}
\usepackage[
  backref=true,
  backend=biber,
  natbib=true,
  style=alphabetic,
  sorting=alphabeticlabel,
  sortcites=true,
  minbibnames=3,
  maxbibnames=999,
  mincitenames=1,
  maxcitenames=4,
  minalphanames=1,
  maxalphanames=6,
  doi=false, 
  eprint=false, isbn=false, 
]{biblatex}

\AtEveryBibitem{%
  \clearname{editor}%
}
\AtEveryBibitem{%
  \clearlist{location}%
} 

\DeclareSortingTemplate{alphabeticlabel}{
  \sort[final]{%
    \field{labelalpha} %
  }
  \sort{%
    \field{year}   %
  }
  \sort{%
    \field{title}  %
  }
}

\AtBeginRefsection{\GenRefcontextData{sorting=ynt}}
\AtEveryCite{\localrefcontext[sorting=ynt]}
\usepackage[most]{tcolorbox}

\tcbset{
    mybox/.style={
        colframe=black,        %
        colback=gray!0,       %
        boxrule=0.5mm,         %
        width=\linewidth,      %
        arc=3mm,               %
        left=5mm,              %
        right=5mm,             %
        top=5mm,               %
        bottom=5mm,            %
        enhanced,              %
        halign=center        %
    }
}

\usepackage{booktabs} 
\usepackage{multicol} 
\usepackage{multirow} 
\usepackage{makecell} 
\usepackage{longtable}

\usepackage{graphicx}
\usepackage{float} 
\usepackage{tikz}
\usepackage{pgfplots}
\pgfplotsset{compat=1.17}
\usepgfplotslibrary{fillbetween}
\usetikzlibrary{arrows.meta}
\usepackage{setspace}

\tikzset{
  myNodeFlex/.style={
    draw,
    rectangle,
    rounded corners,
    text centered,
    minimum height=1.5em,
  }
}

\tikzset{
  myNode/.style={
    draw,
    rectangle,
    rounded corners,
    text centered,
    minimum height=1.5em,
    minimum width=3cm,
    text width=5cm,    
  }
}

\tikzset{
  myNodeNarrow/.style={
    draw,
    rectangle,
    rounded corners,
    text centered,
    minimum height=1.5em,
    minimum width=1cm,
  }
}

\tikzset{
  myNodeWide/.style={
    draw,
    rectangle,
    rounded corners,
    text centered,
    minimum height=1.5em,
    minimum width=6cm,
  }
}

\usepackage{xinttools} %
\usetikzlibrary{calc}

\def\biglen{20cm} %
\tikzset{
  half plane/.style={ to path={
       ($(\tikztostart)!.5!(\tikztotarget)!#1!(\tikztotarget)!\biglen!90:(\tikztotarget)$)
    -- ($(\tikztostart)!.5!(\tikztotarget)!#1!(\tikztotarget)!\biglen!-90:(\tikztotarget)$)
    -- ([turn]0,2*\biglen) -- ([turn]0,2*\biglen) -- cycle}},
  half plane/.default={1pt}
}

\usepackage{physics} %
\usepackage{amsmath}
\usepackage{amssymb}
\usepackage{amsthm}
\usepackage{bbm}
\usepackage{blkarray}
\usepackage{mathtools}
\usepackage{nicefrac}
\usepackage{dsfont}
\usepackage{xspace}
\usepackage{pifont} %

\usepackage{thm-restate} %
\usepackage[capitalise,noabbrev,nameinlink,sort]{cleveref}

\usepackage{enumerate} 
\usepackage[inline,shortlabels]{enumitem} 
\usepackage{typed-checklist}
\usepackage{tcolorbox}
\usepackage[framemethod=TikZ]{mdframed}
\mdfsetup{%
backgroundcolor=white, %
roundcorner=4pt,
linewidth=1pt}
\usepackage{changepage} %

\usepackage{algorithm}
\usepackage{algpseudocode}[1]

\usepackage{mathrsfs} %
\usepackage{soul} %

\theoremstyle{plain} 
\newtheorem{theorem}{Theorem}[section]

\newtheorem{lemma}[theorem]{Lemma}
\newtheorem{fact}[theorem]{Fact}

\newtheorem{claim}[theorem]{Claim}

\newtheorem{definition}{Definition}

\newtheorem*{definition*}{Definition}

\theoremstyle{definition} 

\newtheorem{remark}[theorem]{Remark}

\theoremstyle{remark}

\AfterEndEnvironment{definition}{\noindent\ignorespaces}
\AfterEndEnvironment{infdefinition}{\noindent\ignorespaces}
\AfterEndEnvironment{infcorollary}{\noindent\ignorespaces}
\AfterEndEnvironment{example}{\noindent\ignorespaces}
\AfterEndEnvironment{assumption}{\noindent\ignorespaces}
\AfterEndEnvironment{lemma}{\noindent\ignorespaces}
\AfterEndEnvironment{theorem}{\noindent\ignorespaces}
\AfterEndEnvironment{proposition}{\noindent\ignorespaces}
\AfterEndEnvironment{fact}{\noindent\ignorespaces}
\AfterEndEnvironment{question}{\noindent\ignorespaces}
\AfterEndEnvironment{corollary}{\noindent\ignorespaces}
\AfterEndEnvironment{model}{\noindent\ignorespaces}
\AfterEndEnvironment{remark}{\noindent\ignorespaces}
\AfterEndEnvironment{proof}{\noindent\ignorespaces}
\AfterEndEnvironment{fact}{\noindent\ignorespaces}
\AfterEndEnvironment{minftheorem}{\noindent\ignorespaces}
\AfterEndEnvironment{inftheorem}{\noindent\ignorespaces}
\AfterEndEnvironment{maintheorem}{\noindent\ignorespaces}
\AfterEndEnvironment{restatable}{\noindent\ignorespaces}
\AfterEndEnvironment{infassumption}{\noindent\ignorespaces}
\AfterEndEnvironment{conjecture}{\noindent\ignorespaces}
\AfterEndEnvironment{claim}{\noindent\ignorespaces}

\crefname{section}{Section}{Sections}
\crefname{theorem}{Theorem}{Theorems}
\crefname{lemma}{Lemma}{Lemmas}
\crefname{definition}{Definition}{Definitions}
\crefname{infdefinition}{Informal Definition}{Informal Definitions}
\crefname{conjecture}{Conjecture}{Conjectures}
\crefname{corollary}{Corollary}{Corollaries}
\crefname{infcorollary}{Informal Corollary}{Informal Corollaries}
\crefname{construction}{Construction}{Constructions}
\crefname{conjecture}{Conjecture}{Conjectures}
\crefname{claim}{Claim}{Claims}
\crefname{observation}{Observation}{Observations}
\crefname{proposition}{Proposition}{Propositions}
\crefname{fact}{Fact}{Facts}
\crefname{question}{Question}{Questions}
\crefname{problem}{Problem}{Problems}
\crefname{remark}{Remark}{Remarks}
\crefname{model}{Model}{Models}
\crefname{example}{Example}{Examples}
\crefname{equation}{Equation}{Equations}
\crefname{appendix}{Appendix}{Appendices}
\crefname{algorithm}{Algorithm}{Algorithms}
\crefname{model}{Model}{Models}
\crefname{figure}{Figure}{Figures}
\crefname{inftheorem}{Informal Theorem}{Informal Theorems}
\crefname{infassumption}{Informal Assumption}{Informal Assumptions}
\crefname{minftheorem}{Main Informal Theorem}{Main Informal Theorems}
\crefname{maintheorem}{Main Theorem}{Main Theorems}
\crefname{assumption}{Assumption}{Assumptions}
\crefname{case}{Case}{Cases}
\crefname{program}{Program}{Programs}
\crefname{inequality}{Inequality}{Inequalities}

\newlist{asmpenum}{enumerate}{1} %
\setlist[asmpenum]{label={\arabic*.},ref=\theassumption.{\arabic*}}
\crefname{asmpenumi}{Assumption}{Assumptions}

\usepackage{etoolbox}

\newcommand{\yesnum}{\addtocounter{equation}{1}\tag{\theequation}}  

\makeatletter
\renewcommand{\eqref}[1]{\textup{\eqrefform@{\ref{#1}}}}
\let\eqrefform@\tagform@

\newcommand{\changetag}[1]{%
  \renewcommand\tagform@[1]{\maketag@@@{(\ignorespaces#1\unskip\@@italiccorr)}}%
}
\makeatother

\makeatletter
\newcommand{\tagnum}[2]{%
    \refstepcounter{equation}%
    \tag{#1) \ (\theequation}%
    \protected@write \@auxout {}{%
        \string \newlabel {#2}{{\theequation}{\thepage}{}{equation.\theequation}{}}%
    }%
}
\makeatother

\newcommand{\quadtext}[1]{\quad\text{#1}\quad}
\newcommand{\qquadtext}[1]{\qquad\text{#1}\qquad}
 
\newcommand{\quadand}{\quadtext{and}}
\newcommand{\qquadand}{\qquadtext{and}}

\def\abs#1{\left| #1 \right|}

\def\sabs#1{| #1 |}

\newcommand{\sinparen}[1]{\ensuremath{(#1)}}
\newcommand{\sinbrace}[1]{\ensuremath{\{#1\}}}

\newcommand{\inbrace}[1]{\ensuremath{\left\{#1\right\}}}

\newcommand{\inparen}[1]{\ensuremath{\left(#1\right)}}
\newcommand{\insquare}[1]{\ensuremath{\left[#1\right]}}

\newcommand{\N}{\mathbb{N}}
\newcommand{\R}{\mathbb{R}}

\newcommand{\Z}{\mathbb{Z}}

\newcommand{\evE}{\ensuremath{\mathscr{E}}}

\newcommand{\evG}{\ensuremath{\mathscr{G}}}

\newcommand{\E}{\operatornamewithlimits{\mathbb{E}}}

\newcommand\ind{\mathds{1}}

\newcommand{\zo}{\ensuremath{\inbrace{0, 1}}}

\newcommand{\sfrac}[2]{{#1/#2}} 
\newcommand{\nfrac}[2]{\nicefrac{#1}{#2}}

\newcommand{\poly}{\mathrm{poly}}
\newcommand{\polylog}{\mathrm{polylog}}

\newcommand{\supp}{\operatorname{supp}}

\newcommand{\eps}{\varepsilon}
\renewcommand{\epsilon}{\varepsilon}
\makeatletter
\newcommand*{\tran}{{\mathpalette\@tran{}}}
\newcommand*{\@tran}[2]{\raisebox{\depth}{$\m@th#1\intercal$}}
\makeatother

\mathchardef\NABLA"272
\newcommand*{\Nabla}{\boldsymbol\NABLA}
\let\nabla\Nabla

\newcommand{\wh}[1]{\widehat{#1}}

\newcommand{\wt}[1]{\widetilde{#1}}

\newcommand{\customcal}[1]{\euscr{#1}}

\newcommand{\cB}{\customcal{B}}
\newcommand{\cC}{\customcal{C}}
\newcommand{\cD}{\customcal{D}}

\newcommand{\cF}{\customcal{F}} 
\newcommand{\cG}{\customcal{G}}

\newcommand{\cL}{\customcal{L}}

\newcommand{\cP}{\customcal{P}}

\newcommand{\cS}{\customcal{S}}

\newcommand{\cV}{\customcal{V}} 

\newcommand{\cX}{\customcal{X}}

\DeclareMathAlphabet{\mathdutchcal}{U}{dutchcal}{m}{n}
\SetMathAlphabet{\mathdutchcal}{bold}{U}{dutchcal}{b}{n}
\DeclareMathAlphabet{\mathdutchbcal}{U}{dutchcal}{b}{n}

\DeclareMathAlphabet\urwscr{U}{urwchancal}{b}{n}%
\DeclareMathAlphabet\rsfscr{U}{rsfso}{m}{n}
\DeclareMathAlphabet\euscr{U}{eus}{m}{n}
\DeclareFontEncoding{LS2}{}{}
\DeclareFontSubstitution{LS2}{stix}{m}{n}
\DeclareMathAlphabet\stixcal{LS2}{stixcal}{m} {n}

\renewcommand{\paragraph}[1]{\medskip \noindent\textbf{#1}~}

\newcommand{\eg}{\emph{e.g.}}
\newcommand{\ie}{\emph{i.e.}}

\usepackage{upgreek}
\renewcommand{\gamma}{\upgamma}
\renewcommand{\pi}{\uppi}

\newcommand{\eat}[1]{}
\newcommand{\negLL}{\ensuremath{\mathscr{L}}}

\newcommand{\hypo}[1]{\mathdutchcal{#1}}

\newcommand{\hyC}{\hypo{C}}

\newcommand{\hyF}{\hypo{F}}
\newcommand{\hyG}{\hypo{G}}
\newcommand{\hyH}{\hypo{H}}

\newcommand{\hyT}{\hypo{T}}

\usepackage{filecontents}
\usepackage{caption}
\usepackage{subcaption}
\usepackage{graphicx}

\renewcommand{\cL}{\negLL}

\usepackage{array}
\newcolumntype{L}[1]{>{\raggedright\let\newline\\\arraybackslash\hspace{0pt}}m{#1}}
\newcolumntype{C}[1]{>{\centering\let\newline\\\arraybackslash\hspace{0pt}}m{#1}}
\newcolumntype{R}[1]{>{\raggedleft\let\newline\\\arraybackslash\hspace{0pt}}m{#1}}

\usepackage{fancyhdr}

\fancypagestyle{custom}{
  \fancyhf{}                %
  \fancyfoot[C]{\textcolor{black}{\fcFishD{0.025}{black}{0.5}}} %

}

\newcommand{\vc}{\textrm{\rm VC}}
\newcommand{\VC}{\vc}
\newcommand{\oneStar}{\ensuremath{\mathfrak{s}_1}}

\newcommand{\hstar}{h^\star}
\newcommand{\Hstar}{H^\star}

\newcommand{\one}{\mathds{1}}

\newcommand{\V}{\mathsf{V}}
\newcommand{\closure}{\mathsf{Closure}}
\newcommand{\conv}{\mathsf{conv}}
\newcommand{\aff}{\mathsf{aff}}

\title{What is Learnable in Valiant's Theory of the Learnable?}

\author{
        \begin{tabular}{C{7.5cm}C{7.5cm}}
        {\bf Steve Hanneke}
            & {\bf Anay Mehrotra}\\
        {Purdue University} 
            & {Stanford University}\\
        \mbox{{\href{mailto:steve.hanneke@gmail.com }{steve.hanneke@gmail.com}}} 
            &   \mbox{{\href{mailto:anaymehrotra1@gmail.com}{anaymehrotra1@gmail.com}}}\\[4mm]
        {\bf Grigoris Velegkas}
            & {\bf Manolis Zampetakis}\\
                {Google Research} 
                    & {Yale University}\\
        \mbox{{\href{mailto:gvelegkas@google.com}{gvelegkas@google.com}}}
            & 
                \mbox{{\href{mailto:manolis.zampetakis@yale.edu}{manolis.zampetakis@yale.edu}}} 
        \end{tabular}
}
\date{}

\begin{document}

\maketitle

\begin{abstract}
Valiant's seminal 1984 paper \citep[Commun.\ ACM]{valiant1984theory} is widely credited with introducing the PAC learning model, but, in fact, it introduced a different model: while a PAC learner observes both positive and negative examples and may err on both sides, Valiant's learner receives only positive examples, may issue membership queries, and must produce a hypothesis with no false positives.
Several variants of this model have been studied and characterized: for instance, \citet[STOC]{natarajan1987learning}, \citet[Machine Learning]{shvaytser1990necessary}, and \citet[Math.~Systems Theory]{kivinen1995one} characterized the special case of Valiant's model without membership queries.
We revisit the original model and ask: \textit{Which hypothesis classes are learnable in it?}

For every finite domain, and in particular for Valiant's original Boolean-hypercube setting (where learnability requires polynomial dependence on the ambient dimension), we show:
a concept class is learnable if and only if every realizable positive sample can be ``certified'' by a polynomial-size adaptive query-compression scheme.
This is a new variant of sample compression where the learner certifies samples via a short interaction with the membership oracle instead of the usual non-interactive compression schemes that characterize PAC learning.
Our characterization, along with simple examples, shows that learnability in Valiant's model is strictly sandwiched between PAC learnability and the variant without membership queries.
This is one of the rare cases where introducing membership queries changes the set of learnable classes (and not just the sample or computational complexity of learning).
Further, our characterization bounds the sample complexity of learning within polynomial factors.

Next, we study the natural extension of the model to arbitrary domains. 
While we do not obtain an exact characterization here, our techniques readily generalize and show that learnability remains strictly sandwiched between PAC learning and the variant without queries.
Finally, we show that, while halfspaces over $\R^d$ are not learnable without queries (for any $d\geq 2$), they become learnable with queries: to show this, we give a $\poly(d)\cdot \wt{O}(\sfrac{(\log{\nfrac1\delta})}{\eps})$ sample and $\poly(d)\cdot \polylog(\nfrac{1}{\eps\delta})$ query algorithm for learning halfspaces and show that at least $\Omega(d)$ samples or queries are necessary.
To the best of our knowledge, this is the first algorithm to learn halfspaces in Valiant's model; previous algorithms could only learn halfspaces with finite bit-complexity, which reduced the problem to learning a finite class.

Together, these results uncover a surprisingly rich theory behind Valiant's original notion of learnability and introduce ideas that may be of independent interest in learning theory.
\end{abstract}
 
\thispagestyle{empty} 

\newpage

\thispagestyle{empty} 

{
    \setstretch{1.0}
    \tableofcontents 
}

\thispagestyle{empty} 

\newpage 

\pagenumbering{arabic}

\section{Introduction}
Valiant's seminal 1984 paper \citep{valiant1984theory} is widely credited with introducing the PAC learning model, but the learning model it actually formulates is different.
In particular, a PAC learner observes both positive and negative examples and may make mistakes on either side (that is, it may make both false positive and false negative errors, provided this happens with probability at most $\eps$).
A learner in Valiant's model receives only positive examples (\ie{}, samples drawn from a distribution supported entirely on the positive region of an unknown target $\hstar$), may issue membership queries of the form ``is $\hstar(x)=1$?'' for any $x$ in the domain, and must output a hypothesis that introduces no false positives while covering all but an $\eps$-fraction of the positive distribution.

Valiant defined learnability in this model for a sequence of concept classes $\hyH=\inparen{\hyH_1,\hyH_2,\dots}$, where each $\hyH_d$ is a class of Boolean functions on $\zo^d$, as follows.

\begin{definition}[Valiant's Learning Model \citep{valiant1984theory}]
    \label{def:valiant}
    A sequence of concept classes $\hyH=\inparen{\hyH_1,\hyH_2,\dots}$ is \emph{learnable in Valiant's model} if there exists a learner $\cL$ such that, for every $\eps,\delta\in(0,1)$, $d\in \N$, target hypothesis $\hstar\in \hyH_d$, and distribution $\cD$ satisfying $\supp(\cD)\subseteq \supp(\hstar)$, the learner uses at most
    \[
        m(d,\eps,\delta)=\poly\!\inparen{d,\nfrac{1}{\eps},\nfrac{1}{\delta}}~~\text{positive examples}~~\text{and}~~\text{membership queries,}
    \]
    and outputs a hypothesis $h$, such that, with probability at least $1-\delta$,
    \[
        \supp(h)\subseteq \supp(\hstar)
        \qquad\text{and}\qquad
        \Pr\nolimits_{x\sim\cD}\inparen{x\notin \supp(h)}\leq \eps.
    \]
\end{definition}
Compared to PAC learning, this model differs in three ways.
First, the learner sees only positive data.
Second, the guarantee is one-sided: the learner may miss an $\eps$-fraction of the positive distribution, but may not cover any elements outside of $h^\star$.
Third, the learner may interact with the target through membership queries.\footnote{If there are multiple hypotheses consistent with $\cD$, all answers of the oracle are consistent with at least one of them.}
The requirement that $m(d,\eps,\delta)$ be polynomial in $d$ is also essential as the domain $\zo^d$ is finite, and so one can trivially learn any target with $2^d$ queries.
Hence, if $m(d,\eps,\delta)$ is not required to be polynomial, then the model becomes vacuous.
Finally, the learner's membership queries may be adaptive: each query made by the learner may depend on the responses to their previous queries and any positive examples they have seen. Equivalently, the learner may traverse a binary query tree whose branches correspond to oracle responses.
Each of these three ingredients has independent motivation: %
\begin{enumerate}[leftmargin=15pt]
    \item The positive-only viewpoint has long been recognized as natural.
    Already in related early work \citep{valiant1984deductive}, Valiant defended his definition of learning compared to two-sided versions, such as PAC, by noting that ``to discuss the distribution of the attributes of elephants, we may prefer not discussing the distribution of the attributes of non-elephants.''
    This perspective is even more compelling in modern settings motivated by generative models, where one typically observes valid text during training but there is no canonical distribution over invalid outputs \citep{kleinberg2024language,kalavasis2025limitslanguagegenerationtradeoffs}.
    The idea itself dates back to Gold's seminal model of identification in the limit \citep{gold1967language}, where the learner is presented with positive examples but no negative examples.
    More broadly, there is a long and ongoing body of work on learning from positive-only examples, \eg{}, \citep{frieze1996linearTransformations,Nguyen2009learningParallelepiped,anderson2013simplices,de2015satisfying,Kontonis2019EfficientTS,canonne2020satisfying,he2023robustMoments,lee2024unknown,lee2026smoothed,mansouri2025learning}.
    \item The second ingredient, the one-sided guarantee, is important in many applications where false positives are substantially more costly than false negatives, so a conservative hypothesis is preferable; see \citep{elkan2008pu,li2010positive}.
    One-sided guarantees are also important in missing data problems, \eg{}, truncated statistics \citep{daskalakis2018efficient,Kontonis2019EfficientTS,lee2024unknown}, where without one-sided guarantees many classical estimators, like maximum-likelihood, are not well-defined \citep{lee2024unknown}.
    This also naturally arises in generative modeling, where one-sided error is required to avoid hallucinations \citep{hanneke2018actively,kleinberg2024language,kalavasis2025limitslanguagegenerationtradeoffs,charikar2024facets}.
    
    \item The third ingredient is membership queries, which allow the learner to acquire information about points that never appear in the positive sample. More broadly, they place the model within the long tradition of interactive and query-based learning.
    Since Angluin's seminal work, membership queries and related query mechanisms have been studied extensively in exact learning, query learning, and active learning \citep{angluin1988query,kulkarni1993active,angluin2004queries,balcan2006agnostic,hanneke2007bound,balcan2010true,bshouty2018exact,hanneke2025agnostic};  we refer the reader to \cite{hanneke2009theoretical,settles2012active,aggarwal2014active} for an overview of active learning.
\end{enumerate}
\noindent Over the years, learnability \emph{has} been characterized in several neighboring models that modify one or more of these ingredients: the positive-only sample model \citep{natarajan1987learning,shvaytser1990necessary,kivinen1995one}, the standard PAC setting and its query variants \citep{blumer1989learnability,eisenberg1990on}, and models with stronger or different query interfaces \citep{kulkarni1993active,angluin1988query,hayashi2003uniform,angluin2004queries,bshouty2018exact}.
These results clarify the landscape around the original model, yet the question that \citet{valiant1984theory} posed, which combined all three ingredients, has resisted a complete answer:

\smallskip

\begin{center}
    \emph{Which sequences of concept classes are learnable in Valiant's original model?}
\end{center} 
 
\subsection{Our Results}

\subsubsection{Characterization of Learnability in Valiant's Model}
Our first result answers the above question by giving a complete characterization of learnability in Valiant's original model.
 
A natural starting point is the existing characterizations of PAC learnability in binary classification.
There are two standard ways to characterize binary classification in PAC learning.
One is through finiteness of the VC dimension~\citep{vapnik1974theory,vapnik1971uniform,blumer1989learnability} and the other is via the existence of bounded-size sample compression schemes~\citep{LittlestoneWarmuth1986,floyd1995sample,moran2016sample}.
The VC-dimension viewpoint bounds the complexity of a hypothesis class through the labelings it can realize.
The compression viewpoint asks whether datasets are compressible: given an arbitrary labeled sample consistent with some $h\in\hyH$, can the learner retain only $k$ examples (and a small number of additional bits) from which the labels of the entire sample can be recovered?

The second viewpoint turns out to be more fruitful in Valiant's model.
Because the learner may ask membership queries, the natural analogue is not an ordinary ``static'' compression scheme but an adaptive one.
This leads naturally to an interactive notion of compression, in which the certificate is not just a subsample, but an adaptive query transcript whose answers force the observed sample to be positive for every hypothesis consistent with the resulting transcript.

To formulate this notion, we first introduce version spaces:
For a finite set of positive examples $S\subseteq \zo^d$, let $\V\!\inparen{S}
    \coloneqq
    \inbrace{h\in \hyH_d : S\subseteq \supp(h)}$
be the set of hypotheses in $\hyH_d$ consistent with $S$.
Likewise, if $T$ is a query-response transcript, let $\V(T)$ denote the set of hypotheses in $\hyH_d$ consistent with that transcript.
If $\Sigma$ is an adaptive query strategy and $r$ is a realized response sequence, we write $\V\!\inparen{\Sigma,r}$ for the version space defined by the corresponding realized root-to-leaf transcript.

\begin{definition}[Adaptive-Query Compression Scheme]
    \label{def:query-compression}
    Fix $d\in\N$ and a set $S\subseteq \zo^d$ with $\V\!\inparen{S}\neq\emptyset$.
    A deterministic adaptive query strategy $\Sigma_S$ is an \emph{adaptive-query compression scheme} for $S$ if, for every response transcript $r$ of $\Sigma_S$ that is realizable with respect to $\V\!\inparen{S}$,\footnote{Here, ``realizable with respect to $\V\!\inparen{S}$'' means that the transcript is produced by $\Sigma_S$ on some hypothesis in $\V\!\inparen{S}$.}
    \[
        S
        \subseteq
        \bigcap_{h\in \V\!\inparen{\Sigma_S,r}} \supp(h).
    \]
    The \emph{size} of the scheme is the depth of $\Sigma_S$, \ie{}, the maximum number of queries on any root-to-leaf path.
\end{definition}
For brevity, we will also refer to an adaptive-query compression scheme simply as a \emph{query compression scheme}.
The role of the transcript is to \emph{certify} the data $S$: once the responses are known, every hypothesis consistent with the transcript must already label every point of $S$ as positive, so $S$ carries no further information.
Note that this compression scheme retains no explicit subsample of $S$. This is without loss of generality, as the learner has membership queries and
if they wish to recover the labels of $k$ sample points, they can simply query them directly. %
Further, because the learner only observes positive examples, a standard reconstruction map would make compression vacuous: one could always ``reconstruct'' the sample by simply labeling everything positive.
The transcript must therefore certify the sample through the version space itself (\cref{rem:comparison-of-compression}).

There are two key differences between standard compression schemes and the scheme above, which make the two incomparable.
On the one hand, query compression schemes are more general: in addition to the information already present in the sample, they can use additional information obtained through membership queries.
On the other hand, query compression schemes are more restrictive in how reconstruction works (\cref{rem:comparison-of-compression}).

\smallskip 

Our characterization is based on query compression schemes and is as follows:
\begin{restatable}[Characterization]{theorem}{thmCharacterization}
\label{thm:characterization}
A sequence of concept classes $\hyH=\inparen{\hyH_1,\hyH_2,\dots}$ is learnable in Valiant's model if and only if every realizable positive sample $S\subseteq \zo^d$ admits an adaptive-query compression scheme of size $\poly\!\inparen{d}$.
\end{restatable}
In other words, learnability in Valiant's model is exactly the ability to replace an arbitrary realizable positive sample by a short interactive certificate.
Moreover, the proof is quantitative: %
\begin{restatable}{corollary}{thmCharacterizationQuantitative}
\label{thm:characterizationQuantitative}
Fix a sequence of classes $\hyH=\inparen{\hyH_1,\hyH_2,\dots}$.
The following hold in Valiant's model:
\vspace{-2mm}
\begin{enumerate}[itemsep=0pt,leftmargin=16pt]
    \item If $\hyH$ admits an adaptive-query compression scheme of size $q(d)$ for each realizable positive sample $S\subseteq \zo^d$, then $\hyH$ is learnable with sample/query complexity $m(d,\eps,\delta)\leq q(d)\cdot \wt{O}\!\inparen{(\nfrac{d}{\eps})\, \log{\nfrac1\delta}}$.
    \item Conversely, if $\hyH$ is learnable with sample/query complexity $m(d,\eps,\delta)$, then every finite realizable positive sample $S\subseteq \zo^d$ admits an adaptive-query compression scheme of size $O\sinparen{m\!\inparen{d,\nfrac{1}{20},\nfrac{1}{20}}^2\cdot d^2}$.
\end{enumerate}
\end{restatable}
Thus, adaptive-query compression captures the sample and query complexity of Valiant's model up to polynomial factors.
Here, the adaptivity of the queries is also essential.
The learner does not know in advance which branch of the query tree will be realized, since this depends on the unknown target (\cref{rem:adaptivity}).
The characterization extends naturally beyond the Boolean hypercube to any finite domain of size $N$, with the requirement that $m$ be polynomial in $\log{N}$ instead of $d$ (\cref{rem:finite-characterization}). 
The proof of \cref{thm:characterization,thm:characterizationQuantitative} appears in \cref{sec:proofof:thm:characterization}.

\subsubsection{Some Implications of \cref{thm:characterization,thm:characterizationQuantitative}}
A first benchmark for our characterization is the special case of Valiant's model in which membership queries are not allowed and the learner must work only with the observed positive sample.
As mentioned earlier, this positive sample-only variant was characterized by \citet{natarajan1987learning,shvaytser1990necessary} (with sharp sample complexity due to \citet{kivinen1995one}).
Concretely, $\hyH$ is learnable from positive examples alone if and only if, for each $d$, the intersection closure of $\hyH_d$, say $\overline{\hyH_d}$, has VC dimension $\poly\!\inparen{d}$ (\cref{thm:positive-characterization}).\footnote{The intersection closure of $\hyH_d$ is the smallest superclass of $\hyH_d$ closed under finite intersections.}
Our characterization allows us to recover one side of this characterization: if $\vc\!\inparen{\overline{\hyH_d}}=\poly\!\inparen{d}$, then $\hyH$ is learnable in the above model. In particular, we get this by combining our characterization with the following simple result:

\begin{restatable}[Positive-Only Learning as the Non-Interactive Special Case]{proposition}{positiveOnlySpecialCase}
\label{cor:intro:positiveOnlySpecialCase}
If, for each $d$, the intersection closure of $\hyH_d$ has VC dimension $\poly\!\inparen{d}$, then every finite realizable positive sample $S\subseteq \zo^d$ admits an adaptive-query compression scheme of size $\poly\!\inparen{d}$.
Moreover, this scheme can be chosen non-adaptively, and all of its queried points can be taken from $S$ itself.
\end{restatable}
Thus, whenever positive-only learning is possible without membership queries, the certificate in our characterization can already be extracted from the observed sample itself: no adaptivity or extra queries outside the sample are needed.
In this sense, the classical characterization of positive-only learning is a non-interactive special case of our characterization.
The proof of \cref{cor:intro:positiveOnlySpecialCase} appears in \cref{sec:proofof:cor:intro:positiveOnlySpecialCase}.

Another natural question is whether learnability in Valiant's model coincides with one of the two familiar extremes: PAC learning (governed by $\vc\!\inparen{\hyH_d}$) or positive-only learning without membership queries (governed by the VC dimension of the intersection closure of $\hyH_d$).
(Over the Boolean domain, the standard PAC learning model also requires the sample complexity to be $\poly(d)$ and is characterized by $\vc\!\inparen{\hyH_d}=\poly(d)$; \cite{kearns1994introduction}.)
The following corollary shows that Valiant's model does not collapse to either extreme.

\begin{restatable}[Sandwich between PAC and Positive-Only Learning]{corollary}{implicationscor}
    \label{thm:implications}
    Fix $\hyH$ and its intersection closure $\overline{\hyH}$.
    The following hold in Valiant's model.
    \begin{enumerate}[itemsep=0pt,leftmargin=16pt]
        \item If $\vc\!\inparen{\hyH_d}=d^{\omega(1)}$, then $\hyH$ is not learnable.
        \item If, for infinitely many $d$, $\overline{\hyH}_d$ contains every subset of $\zo^d$, then $\hyH$ is not learnable.
        In particular, there exists $\hyH$ with $\vc\!\inparen{\hyH_d}=1$ for each $d\in\N$ that is not learnable.
        \item There exists a sequence of classes that is learnable even though, for infinitely many $d$, $\vc\!\inparen{\overline{\hyH}_d}=d^{\omega(1)}$.
    \end{enumerate}
\end{restatable}
Taken together, these three items show that Valiant's model lies strictly between PAC learning and positive-only learning without membership queries.
Item~1 shows that Valiant's model is no more permissive than PAC learning, while Item~2 shows that this containment is strict.
In the other direction, positive-only learning without membership queries is a special case of Valiant's model, since a learner with query access can always ignore it.
Item~3 shows that this inclusion is also strict.
Thus Valiant's model lies strictly between PAC learning and positive-only learning without membership queries.
This is one of the rare cases where introducing membership queries changes the set of learnable classes (and not just the sample or computational complexity of learning).
The proof of \cref{thm:implications} appears in \cref{sec:proofof:thm:implications}.

\subsubsection{Results for Extension to General Domains}

Next, we study the natural extension of Valiant's model to an arbitrary domain $\cX$.
As is standard over arbitrary domains, learnability is defined for a single concept class $\hyH\subseteq \zo^\cX$ rather than for a sequence indexed by dimension.

\begin{definition}[Valiant's Model over a General Domain]
    \label{def:valiant:general}
    Fix a domain $\cX$ and a concept class $\hyH\subseteq \zo^\cX$.
    We say that $\hyH$ is \emph{learnable} in Valiant's model over $\cX$ if there is a learner and $m:(0,1)^2 \rightarrow \N$ such that for every target hypothesis $\hstar\in \hyH$, every distribution $\cD$ over $\cX$ compatible with $\hstar$, and every $\eps,\delta\in(0,1)$, the learner, using $m(\eps,\delta)$ positive examples from $\cD$ and membership queries to $\hstar$, outputs with probability at least $1-\delta$ a hypothesis $h$ satisfying $\supp(h)\subseteq \supp(\hstar)$ and
    $\Pr\nolimits_{x\sim\cD}\inparen{x\notin \supp(h)}\leq \eps.$
\end{definition}
This is the same one-sided requirement as before; the only change is that the instance space is now arbitrary and learnability is defined for a fixed class rather than a dimension-indexed sequence.
While we do not obtain an exact characterization in this setting, we prove the following general-domain analogue of \cref{thm:implications}.

\begin{restatable}[Sandwich between PAC and Positive-Only Learning]{theorem}{theoremNecessarySufficient}
    \label{thm:intro:necessary-sufficient}
    Fix a concept class $\hyH\subseteq \zo^\cX$ and let $\overline{\hyH}$ denote its intersection closure.
    The following hold in Valiant's model over a general domain.
    \begin{enumerate}[itemsep=0pt,leftmargin=16pt]
        \item If $\hyH$ is learnable with $m\!\inparen{\eps,\delta}$ positive examples/membership queries, then $\vc\!\inparen{\hyH}\leq 9\,m\!\inparen{\nfrac{1}{3},\nfrac{1}{3}}+1.$
        \item If $\abs{\cX}=\infty$, $\hyH$ contains the constant-one hypothesis, and for every $x\in \cX$ it also contains a hypothesis that labels $x$ zero and every point in $\cX\setminus\inbrace{x}$ positive, then $\hyH$ is not learnable.\footnote{Equivalently, that $\hyH$ shatters the whole domain in a $1$-star sense; see \cref{def:one-star-number}.}
        In particular, there exists $\hyH$ with $\VC\inparen{\hyH}=1$ that is not learnable.
        \item If $\vc\!\inparen{\overline{\hyH}}<\infty$, then $\hyH$ is learnable with sample and query complexity $\wt{O}\!\inparen{
            (\nfrac{1}{\eps})\cdot
            \inparen{
                \vc\!\inparen{\overline{\hyH}}
                +
                \log\!\nfrac{1}{\delta}
            }
        }.$
    \end{enumerate}
\end{restatable}
This is the natural general-domain analogue of the picture from the Boolean setting.
Item~1 is familiar from PAC learning. %
Item~2 gives a sufficient condition for non-learnability that already shows the containment inside PAC learning is strict: the class consisting of the constant-one hypothesis together with the complements of singletons has VC dimension $1$, yet item~2 shows it is not learnable.
Item~3 is the sufficient condition for learnability via the intersection closure $\overline{\hyH}$, mirroring the classical positive-only characterization.
One might hope that, over arbitrary domains, one of the three criteria above could be turned into a characterization; we show that none of them is tight (\cref{thm:nonTightness}).
The proof of \cref{thm:intro:necessary-sufficient} appears in \cref{sec:proofof:thm:intro:necessary-sufficient}.

\paragraph{Learnability of Halfspaces in $d$-Dimensional Euclidean Space.}
A particularly compelling instance of the gap between Valiant's model and the version without membership queries is the class of halfspaces.
Halfspaces (or linear threshold functions) are among the most basic and extensively studied concept classes in learning theory \citep{shalev2014understanding,mohri2018foundations}, and their learnability is well understood in several different learning models.
In our setting they are especially revealing because, by the characterizations of \citet{natarajan1987learning,shvaytser1990necessary}, halfspaces over $\R^d$ (for $d\geq 2$) are not learnable from positive examples alone, and yet we show that they become learnable once membership queries are available.

\begin{restatable}[Halfspaces Are Learnable on $\R^d$]{theorem}{thmhalfspace}
\label{thm:halfspace}
There is an algorithm that learns halfspaces in $\R^d$ in Valiant's model using $m\!\inparen{\eps,\delta}$ positive examples and $q\!\inparen{\eps,\delta}$ membership queries, where
\[
        m\!\inparen{\eps,\delta}
        =
        O\!\inparen{
            \frac{
                d^2\log d\cdot \log\!\nfrac{1}{\eps}
                +
                \log\!\nfrac{1}{\delta}
            }{\eps}
        }
        \qquadand
        q\!\inparen{\eps,\delta}
        =
        O\!\inparen{
            d^3\log^3 d\cdot \log\!\nfrac{1}{\eps\delta}
        }.
    \]
\end{restatable}
To the best of our knowledge, this is the first algorithm for learning halfspaces in Valiant's model.
Moreover, using \cref{thm:intro:necessary-sufficient}, we can also conclude that the total sample/query complexity of the above model is optimal up to polynomial factors:
this is because halfspaces in $\R^d$ have VC dimension $d+1$, and, hence, item~1 of \cref{thm:intro:necessary-sufficient} implies that any learner requires
$\max\!\inbrace{
        m\!\inparen{\nfrac{1}{3},\nfrac{1}{3}},
        q\!\inparen{\nfrac{1}{3},\nfrac{1}{3}}
    }
    =
    \Omega(d)$
positive examples or membership queries.
The proof of \cref{thm:halfspace} appears in \cref{sec:proofof:thm:halfspace}.

It is natural to ask whether the same holds over the Boolean cube.
Perhaps surprisingly, restricting halfspaces in $\R^d$ to the Boolean cube $\zo^d$ makes the class \emph{not} learnable.
\begin{restatable}[Halfspaces Are Not Learnable on the Boolean Cube]{proposition}{thmbooleanhalfspace}
\label{thm:booleanHalfspace}
Let $\hyH^{\rm H}=\inparen{\hyH^{\rm H}_1,\hyH^{\rm H}_2,\dots}$, where $\hyH^{\rm H}_d$ is the class of restrictions to $\zo^d$ of halfspaces in $\R^d$.
Then $\hyH^{\rm H}$ is not learnable in Valiant's model (where the learner can only query points in $\zo^d$).
\end{restatable}
The reason is that halfspaces on $\zo^d$ can express every subset of the cube. The constant-one halfspace labels every vertex as positive, and for each vertex $v \in \zo^d$ there is a halfspace that labels every vertex except $v$ as positive. By intersecting these halfspaces one can obtain any arbitrary subset of the cube, so the impossibility result from Item~2 of \cref{thm:implications} applies. 
Over $\R^d$, by contrast, the learner can query points outside $\zo^d$, and the algorithm in \cref{thm:halfspace} exploits exactly this additional freedom. See \cref{sec:proofof:thm:booleanHalfspace} for details.

\subsection{Technical Overview}
    \label{sec:technicalOverview}
    In this section, we explain the key ideas and challenges in characterizing Valiant's model, and in learning halfspaces from positive examples and membership queries.

    \subsubsection{Challenges in Characterizing Valiant's model}
        First, we discuss challenges in using standard approaches for characterizing Valiant's model.

        Recall that Valiant's model differs from the PAC learning model in three ways: (C1) the learner must make only one-sided error, (C2) the learner only observes positive examples, and (C3) the learner has access to membership queries to the target hypothesis.

    \paragraph{Challenge I: ``One-Sided Covers'' Are Not Known for General Concept Classes.}
        A standard way of characterizing learnability in PAC learning is via the existence of uniform $\eps$-covers $\mathcal{C}(\hyH)$ for concept classes $\hyH$ with small VC dimensions, and then utilizing uniform convergence over the hypotheses in the $\eps$-cover.
        The resulting universal learning algorithm is then empirical risk minimization (ERM) over the $\eps$-cover $\mathcal{C}(\hyH)$.
        For such a learner to work in Valiant's model, it has to achieve one-sided error, that is, satisfy Condition C1. 
        To do this, we would need the cover to satisfy the following stronger ``one-sided'' guarantee:
        \[
            \forall_{h\in \hyH},~~ 
            \exists_{c\in \mathcal{C}(\hyH)},~~
            \text{s.t.},\qquad \supp(c)\subseteq \supp(h)
            \quadand \Pr_{x\sim \cD}(x\in \supp(h)\setminus\supp(c))\leq \eps\,.
        \]
        Let us call such a cover an $\eps$-containee.
        Similar covers have been studied by prior work, \eg{}, by \citet{dutta2022uniform,braverman2021near}. 
        In particular, these works introduce a concept of $\eps$-containers, which is equivalent to having $\eps$-containees for the complement class $\inbrace{\cX\setminus \supp(h): h\in \hyH}$. 
        However, unlike the standard $\eps$-covers, these $\eps$-containers and $\eps$-containees are much less well understood, and, for instance, we do not know a general characterization of which pairs of concept classes $\hyH$ and distributions $\cD$ admit them.
        Hence, any approach to characterize Valiant's model by utilizing covers would need to first understand when $\eps$-containees exist.
        Further, even if we are able to do this, it is still unclear whether this is the right approach because it is not clear how, given a $\eps$-containee, one would utilize membership queries for learning.
        As we have seen, membership queries do change what is learnable in Valiant's model, so we cannot simply ignore them. 

    \paragraph{Challenge II: No Known ``Universal'' Learning Algorithm.}
        Another idea is to try to take inspiration from \citet{natarajan1987learning,shvaytser1990necessary}'s characterization of the special case where the learner does not have access to membership queries.
        To get a characterization, they relied on the following simple observation: given positive samples $S$, any valid learner must output \textit{a subset} of the closure $\mathsf{Closure}(S)\coloneqq \bigcap_{h\in \V(S)}\supp(h)$.
        Since the closure always has one-sided error (as $\bigcap_{h\in \V(S)}\supp(h)\subseteq \supp(h')$ for any $h'\in \V(S)$), one can always convert a learning algorithm into one that outputs the closure of $S$.
        And now, since one-sided error is guaranteed and the learner does not have access to membership queries, it is sufficient to analyze when uniform convergence holds over the concept class of all closures: $\overline{\hyH}=\inbrace{\bigcap_{h\in \hyG}\supp(h) : \hyG\subseteq \hyH \text{~where $\abs{\hyG}$ is finite}}$.
        This is why the characterization in this special case turns out to be a bound on $\vc\!\inparen{\overline{\hyH}}$.
        In Valiant's model, however, the learner has access to membership queries and can use these queries to change the version space of hypotheses consistent with the observed transcript.
        Hence, to characterize Valiant's model, we need to understand when queries can be used to alter the version space so that its closure is ``well behaved,'' \eg{}, it has a small VC dimension, and this is an algorithmic puzzle we need to solve to prove \cref{thm:characterization}.

    \paragraph{Challenge III: Techniques for Characterizing Exact Learning Are Too ``Coarse.''}
        We saw in the previous paragraph that accounting for membership queries is an added challenge in characterizing Valiant's model. 
        Given this, a natural approach is to try to utilize techniques for characterizing classes that are learnable with membership queries. 
        There are a number of dimensions known to characterize classes that are learnable from different forms of queries, including membership queries, equivalence queries, and disjointness queries (among others) \citep{angluin1988query}.
        Essentially, all of these characterizations are for exact learning, where the goal is to identify the target concept from the hypothesis class exactly. 
        At a high level, these characterizations build a ``query tree'' over the hypothesis class, where each query splits the surviving hypotheses according to the response, and the goal is to isolate the target with as few steps as possible.
        The characterization then corresponds to the existence of a query tree of finite depth.
        This approach is insufficient for us for two reasons:
            First, because we see positive examples, given samples $S$, the hypothesis class for which we want to create a query tree is $\V(S)$ and not $\hyH$.
            So at the least, we would need to understand some necessary and sufficient structure on $\hyH$ such that with high probability the resulting version space $\V(S)$ has a small finite-depth query tree.
            Second, because our goal is not to do exact learning, it is not even clear whether a query tree is the right combinatorial objective: it is certainly sufficient for learning, but it may not be necessary, since we do allow the output to make errors on one side.

    \subsubsection{Overview of the Proof of Characterization}
        Our proof makes a sequence of deliberate choices to overcome the challenges described above:
        Rather than pursuing a characterization via a notion of covers, we work with sample compression schemes, which turn out to be better suited to the one-sided requirement in Valiant's model and thereby avoid Challenge~I.
        Our compression schemes are query based (\cref{def:query-compression}) and so, at first sight, resemble the ``query trees'' used in exact learning (see Challenge~III above).
        Indeed, both consist of finite adaptive sequences of queries.
        The key difference, however, is in the objective:
        while query trees for exact learning aim to isolate the target hypothesis $\hstar$, our goal is to certify that $S$ is a positive sample, in the sense that every hypothesis consistent with the realized query transcript must be positive on all of $S$ (\cref{def:query-compression}).
        This shift in perspective avoids Challenge~III.
        Finally, somewhat counterintuitively, this viewpoint also yields the right notion of a ``universal'' learning algorithm:
        whenever a concept class $\hyH$ is learnable in Valiant's model, there is a learner that certifies the observed sample using $\poly(d)$ queries and then outputs the closure of the resulting version space (\cref{thm:valiant:learningFromCompression}).
        This overcomes Challenge~II.

    In the remainder of this section, we give an outline of our proof, in three parts.
    \vspace{-4mm}
    \subsubsection*{Part 1 (Learnability implies Small Query Compression Scheme) (\cref{thm:valiant:compression})}   
        In this part, we show that if $\hyH$ is learnable, then $\hyH_d$ has a query compression scheme of size
        \[
            q_d\leq m_{\rm W}(d)\cdot O(d\,\vc(\hyH_d))\,.
        \]
        Here $m_{\rm W}(d)$ is the query/sample complexity of ``weak learning'' for $\hyH$ (\ie{}, learning with $\eps,\delta=\Omega(1)$).
        To establish the above upper bound, we first bound the size of a ``weak'' compression scheme (\cref{def:valiant:weak-compression}), which is only required to certify an $\Omega(1)$-fraction of the positive examples.
        Intuitively, given the existence of a weak compression scheme for each realizable set $S$ of positive samples, we can get a strong compression scheme by repeating this process $O(\log{|S|})$ times: in each round we certify a constant fraction of $S$, remove the certified samples, and recurse on the remainder (\cref{thm:valiant:compression}).
        Then we can deduce the result since $O(\log{|S|})\leq d$ as $S\subseteq \zo^d$.
        
        It remains to explain how to obtain such a weak compression scheme.
        Let $T$ be the entire transcript of interaction between the learner and the environment, including positive samples, queries, and responses.
        Let $\V(T)$ be the version space of hypotheses consistent with this transcript.

        \paragraph{Bound on weak compression-scheme size (assuming learner outputs a closure).}
            If we knew the learner outputs a closure over this version space, then it is simple to bound the weak compression-scheme size:
            We can simply run a weak learner on the uniform distribution over $S$ to produce a transcript $T$ of length $O(m_{\rm W}(d))$. Since it is a weak learner that outputs the closure over $\V(T)$, it must be the case that
            $\sabs{S\cap \bigcap_{h\in \V(T)}\supp(h)}\geq \Omega(|S|)$.
            Thus, we can obtain a weak compression scheme by compiling this realized run into an adaptive query strategy: every sampled point in $T$ is turned into a membership query, and we also keep the learner's ordinary membership queries.

        \paragraph{Bound on weak compression-scheme size (without assumptions on the learner).}
            If the learner's output can be non-closure, then the argument is much more challenging because we can no longer assume that $\sabs{S\cap \bigcap_{h\in \V(T)}\supp(h)}\geq \Omega(|S|)$.
            To address this, we show that whenever $\hyH$ is learnable, there is a learner that succeeds \emph{uniformly} over all potential targets consistent with its realized transcript, namely, all hypotheses in $\V(T)$.
            Concretely, in the (ordinary) Valiant model, for each fixed target $h\in \V(T)$ the learner may fail with probability $\delta$, but the failure event may depend on $h$.
            In the uniform version, with probability at least $1-\delta$, the learner must succeed simultaneously for \emph{every} $h\in \V(T)$.
            The analogous strengthening of PAC learning is known as Probably Uniformly Approximately Correct (PUAC) learning and learnability in it is known to be equivalent to learnability in PAC learning \citep[see Chapters~3, 6, 7]{vidyasagar2003learning}.
            To avoid the case of non-closure learners, we prove the analogous statement for Valiant's model, namely that whenever a class $\hyH$ is learnable in Valiant's model, it is also learnable in this uniform version (\cref{thm:valiant:uniformSuccess}).
            This then implies that whenever $\hyH$ is learnable, there is a weak learner whose output is the closure of the realized version space, which by the earlier argument suffices to bound the weak compression-scheme size.

    \subsubsection*{Part 2 (Small Query Compression Scheme implies Learnability) (\cref{thm:valiant:compression-implies-learnability})} 
        Next, we show that if $\hyH_d$ has a query compression scheme of size $q_d$, then it is learnable with
        \[
            m(d,\eps,\delta)= O\!\inparen{\frac{dq_d+\log{\nfrac1\delta}}{\eps}}~~\text{positive samples}~~\text{and}~~\text{membership queries}\,.
        \]
        As we alluded to before, the algorithm is simple: it runs the query compression scheme on the sample $S$ to obtain a transcript $T=(q,r)$ that certifies $S$, and then outputs the closure over $\V(T)$.
        Since all responses in the transcript were generated by the target $h^\star$, $h^\star\in \V(T)$.
        Hence the closure has one-sided error, and it only remains to show that it generalizes. 
        For this, it suffices to show that the set of all possible outputs of the learner has small VC dimension. 
        This output class is
        \[
            \operatorname{Cl}_{q_d}\!\inparen{\hyH_d}
            \coloneqq
            \inbrace{
                \bigcap_{h\in \V\!\inparen{T}} \supp(h)
                :
                \text{transcript $T$ of length at most $q_d$}
                \text{ and }
                \V\!\inparen{T}\neq\emptyset
            }.
            \yesnum\label{eq:queryClosure:intro}
        \]
        Here, $\V(T)$ is the version space of hypotheses in $\hyH_d$ consistent with $T$.
        This is exactly the relevant output class because every possible output of the learner is the closure of some realized transcript of length at most $q_d$.
        Note that while this output class takes intersections of hypotheses from $\hyH_d$, it is \textit{not} the class of intersections of at most $q_d$ hypotheses from $\hyH_d$; rather, each output is obtained by intersecting the entire version space $\V(T)$, which may contain a large number of hypotheses.
        
        We show that $\vc\!\inparen{\operatorname{Cl}_{q_d}\!\inparen{\hyH_d}}\leq (d+1)(q_d+1)$ (\cref{lem:valiant:kQueryClosureClass:VC}).
        Interestingly, this bound does not depend on the VC dimension or complexity of $\hyH_d$ at all; it comes purely from counting labeled query transcripts of length at most $q_d$ over the finite domain $\zo^d$. This contrasts with the usual class of intersections of at most $q_d$ hypotheses from $\hyH_d$, whose VC dimension generally does depend on $\vc\!\inparen{\hyH_d}$ (\eg{}, \citep{blumer1989learnability,eisenstat2007vcdimension}).
        In this argument, we are crucially using the finiteness of the domain; indeed, on infinite domains even $1$-point query closures can have infinite VC dimension, despite the original class having VC dimension at most two (\cref{thm:onePointClosureInfiniteVC}).

    \begin{remark}
        This compression-based structure of the learner is useful beyond establishing the characterization itself. For instance, it can also be used to deduce closure properties of learnability in Valiant's model, such as union-closedness (see \cref{thm:valiant:unionClosedness}).
    \end{remark}
    
    \subsubsection*{Part 3 (Small VC Dimension Is Necessary for Learning) (\cref{thm:valiant:VCLowerBound})}
    In this part, we prove that learnability in Valiant's model requires sample/query complexity at least linear in the VC dimension, namely that $m\!\inparen{d,\eps,\delta}\geq \inparen{\vc\!\inparen{\hyH_d}-1}/9$ for $\eps,\delta<\nfrac{1}{2}$.
    This lets us replace the $\vc\!\inparen{\hyH_d}$ term appearing in Part~1 by the complexity of the learner itself, and thereby deduce a polynomial-size compression scheme from learnability alone.
    To prove this, we use Yao's minimax principle against deterministic learners under a carefully chosen hard prior.
    In particular, we select the following prior: Starting from a shattered set $X$, we choose a random hidden positive set $U\subseteq X$ of size $\Theta(M)$, let the target be the indicator of $U$, and let the example oracle draw uniformly from $U$, where $M=m\!\inparen{d,\eps,\delta}$.
    Thus, to achieve one-sided error with non-trivial coverage, the learner must recover most of the hidden set $U$ using only positive examples from $U$ and at most $M$ adaptive membership queries, which turns out to be impossible unless \(\abs{X}=O(M)\).
    This implies that $\vc{}(\hyH_d)\geq \abs{X}=O(M)\geq m(d,\eps,\delta).$
 
    \subsubsection{Overview of Learning Halfspaces in $d$-dimensional Euclidean Space}
        
    In this section, we give an overview of our algorithm for learning halfspaces.

        As we have mentioned before, for halfspaces to be learnable in Valiant's model, it is crucial to use membership queries, as without these queries halfspaces are not learnable \citep{natarajan1987learning,shvaytser1990necessary}.
        One also needs to use positive examples because the class of halfspaces is an infinite class, and a result of \citet{kulkarni1993active} implies that with just membership queries, we can only learn finite classes (even if we remove the one-sided error requirement and allow for two-sided error).
        Our learning algorithm is based on a simple observation inspired by our characterization via compression schemes.
        
        \paragraph{Overview for $d=2$.}
        Suppose $d=2$ and consider $m$ positive examples $S$ for an unknown target halfspace $\hstar$.
        For any $S$, there is a set of just three points that certifies all of $S$ (regardless of how large $m$ is).
        To see this, construct the convex hull $C$ of $S$ and let its vertices be $v_1,v_2,\dots,v_{m'}$.
        Observe that there are exactly two facets or edges incident on each vertex $v_i$, and form a triangle $T_i$ by extending these two edges and then selecting a third line passing through one of the other vertices and enclosing the entirety of $C$.
        By construction, for each $1\leq i\leq m'$, $T_i\supseteq C$.
        Then, the key observation is that, for at least one choice of $1\leq i^\star\leq m'$, the triangle $T_{i^\star}$ also satisfies
        \[
            T_{i^\star}\subseteq \supp(\hstar).
        \]
        To see this, imagine sliding a line parallel to the boundary of $\hstar$ toward $C$ until it first touches the convex hull.
        Since $C$ is a polygon, this first contact occurs at some vertex $v_{i^\star}$, and for this choice of $i^\star$ the two edges incident on $v_{i^\star}$ form a wedge that is a subset of $\supp(\hstar)$.
        One can then choose the third line so that the resulting triangle $T_{i^\star}$ still contains $C$ while remaining inside $\supp(\hstar)$.

        If we knew $i^\star$, we would be done (at least for $d=2$), as we can simply query the three vertices of $T_{i^\star}$, which gives us a query compression scheme of size three, and then output the closure of this scheme (which is a triangle, in fact, simply $T_{i^\star}$).
        This is sufficient because (1) one-sided error is guaranteed since $T_{i^\star}\subseteq \supp(\hstar)$ and (2) a bound on the false-negative rate follows because the output of this learner is always a triangle, and the class of triangles has $O(1)$ VC dimension and, hence, generalizes.
        Of course, the index $i^\star$ is unknown.
        We are able to overcome this because of the useful property that the condition $T_{i}\subseteq \supp(\hstar)$ can be certified using membership queries (this uses the convexity of $\supp(\hstar)$ and the fact that $T_i$ is a closure over the class of hypotheses for three queries on its vertices).
        This then allows us to iterate over each $1\leq i\leq m'$ and output the first $T_i$ that passes the check.
        This algorithm uses at most $3m'\leq 3m$ queries and, using standard VC analysis, we can bound the required number of positive samples by $O((\log{\nfrac1\delta})/\eps)$.

        \paragraph{Additional challenges in $d>2$.}
            A natural analog of this algorithm in high dimensions would form the convex hull $C$ of the positive samples $S$, and iterate over its vertices $v_1,v_2,\dots,v_m$.
            For each vertex $v_i$, it would form a cone by picking all the facets $\cF(v_i)$ incident on $v_i$ and then, as before, closing it off with another facet not in $\cF(v_i)$ to obtain a polytope $P_i$.
            This algorithm can also certifiably find a polytope $P_{i^\star}$ with the guarantee that $C\subseteq P_{i^\star}\subseteq \supp(\hstar)$, and this does indeed guarantee one-sided error for $P_{i^\star}$.
 
            However, unlike the case for $d=2$, where $\abs{\cF(v_i)}=2$, here the number of incident facets can be as large as $m^{O(d)}$.
            Hence, the corresponding class of outputs has VC dimension that grows with this complexity, so we do not obtain any useful generalization bound depending only on $d$.
            To overcome this, we simplify the set of outputs of this algorithm by not picking all facets adjacent to a vertex $v_i$, but only a subset $J$ of $d$ facets with linearly independent normals; this suffices to ensure that the resulting wedge is full dimensional, after which we close it off with a final facet.
            There are several such choices of $J$, and due to certain examples, it turns out that we may need to iterate over all of them in order to find one for which the resulting polytope is a subset of $\supp(\hstar)$.
            Hence, we now need to index by both $i$ and $J$, $P_{i,J}$.
            Since we have only removed facets, the guarantee that $P_{i,J}\supseteq C$ continues to hold.
            Since each $P_{i,J}$ is defined by $d+1$ facets, it is a simplex and has a much smaller VC dimension of $\poly(d)$.

            A second challenge is that, for any vertex $v_i$, the number of choices of the subset $J$ can still be as large as $m^{O(d^2)}$, and, since certifying each $P_{i,J}$ requires $d+1$ queries, a naive search will require exponentially many queries in $d$ to find $P_{i^\star,J^\star}$.
            Here we are able to reduce the number of queries required to be polynomial in $d$ by observing that the problem we are solving is equivalent to a point-location problem in which we must determine the labels of a finite set $X$ of candidate witness points with respect to the unknown halfspace; recent point-location results imply that this can be done using only $\wt{O}\!\inparen{d\log\abs{X}}$ queries \citep{hopkins2020pointlocation}.

\subsection{Takeaways, Discussion, and Open Problems}
This work gives the first characterization of Valiant's original learning model.
In particular, we show that, despite its close connections to PAC learning and to positive-only learning without queries, learnability in this model is governed by a different combinatorial principle.

Along the way, we introduce adaptive-query compression schemes (\cref{def:query-compression}), a query-based analogue of sample compression that may be useful more broadly in settings where learners receive examples and also have access to membership queries.
We also introduce the class of $k$-query closures (\cref{eq:queryClosure:intro,def:valiant:kQueryClosureClass}), which plays a key role in our proof and may be of independent interest in other one-sided or interactive learning problems.

We view these results as a starting point. They raise a number of natural questions at the intersection of learning theory, query learning, and the theory of one-sided inference.
The most immediate question is to characterize learnability in Valiant's model beyond finite domains.
While we obtain clean necessary and sufficient conditions over arbitrary domains, these are not tight, and so a full characterization remains open.
\begin{description}
    \item[\textbf{Open Problem 1.}] 
        What is the characterization of learnability in the extension of Valiant's model to arbitrary infinite domains?
\end{description}
Our characterization over finite domains is also quantitative: it characterizes the sample complexity of learning in Valiant's model up to polynomial factors.
A natural direction is to obtain a tight characterization of sample complexity, say up to logarithmic factors.
\begin{description}
    \item[\textbf{Open Problem 2.}] 
        What is the tight characterization of sample complexity of learning in Valiant's model over Boolean domains? 
\end{description}
A third direction concerns the computational complexity of learning in this model.
Our halfspace result over $\R^d$ shows that membership queries can fundamentally enlarge what is learnable in the general-domain setting.
However, the algorithm we obtain is not polynomial time.
This leads to the following question:
 \begin{description}
    \item[\textbf{Open Problem 3.}] 
    Is there a polynomial-time algorithm for learning halfspaces in Valiant's model?
\end{description}
Finally, even at the information-theoretic level, our understanding of halfspaces is incomplete. There remains a substantial gap between the sample and query complexity achieved by our algorithm and the lower bounds we can currently prove.
 \begin{description}
    \item[\textbf{Open Problem 4.}] 
    What are the optimal sample and query complexities for learning halfspaces in Valiant's model?
\end{description}
More broadly, it would be interesting to understand whether the structural notions that arise in our analysis, most notably adaptive-query compression schemes, can be useful beyond Valiant's model itself.

\subsection{Related Work}
    In this section, we present some further related work.
    We note that, given the vast amount of work on PAC learning, we cannot hope to survey all the work.
    For a comprehensive treatment of this setting, see, \eg{}, \cite{vidyasagar2003learning,kearns1994introduction,shalev2014understanding,mohri2018foundations}.

        \paragraph{Learning from positive examples.}
    There is a long line of work on learning from positive data, going back at least to Gold's seminal model of language identification in the limit, where the learner is presented with a \emph{text}, \ie{}, an arbitrary enumeration of the target ``language'' containing only positive examples \citep{gold1967language}.
    This viewpoint was developed further in the rich literature on inductive inference from positive data, \eg{}, \citet{angluin1980inductive,angluin1979finding}.
    Closer to our setting is the work of \citet{natarajan1987learning}, who studied the special case of Valiant's model without membership queries and characterized proper learnability from positive examples alone.
    Subsequently, \citet{shvaytser1990necessary,kivinen1995one} completed this picture for improper learners and established sharp sample complexity bounds.
    These results paint a largely negative worst-case picture: even fundamental classes such as two-dimensional halfspaces are not learnable from positive samples alone.
    To bypass this bottleneck, a substantial line of work has studied positive-only learning under additional structural or distributional assumptions, often for specific hypothesis classes and typically under the uniform distribution on the Boolean hypercube or under Gaussian distributions; see, \eg{}, \cite{frieze1996linearTransformations,Nguyen2009learningParallelepiped,anderson2013simplices,de2015satisfying,Kontonis2019EfficientTS,canonne2020satisfying,he2023robustMoments,lee2024unknown,lee2026smoothed,mansouri2025learning} (also see \cite{anay2026thesis}).
    More broadly, learning from positive samples also arises in modern contexts motivated by (1) generative modeling -- where training typically involves only valid text and there is no canonical distribution over invalid outputs \citep{kleinberg2024language,kalavasis2025limitslanguagegenerationtradeoffs} -- and (2) treatment-effect estimation in observational studies beyond unconfoundedness, which can be formulated as mean estimation from positive-only samples \cite{cai2025makestreatmenteffectsidentifiable}.

    \paragraph{Learning from queries.} There is a long and rich line of 
    work on learning from different types of queries, starting with early work from \citet{angluin1988query}. In the PAC learning literature, queries are most commonly studied in the form of \emph{active learning}, where the learner observes a stream of unlabeled i.i.d.\ data and can ask membership queries for their labels \citep{cohn1994improving,balcan2006agnostic,balcan2010true,dasgupta2011two,hanneke2009theoretical}.

    \paragraph{Learning through compression schemes.}
Compression schemes have been highly influential in learning theory, going back to the foundational work of \citet{floyd1995sample}, who connected compression to PAC learnability.
Subsequently, \citet{warmuth2003compressing} asked whether every class of VC dimension $d$ admits a compression scheme of size $O(d)$. This became one of the most extensively studied questions in learning theory.
\citet{moran2016sample} gave the first general compression bound that depends only on the VC dimension: every binary concept class of VC dimension $d$ admits a sample compression scheme of size $2^{O(d)}$. This established the equivalence between binary PAC learnability and bounded-size sample compression.
This line of work has since been extended in several directions, including efficient learner-to-compression conversions and compression for real-valued learners \citep{hanneke2019sample}.
On the negative side, \citet{chase2024dual} identified limitations of one of the major approaches to the VC-size compression conjecture: embedding arbitrary VC classes into extremal classes of comparable dimension.
Beyond binary classification, \citet{pabbaraju2024multiclass} showed that the analogue fails in the multiclass setting: learnability does not imply bounded-size \mbox{sample compression as a function of DS dimension alone.}

\vspace{-1mm}

    \paragraph{Learning halfspaces with random examples and queries.}
    There is a large body of work on learning halfspaces with random examples (typically both positive and negative examples) and binary-response queries (typically using, in addition to membership queries, other query types such as equivalence or label queries), \eg{}, \citep{baum1991neural,kwek1998pac,blum1998learning,balcan2007margin,balcan2013active,gonen2013efficient,awasthi2017power,yan2017revisiting,zhang2018efficient,zhang2020efficient,li2024efficient,diakonikolas2024active}.
    To the best of our knowledge, in the more relevant setting where the learner sees only positive examples and has access only to membership queries, the only prior work we are aware of is \citet{goldberg2000precision}.
    There are also the following implications for halfspaces:
        The set of halfspaces with bounded finite bit-complexity (which is a finite class) can be learned exactly (and, hence, also with one-sided error) using $\poly(d)$ membership queries and a single arbitrary positive example.
    Thus, to the best of our knowledge, our result on halfspaces is the first to learn halfspaces, without requiring assumptions on bit-complexity, in Valiant's model.

\section{Preliminaries} %
    \label{sec:preliminaries}
    In this section, we collect notation and background used throughout the paper.

    \paragraph{Notation.}
        Logarithms are base two throughout.
        We use calligraphic letters (such as $\cD$ and $\cP$) for distributions over a domain $\cX$.
        When $\cX$ is finite or countably infinite, $\cD(\cdot)$ denotes the probability mass function.
        When $\cX$ is uncountable, for simplicity, we assume that $\cD$ has a density and use the same notation for that density.
        We write $\supp(\cD)\coloneqq \inbrace{x\in \cX : \cD(x)>0}$
        for the support of $\cD$.
        For a hypothesis $h\colon \zo^d\to\zo$, we write $\supp(h)\coloneqq \sinbrace{x\in \zo^d : h(x)=1}.$
        Given a statement $E$, we write $\ind{}\!\inbrace{E}$ for its indicator.

    \paragraph{Example and Membership Oracles.}
    We work in the realizable binary classification setting.
    The domain is $\cX=\zo^{d}$, the target belongs to a concept class $\hyH\subseteq \zo^\cX$, and $\cD$ is a distribution over $\cX$.
    In Valiant's model, the learner {only receives positive examples from the target concept.}
    To formalize this, we introduce the following notion of compatibility.
    \begin{restatable}[Compatibility]{definition}{defCompatibility}
        \label{def:compatible}
        A distribution $\cD$ over $\cX$ is \emph{compatible} with a concept $h\colon \cX\to \zo$ if every point in the support of $\cD$ is positive for $h$; equivalently,
        $\supp(\cD)\subseteq \supp(h)\coloneqq \inbrace{x\in \cX : h(x)=1}.$
    \end{restatable}
    Thus, if $\cD$ is compatible with a target hypothesis $\hstar$, then every draw from $\cD$ is a positive example for $\hstar$.
    The learner is given access to a positive-example oracle and a membership oracle, defined as follows.
    \begin{definition}
        [Example Oracle]
        \label{def:exampleOracle}
        An example oracle for $\cD$ is a primitive that outputs $x \hspace{-0.5mm}\sim \hspace{-0.5mm} \cD$ on each query.
    \end{definition}
    \vspace{-9mm}
    \begin{definition}
        [Membership Oracle]
        \label{def:membershipOracle}
        A membership oracle for $h\colon \cX\to \zo$ is a primitive that, on input $x\in \cX$, returns the label $h(x)$.
    \end{definition}
    \subsection{Discussion of Valiant's Model of Learning}
    \label{sec:discussion:valiant} 
    There are several differences between Valiant's model (\cref{def:valiant}) and the usual PAC model for binary classification over arbitrary domains \citep{shalev2014understanding,blum1998polynomial}.
    First, \cref{def:valiant} requires one-sided error: the output hypothesis must have no false positives. %
    This is a strong constraint: even if one is given an $\eps$-cover of $\hyH_d$ in the usual sense, it does not suffice for learning, since the covering hypotheses may introduce false positives.
    Second, the sample and query complexities are required to be polynomial in the ambient dimension $d$, whereas the PAC model only requires them to be finite \citep{shalev2014understanding,blum1998polynomial}.
    This requirement is essential: for each fixed $d$, the domain $\zo^d$ has size $2^d$, so exhaustively querying the membership oracle trivially learns any concept; without the polynomial requirement, the definition would be vacuous on finite domains, and, hence, requiring $\poly(d)$ samples and queries is the standard modeling choice for learning problems over the Boolean hypercube \cite{kearns1994introduction}.
    Third, unlike the PAC model where the feature distribution is arbitrary and unrelated to the target hypothesis, here the distribution $\cD$ is required to be compatible with $\hstar$, \ie{}, the learner only receives positive examples.
    Lastly, compared to work on active learning, the learner is allowed to query the label of arbitrary points of the domain.
    These differences place Valiant's model outside the reach of the existing machinery and require us to explore new ideas.

    The above model differs from Valiant's original formulation in two respects: 
    First, Valiant required the learning algorithm to be proper (\ie{}, its output was required to be in $\hyH_d$) and, second, he required the learner to run in polynomial time.
    We drop both requirements to develop a general theory, as both are known to complicate characterizations of learnability.
    Indeed, while sample-efficient learnability has been characterized in the PAC, online, and active learning models, among others \citep{blumer1989learnability,hanneke2014theory,littlestone1988learning}, no comparably general characterization of \emph{computationally} efficient learnability is known in any of them.
    Further, requiring properness can make it much harder to obtain a clean general theory. This already happens in closely related PAC-style settings: for instance, in multiclass classification, a characterization of proper learning is known to be undecidable under ZFC \citep{asilis2025proper}.

    Valiant's original definition used a single parameter $\eps$, for both the failure probability and the allowed false-negative mass.
    We separate them into $\eps$ and $\delta$, since we require this flexibility in our proofs.
    We show that after appropriately boosting a weak learner, the sample and query complexities depend at most linearly on $\nfrac{1}{\eps}$ (up to logarithmic factors) and logarithmically on $\nfrac{1}{\delta}$ (see \cref{thm:valiant:boosting}).
    Finally, for some of his results, Valiant allowed the features to be undetermined (partial assignments); for simplicity, we focus on total assignments over $\zo^d$.%

\subsection{Adaptive Query Strategies and Version Spaces}
\label{sec:preliminaries:adaptiveQueryStrategies}
In Valiant's model, the learner's queries to the membership oracle may be adaptive: the $i$-th query may depend on all previously observed positive examples and past query-response pairs.
We reserve $q=\inparen{q_1,\dots,q_t}$ for a finite sequence of queried points in $\zo^d$ and $r=\inparen{r_1,\dots,r_t}$ for the corresponding responses.
For a hypothesis $h\colon \zo^d\to\zo$, we write $h(q)\coloneqq \inparen{h(q_1),\dots,h(q_t)}.$

\begin{definition}[Adaptive Query Strategy]
\label{def:valiant:adaptiveQueryStrategy}
Fix $d\in\N$.
A deterministic adaptive query strategy $\sigma$ of depth at most $t$ is a rooted binary tree of depth at most $t$ whose internal nodes are labeled by points of $\zo^d$.
The root label is the first query; if the current node is labeled $q_i$ and the response is $r_i\in\zo$, the strategy follows the left child when $r_i=0$ and the right child when $r_i=1$.
For any response sequence $r$ realized along a root-to-leaf path of $\sigma$, we write $q\!\inparen{\sigma,r}$
for the corresponding sequence of queried points on that path.
\end{definition}
When discussing randomized learners, we condition on their internal randomness and work with deterministic strategies. %
\begin{remark}[Positive Examples May Be Collected First]
\label{rem:positive-samples-first}
Suppose a learner satisfying \cref{def:valiant} makes at most $m\!\inparen{d,\eps,\delta}$ calls to the example oracle and at most the same number of calls to the membership oracle.
Then, without changing its guarantee, we may assume that it first requests $m\!\inparen{d,\eps,\delta}$ independent draws from the example oracle, stores them, and only then interacts with the membership oracle.
This is because the transformed learner can simulate the original one.
Hence, without loss of generality, one can assume that the learner first observes its entire positive sample and then, after fixing its internal randomness, runs a deterministic adaptive query strategy.
\end{remark}
A version space is the set of hypotheses consistent with the information revealed so far; this notion recurs throughout the paper. 
\begin{definition}[Version Spaces]
    \label{def:valiant:positiveVersionSpace}
    \label{def:valiant:queryVersionSpace}
    Fix $d\in\N$.
    We introduce two associated version spaces:
    \begin{enumerate}[leftmargin=15pt,itemsep=0pt]
        \item For a finite set $S\subseteq \zo^d$, the positive-sample version space is
        $
            \V\!\inparen{S}
            \coloneqq
            \inbrace{h\in \hyH_d : S\subseteq \supp(h)}.
        $
        \item For a query sequence $q$ and response sequence $r$, the query-transcript version space is
        $
            \V\!\inparen{q,r}
            \coloneqq
            \inbrace{h\in \hyH_d : h(q)=r}.
        $
    \end{enumerate}
\end{definition}
Thus, the hypotheses consistent with both a positive sample $S$ and a query transcript $\inparen{q,r}$ are precisely $\V\!\inparen{S}\cap \V\!\inparen{q,r}.$
If $T=\inbrace{(q_1,r_1),\dots,(q_t,r_t)}$ denotes a transcript, we write $\V(T)\coloneqq \V\!\inparen{q,r}$.
Likewise, for a deterministic adaptive query strategy $\sigma$ and a realized response sequence $r$, %
\begin{equation}
    \V\!\inparen{\sigma,r}
    \coloneqq
    \V\!\inparen{q\!\inparen{\sigma,r},r}.
    \label{eq:valiant:adaptiveQueryVersionSpace}
\end{equation}

\subsection{VC Dimension and One-Centered Star Number}
    \label{sec:additional-preliminaries}
        Next, we introduce relevant dimensions from learning theory that we will reference throughout.
        
    \begin{definition}[VC Dimension \cite{vapnik1974theory,vapnik1971uniform}]
    \label{def:vc-dimension}
    The Vapnik--Chervonenkis (VC) dimension of $\hyH$, denoted $\vc(\hyH)$, is the largest integer $k$ for which there exists a set of points $\inbrace{x_1,\dots,x_k}\subseteq\cX$ that is shattered by $\hyH$, \ie{}, $\bigl\{ \bigl(h(x_1),\dots,h(x_k)\bigr) : h\in\hyH\bigr\}
        \,=\,\{0,1\}^k\,.$
    If no such finite $k$ exists, then $\vc(\hyH)=\infty$.
    \end{definition}
    In other words, if $\vc(\hyH)=k$, then there exists a set of $k$ points $X=\inbrace{x_1,x_2,\dots,x_k}$ such that every pattern on the points $X$ can be realized by some hypothesis $h$ in the hypothesis class $\hyH$. 
    However, no set of size $k+1$ is shattered.
    \begin{definition}[Star Number
    \citep{HannekeYang2015}]
        \label{def:star-number}
    For any concept class $\hyH\subseteq\zo{}^\cX$ and any function $h\colon\cX\to\zo$ (not necessarily in $\hyH$), the star number of $\hyH$ \emph{centered} at $h$, denoted $\mathfrak{s}_h(\hyH)$, is the largest integer $k$ for which there exist points $x_1,\dots,x_k\in\cX$ such that for each $i\in\{1,\dots,k\}$ there is a concept $h_i\in\hyH$ satisfying
    \[
    h_i(x_j)=
    \begin{cases}
    h(x_j),&j\neq i,\\
    1 - h(x_j),&j = i\,,
    \end{cases}
    \]
    and there is also a concept $h_0\in \hyH$ with $h_0(x_i)=h(x_i)$ for all $1\leq i \leq k$.
    If no such finite $k$ exists, then $s_h(\hyH)=\infty$.  
    \end{definition}
    Intuitively, the star number captures a notion of ``local disagreement capacity.'' 
Given a reference labeling $h$, it counts how many hypotheses in $\hyH$ can make exactly one ``surgical edit'' to $h$'s predictions, with no two hypotheses editing the same point.
The star number has found a number of applications in active learning.
Recently, \citet{hanneke2024Star} identified and studied the specific star number defined by the all-one hypothesis, which they termed the \textit{$1$-centered star number}.
    \begin{definition}[1-Centered Star Number \citep{hanneke2024Star}]
        \label{def:one-star-number}
        Let $1\!:\cX\to\{0,1\}$ be the constant-$1$ function. The $1$-centered star number of $\hyH$ is denoted by $\oneStar(\hyH)$. 
    \end{definition}
    The $1$-centered star number is particularly interesting for us due to the following connection, found by \citet{hanneke2024Star}, between $\oneStar(\hyH)$ and the VC dimension of the intersection closure of $\hyH$.\footnote{The intersection closure of a class $\hyH$ is the smallest class containing $\hyH$ that is closed under finite intersections.}
    \vspace{-5mm}
    \begin{theorem}[Theorem 19 of \cite{hanneke2024Star}]
        \label{thm:oneStarDimension}
        For any concept class $\hyH\subseteq \zo^\cX$, the smallest VC dimension of any class $\overline{\hyH}$ containing $\hyH$ that is closed under intersections is $\oneStar(\hyH)$. %
    \end{theorem}
    The results of \citet{natarajan1987learning,shvaytser1990necessary}, expressed in the notation of star numbers, lead to the following characterization of learnability in the special case of Valiant's model in which no membership queries are allowed.
    \begin{theorem}[\citet{natarajan1987learning,shvaytser1990necessary} Combined with \cref{thm:oneStarDimension}]
        \label{thm:positive-characterization}
        A sequence of concept classes $\hyH=(\hyH_1,\hyH_2,\dots)$ is learnable in Valiant's model \underline{without} membership queries if and only if there is a polynomial $q(\cdot)$ such that $\oneStar(\hyH_d)=q(d)$ for each $d\in \N$.
    \end{theorem}
    Next, we turn to our results that characterize Valiant's model in its full generality.

\section{Characterization of Learnability}
\label{sec:proofof:thm:characterization}
        In this section, we prove \cref{thm:characterization}, which we restate below.
        \thmCharacterization*
        \subsection{Additional Discussion of Query-Compression and Characterization}
        Before proving \cref{thm:characterization}, in this section, we collect some useful remarks.
        
\begin{remark}[Extension to Any Finite Domain]
    \label{rem:finite-characterization}
    Our proof of \cref{thm:characterization} naturally extends beyond the Boolean hypercube to any \textit{finite} domain, where learnability is defined as follows:
    Fix a sequence of domains $\cX=(\cX_1,\cX_2,\dots)$ with $\abs{\cX_N}=N$ for each $N\in \N$, and fix a corresponding sequence of concept classes $\hyH=\inparen{\hyH_1,\hyH_2,\dots}$. %
    Then, $\hyH$ is learnable if there is a learner satisfying the requirements of \cref{def:valiant} with $m(N,\eps,\delta)=\poly(\log{N},1/\eps,1/\delta)$.
    The reason is that the proof of \cref{thm:characterization} uses the Boolean-cube structure only through two counting bounds, neither of which is essential:
    \begin{itemize}
        \item To obtain compression schemes from learnability, the proof of \cref{thm:valiant:weak-compression} uses Sauer--Shelah's lemma to bound the size of the class; over an $N$-point domain, this gives $\log \abs{\hyH_N}\leq O\!\inparen{\vc\!\inparen{\hyH_N}\log N}$.
    Together with \cref{thm:valiant:VCLowerBound}, this yields a weak compression scheme of size $\poly\!\inparen{\log N}$, and then \cref{thm:valiant:compression} upgrades it to a full query-based compression scheme of size $\poly\!\inparen{\log N}$ (where we used that $\log\abs{S}\leq \log N$). %
    \item 
    To deduce learnability from compression schemes, the proof of \cref{thm:valiant:learningFromCompression} only needs a bound on the size of the $k$-query closure class.
    Over an $N$-point domain, a labeled query transcript of length $t$ has at most $(2N)^t$ possibilities, so $\abs{\operatorname{Cl}_{k}\!\inparen{\hyH_N}}\leq \sum_{t=0}^{k}(2N)^t\leq (2N)^{k+1}$, and hence $\vc\!\inparen{\operatorname{Cl}_{k}\!\inparen{\hyH_N}}\leq O\!\inparen{(k+1)\log N}$.
    The same uniform-convergence argument then yields learnability from compression with dependence polynomial in $\log N$.
    \end{itemize}
    Thus the entire characterization carries over verbatim, replacing $d$ by $\log N$ throughout.
\end{remark}

\begin{remark}[Comparison of \cref{def:query-compression} to Standard Compression Schemes]
    \label{rem:comparison-of-compression}
    Query compression schemes and standard sample compression schemes are incomparable to each other and differ in two respects.
    First, query compression schemes are more general: a standard compression scheme must compress the labels using only information already contained in the sample, whereas a query compression scheme may request additional information via membership queries.
    Second, query compression schemes are more restrictive in how labels are recovered: a standard compression scheme may reconstruct labels by any means, whereas here we require that the labels be forced by the version-space closure of the transcript.
    To see why this stronger requirement is needed, observe that in Valiant's model the learner only observes positive examples. Under the usual notion of compression, any all-positive sample can be trivially compressed by a scheme of size~$0$: the reconstruction map simply labels every point as positive. This would make compression vacuous, so we instead require that the labels have to be forced by the version space.
\end{remark}

\begin{remark}[The Role of Adaptivity]
    \label{rem:adaptivity}
    Adaptivity in \cref{def:query-compression} is essential.
    Without it, the learner would need to query every node of the strategy tree $\Sigma_S$, which may contain as many as $2^t$ nodes when $\Sigma_S$ has depth $t$---far too many for learnability in Valiant's model, where the complexity must be $\poly(d)$.
    After the interaction, the realized branch consists of only $t$ query-response pairs along a single root-to-leaf path.
    Hence, the version space certifying $S$ is only defined by these $t$ query-response pairs.
    One might therefore hope to replace the adaptive strategy by a static list of $t$ queries chosen in advance. 
    However, the key difficulty is that the learner does not know in advance which of the $2^t$ possible paths will be realized, since this depends on the responses of the membership oracle and hence on the unknown target $\hstar$.
    Thus, the queries that are ultimately needed are only determined via interaction with the membership oracle.
\end{remark}

\subsection{Proof of \cref{thm:characterization} (Characterization)}
    In this section, we prove our main result: \cref{thm:characterization}.
    An outline of the proof appears in \cref{fig:outline:characterization}.

    \begin{figure}[h!]
    \centering
    \begin{tikzpicture}[node distance=2.5cm, scale=0.83, transform shape]

    \node[myNode, line width=2pt, text width=9cm] (main) at (0,-2)
        {{Characterization of Learnability \\[1mm] \cref{thm:characterization}}};

    \node[myNode, line width=2pt, text width=4cm] (compression) at (-6,-4.5)
        {{Learnability implies query compression \\[1mm] \cref{thm:valiant:compression}}};

    \node[myNode, line width=2pt, text width=4cm] (vcLB) at (0,-4.5)
        {{$\vc(\hyH_d)\leq m(d,\eps,\delta)$\\[1mm] \cref{thm:valiant:VCLowerBound}}};

    \node[myNode, line width=2pt, text width=4cm] (learnFromCompression) at (6,-4.5)
        {{Query compression implies learnability \\[1mm] \cref{thm:valiant:learningFromCompression}}};

    \node[myNodeFlex, line width=1pt, text width=4.8cm] (weakCompression) at (-6,-7.1)
        {{Existence of weak query compression schemes \\[1mm] \cref{thm:valiant:weak-compression}}};

    \node[myNodeFlex, line width=1pt, text width=3.4cm] (kClosureLemma) at (6,-7)
        {{VC dimension of $k$-query closure class \\[1mm] \cref{lem:valiant:kQueryClosureClass:VC}}};

    \node[myNodeFlex, line width=1pt, text width=3.1cm] (compressionRate) at (1.5,-6.65)
        {{Compression rate \\[1mm] \cref{def:compression-rate}}};

    \node[myNodeNarrow, line width=1pt, text width=3.4cm] (boosting) at (-10.0,-9.5)
        {{Boosting for Valiant's model \\[1mm] \cref{thm:valiant:boosting}}};

    \node[myNodeNarrow, line width=1pt, text width=3.6cm] (uniformSuccess) at (-6,-9.5)
        {{Uniform success over all potential targets \\[1mm] \cref{thm:valiant:uniformSuccess}}};

    \node[myNodeNarrow, line width=1pt, text width=4cm] (uniformSuccessDef) at (-6,-12)
        {{Uniform success over all potential targets \\[1mm] \cref{def:valiant:hpAllPotentialTargets}}};

    \node[myNodeNarrow, line width=1pt, text width=3.1cm] (weakCompressionDef) at (-2.0,-9.5)
        {{Weak query-based compression scheme \\[1mm] \cref{def:valiant:weak-compression}}};

    \node[myNodeNarrow, line width=1pt, text width=3.6cm] (kClosureDef) at (6.0,-9.5)
        {{The $k$-query closure class \\[1mm] \cref{def:valiant:kQueryClosureClass}}};

    \draw[-stealth, line width=0.5mm] (compression) -- (main);
    \draw[-stealth, line width=0.5mm] (vcLB) -- (main);
    \draw[-stealth, line width=0.5mm] (learnFromCompression) -- (main);

    \draw[-stealth, line width=0.5mm] (weakCompression) -- (compression);

    \draw[-stealth, line width=0.5mm] (boosting) -- (weakCompression);
    \draw[-stealth, line width=0.5mm] (uniformSuccess) -- (weakCompression);
    \draw[-stealth, line width=0.5mm] (weakCompressionDef) -- (weakCompression);

    \draw[-stealth, line width=0.5mm] (uniformSuccessDef) -- (uniformSuccess);

    \draw[-stealth, line width=0.5mm] (compressionRate) -- (learnFromCompression);
    \draw[-stealth, line width=0.5mm] (kClosureLemma) -- (learnFromCompression);

    \draw[-stealth, line width=0.5mm] (kClosureDef) -- (kClosureLemma);

    \draw[-stealth, line width=0.5mm] (weakCompressionDef) -- (weakCompression);

    \end{tikzpicture}
    \caption{\textit{Outline of the proof of \cref{thm:characterization}.}}
    \label{fig:outline:characterization}
\end{figure}
        
\subsubsection{Boosting in Valiant's Model}
We begin by proving a boosting theorem for Valiant's model.
\begin{theorem}[Boosting for Valiant's Model]
    \label{thm:valiant:boosting}
    Suppose $\hyH$ is learnable in Valiant's model (\cref{def:valiant}), and let
    $m_{\mathrm W}\inparen{d,\eps,\delta}$ denote the sample/query complexity of some learner witnessing this.
	    Then, for every $\eps\in(0,1)$ and $\delta\in(0,\nfrac12]$, there is a learner satisfying the Valiant guarantee with sample/query complexity
    \[
        m\!\inparen{d,\eps,\delta}
        =
        m_{\mathrm W}\!\inparen{d,\nfrac{1}{18},\nfrac13}
        \cdot
        O\!\inparen{
            \frac{1}{\eps}
            \log\frac1\delta
        }\,.
    \]
\end{theorem}
We present the proof of \cref{thm:valiant:boosting} in \cref{sec:proofof:thm:valiant:boosting}. 
It uses standard boosting ideas adapted to Valiant's learning model:
First, we reduce the failure probability to $\delta$ by taking a majority vote over $O\!\inparen{\log\nfrac1\delta}$ independent runs.
The argument for improving the dependence on $\eps$ is a bit more involved:
we run $O\!\inparen{\nfrac1\eps}$ stages of the learner, where in the $(i+1)$-th stage we ``focus'' on the positive samples which are not positively classified by any of the first $i$ hypotheses.
To do this, we use rejection sampling and only feed the learner the positive examples that are not positively classified by any of the first $i$ hypotheses.
(Since each intermediate hypothesis has no false positives, the resulting residual distribution remains compatible with the target.)
Thus, the proof follows the standard boosting template of repeatedly learning the residual distribution, with the main change being how the residual distribution is defined and simulated.

\subsubsection{Uniform Success Over All Potential Targets}
In this section, we show that any learnable class in Valiant's model admits a learner whose output, with high probability, is simultaneously valid for every hypothesis that remains consistent with the realized interaction.

\begin{definition}[Uniform Success Over All Potential Targets]
\label{def:valiant:hpAllPotentialTargets}
A learner $\cL$ has \emph{uniform success over all potential targets} with sample/query complexity
$m_{\mathrm{unif}}\!\inparen{d,\eps,\delta}$ if the following holds.
For every $d\in \N$, $\eps,\delta\in \inparen{0,1}$, target $\hstar\in \hyH_d$, and distribution $\cD$ compatible with $\hstar$,
run $\cL$ on input $\inparen{d,\eps,\delta}$ with access to the example oracle for $\cD$ and the membership oracle for $\hstar$.
Let $\wh{h}$ be the output, let $S$ be the realized set of example-oracle samples, and let $q,r$ be the realized membership-query transcript.
Then, with probability at least $1-\delta$, for every
\[
    h\in \V\!\inparen{S}\cap \V\!\inparen{q,r}
    \qquad\text{with}\qquad
    \supp(h)\neq\varnothing,
\]
the hypothesis $\wh{h}$ satisfies the conclusions of \cref{def:valiant} with target $h$ and distribution $\cD_h$, where $\cD_h$ denotes $\cD$ conditioned on $\supp(h)$ (when $\cD(\supp(h))=0$, $\cD_h$ can be any fixed distribution supported on $\supp(h)$).
\end{definition}
The above requirement is closely related to the model of Probably Uniformly Approximately Correct (PUAC) learning introduced by \citet[see Chapters~3, 6, and 7]{vidyasagar2003learning}.
In the PUAC model, with probability at least $1-\delta$ over the random sample, the learner must succeed \textit{simultaneously} for all potential targets, rather than merely for each fixed target separately (as in the PAC model).
Hence, the above requirement strengthens Valiant's model in the same spirit as PUAC learning strengthens PAC learning.

Our next result shows that if a class is learnable in Valiant's model, then it is learnable by a learner that also enjoys the uniform-success guarantee above.
This parallels the equivalence between PAC and PUAC learning for binary classification proved by \citet[Theorem~7.11]{vidyasagar2003learning}.
However, that equivalence is obtained via classical uniform-convergence arguments based on suitable covering-number bounds.
Because Valiant's model requires a one-sided guarantee which is not compatible with standard bounds on covering numbers, we prove the corresponding statement from scratch.
\begin{theorem}[Uniform Success Over All Potential Targets]
\label{thm:valiant:uniformSuccess}
Suppose $\hyH$ is learnable in Valiant's model, and let
$m\!\inparen{d,\eps,\delta}$ denote the sample/query complexity of some learner witnessing this.
Then $\hyH$ admits a learner satisfying \cref{def:valiant:hpAllPotentialTargets} with sample/query complexity
\[
    m_{\mathrm{unif}}\!\inparen{d,\eps,\delta}
    =
    m\!\inparen{d,\eps,\frac{\delta}{\abs{\hyH_d}}}.
\]
\end{theorem} 

\begin{proof}[Proof of \cref{thm:valiant:uniformSuccess}]
Fix $d\in\N$, $\eps,\delta\in\inparen{0,1}$, a target $\hstar\in\hyH_d$, and a distribution $\cD$ compatible with $\hstar$ (\cref{def:compatible}).
Let
\[
    N_d \coloneqq \abs{\hyH_d}
    \qquadand
    \delta' \coloneqq \frac{\delta}{N_d}.
\]
Let $\cL$ be a learner witnessing the learnability of $\hyH$ with sample/query complexity $m$, and let $\cL'$ be the learner obtained by running $\cL$ with confidence parameter $\delta'$.
By \cref{rem:positive-samples-first}, we may assume that $\cL'$ first draws 
\[
    M \coloneqq m\!\inparen{d,\eps,\delta'}
\]
example-oracle samples $x_1,\dots,x_M$, and only then makes adaptive membership queries.
Let $S\coloneqq \inbrace{x_1,\dots,x_M}.$ 
Let $q,r$ be the realized membership-query transcript, and let $\wh{h}$ be the output of $\cL'$.

Fix $h\in \hyH_d$ with $\supp(h)\neq\varnothing$, and define $\cD_h$ as in the theorem statement.
Define the event
\[
    \evE_h
    \coloneqq
    \inbrace{
        h\in \V\!\inparen{S}\cap \V\!\inparen{q,r}
        \text{ and }
        \wh{h}
        \text{ fails the conclusions of \cref{def:valiant} for target } h
        \text{ under } \cD_h
    }.
\]
We claim that $ \Pr\inparen{\evE_h}\leq \delta'.$
To prove this, condition on the internal randomness $\rho$ of $\cL'$, so that the learner becomes deterministic; let this learner be $\cL'_{\rm deter}$.
Let $A_h^\rho\subseteq \sinparen{\zo^d}^{M}$ be the set of ordered sample sequences
$x=\inparen{x_1,\dots,x_M}$
for which the corresponding deterministic run lands in $\evE_h$.
If $x\in A_h^\rho$, then $x_i\in \supp(h)$ for every $i$, because $h\in \V\!\inparen{S}$.
Also, $h\!\inparen{q}=r$, because $h\in \V\!\inparen{q,r}$.
Hence, if we rerun the same deterministic learner with the same $\rho$, but now with target $h$ and example distribution $\cD_h$, then on the same sample sequence $x$ the learner sees exactly the same transcript and therefore outputs the same hypothesis.
In particular, this counterfactual run also fails for the compatible pair $\inparen{h,\cD_h}$.
Let $F_h^\rho\subseteq \sinparen{\zo^d}^{M}$ be the set of ordered sample sequences on which this counterfactual deterministic run fails for the pair $\inparen{h,\cD_h}$. The preceding paragraph shows that $A_h^\rho\subseteq F_h^\rho$.
Therefore, if $\cD\!\inparen{\supp(h)}>0$, then
\begin{align*}
    \Pr\inparen{\evE_h\mid \rho}
    =
    \cD\!\inparen{x\in A_h^\rho}^{M}
    =
    \cD\!\inparen{\supp(h)}^{M}
    \cdot
    \cD_h^{M}\!\inparen{x\in A_h^\rho}
    \leq
    \cD_h^{M}\!\inparen{F_h^\rho}\,,
\end{align*}
Otherwise, suppose $\cD\!\inparen{\supp(h)}=0$.
If $M\geq 1$, then on the event $h\in \V\!\inparen{S}$ all $M$ example samples must lie in $\supp(h)$, which happens with probability $0$ under $\cD$.
Hence
\[
    \Pr\inparen{\evE_h\mid \rho}=0\leq \cD_h^M\!\inparen{F_h^\rho}.
\]
If instead $M=0$, then there are no example-oracle samples.
On the event $h\in \V\!\inparen{q,r}$, the deterministic run with target $h$ and distribution $\cD_h$ sees exactly the same membership-query transcript as the original run, and therefore outputs the same hypothesis.
Thus,
\[
    \Pr\inparen{\evE_h\mid \rho}
    \leq \cD_h^{M}\!\inparen{F_h^\rho}\,.
\]
Averaging over $\rho$ and using the Valiant guarantee for the randomized learner $\cL'$ on the compatible pair $\inparen{h,\cD_h}$ gives
\[
    \Pr\inparen{\evE_h}
    \leq
    \E_{\rho}\!\left[\cD_h^{M}\!\inparen{F_h^\rho}\right]
    =
    \Pr_{\rho,\,x\sim \cD_h^M}\!\inparen{
        \text{$\cL'$ fails on }\inparen{h,\cD_h}
    }
    \leq
    \delta'\,.
\]
Next, we apply a union bound over all $h\in \hyH_d$ with $\supp(h)\neq\varnothing$:
\[
    \Pr\inparen{
        \exists h\in \hyH_d
        \text{ with }
        \supp(h)\neq\varnothing
        \text{ such that }
        \evE_h
    }
    \leq
    \sum_{h\in \hyH_d}\Pr\inparen{\evE_h}
    \leq
    N_d\delta'
    =
    \delta.
\]
Equivalently, with probability at least $1-\delta$, no event $\evE_h$ occurs.
Thus, simultaneously for every $h\in \V\!\inparen{S}\cap \V\!\inparen{q,r}$ with $\supp(h)\neq \varnothing$, the output $\wh{h}$ satisfies the conclusions of \cref{def:valiant} for target $h$ under $\cD_h$, implying \cref{def:valiant:hpAllPotentialTargets}.
Finally, the result follows since the learner $\cL'$ is simply $\cL$ run with confidence parameter
$\delta'=\delta/\abs{\hyH_d}$.
Hence its sample/query complexity is $m\!\inparen{d,\eps,\nfrac{\delta}{\abs{\hyH_d}}}.$
\end{proof}

\subsubsection{Learnability Implies Query-Based Compression}
In this section, we show that learnability in Valiant's model yields logarithmic-size query-based compression schemes for realizable positive samples.

In other words, along every realizable branch of the strategy $\Sigma$, the transcript certifies that every point of $S$ is positive for every hypothesis consistent with that branch.

\begin{theorem}[Learnability Implies Query-Based Compression]
\label{thm:valiant:compression}
Suppose $\hyH$ is learnable in Valiant's model, and let
$m_{\mathrm W}\!\inparen{d,\eps,\delta}$ denote the sample/query complexity of some learner witnessing this.
Then for every $d\in\N$ and every nonempty finite set $S\subseteq \zo^d$ with $\V\!\inparen{S}\neq\emptyset$, there exists a query-based compression scheme $\Sigma_S$ (\cref{def:query-compression}) for $S$ whose compression rate is at most
\[
    q_d \cdot \inparen{1+\left\lceil \log_{3/2}\abs{S}\right\rceil},
\]
where $q_d$ is as in \cref{thm:valiant:weak-compression}.
In particular, $q_d
    \leq
    m_{\mathrm W}\!\inparen{d,\nfrac{1}{18},\nfrac13}
    \cdot
    O\!\inparen{
        d\,\vc\!\inparen{\hyH_d}
    }.$
\end{theorem}
The proof proceeds in two stages.
First, we show that every realizable set $S$ admits a weak query-based compression scheme whose transcript certifies a constant fraction of the points in $S$ (\cref{def:valiant:weak-compression,thm:valiant:weak-compression}).
Second, we recurse on the uncertified points.
Since each round removes a constant fraction of the unresolved set, the recursion stops after at most $O\!\inparen{\log \abs{S}}$ rounds.

\begin{definition}[Weak Query-Based Compression Scheme]
\label{def:valiant:weak-compression}
Fix $d\in\N$ and a finite set $S\subseteq \zo^d$.
We say that $S$ \emph{admits a weak query-based compression scheme} $\sigma$ if $\sigma$ is a deterministic adaptive query strategy such that, for every response transcript $r$ of $\sigma$ that is realizable with respect to $\V\!\inparen{S}$,
\[
    \abs{
        S\cap \bigcap_{h\in \V\!\inparen{\sigma,r}} \supp(h)
    }
    \geq
    \frac13\abs{S}.
\]
When this holds, we also say that $\sigma$ is a \emph{weak query-based compression scheme for $S$}.
\end{definition}
The constant $\nfrac{1}{3}$ in the above definition is not crucial; it can be replaced by any positive constant.
This only changes the hidden constant in \eqref{eq:valiant:usefulQueries:size} and the base of the logarithm in \cref{thm:valiant:compression}.

\begin{theorem}[Existence of Weak Query-Based Compression Schemes]
\label{thm:valiant:weak-compression}
Suppose $\hyH$ is learnable in Valiant's model, and let
$m_{\mathrm W}\!\inparen{d,\eps,\delta}$ denote the sample/query complexity of some learner witnessing this.
Then for every $d\in\N$ and every nonempty finite set $S\subseteq \zo^d$ with $\V\!\inparen{S}\neq\emptyset$, there exists a weak query-based compression scheme $\sigma_S$ for $S$ making at most $q_d$ membership queries, where one may take
\[
    q_d
    =
    m_{\mathrm W}\!\inparen{d,\nfrac{1}{18},\nfrac13}
    \cdot
    O\!\inparen{
        d\,\vc\!\inparen{\hyH_d}
    }.
    \yesnum\label{eq:valiant:usefulQueries:size}
\]
\end{theorem}
At a high level, the proof of \cref{thm:valiant:weak-compression} puts the uniform distribution on $S$ and runs the learner from the previous subsection---with the uniform guarantee from \cref{def:valiant:hpAllPotentialTargets}.
This immediately gives the right coverage statement: if the learner has false-negative mass at most $\nfrac23$ under the uniform distribution on $S$, then its output contains at least one third of the points of $S$.
The real difficulty is the one-sided guarantee.
We do not merely want a hypothesis contained in the \emph{true} target.
Rather, we want a transcript that certifies positivity for \emph{every} hypothesis that remains consistent with the realized interaction.
This is exactly why we need the uniform guarantee from \cref{def:valiant:hpAllPotentialTargets}: ordinary learnability controls the output only relative to the true target and says nothing about the rest of the version space.

A second issue is target-independence.
For each $h\in \V\!\inparen{S}$ there may be a different successful random seed and a different successful sample stream.
That is not enough: the final query tree must depend only on $S$, not on the unknown target.
To achieve this, we run the learner with confidence on the order of $1/\abs{\hyH_d}$ and take a union bound over all potential targets.
The result is a \emph{single} realization that works simultaneously for every $h\in \V\!\inparen{S}$.
Once such a realization is fixed, the random example stream can be compiled into actual membership queries, turning the randomized learner into a deterministic query strategy.

\begin{proof}[Proof of \cref{thm:valiant:weak-compression}]
Fix $d\in\N$ and a nonempty finite set $S\subseteq \zo^d$ with $\V\!\inparen{S}\neq\emptyset$.
Let $\mathcal U_S$ denote the uniform distribution on $S$.
For every $h\in \V\!\inparen{S}$, the distribution $\mathcal U_S$ is compatible with $h$.
By \cref{thm:valiant:uniformSuccess,thm:valiant:boosting}, there is a learner $\cL_{\mathrm{unif}}$ which, on input $\inparen{d,\nfrac23,\nfrac{1}{2\abs{\hyH_d}}},$ satisfies the uniform guarantee from \cref{def:valiant:hpAllPotentialTargets} and uses at most
\[
    m_{\mathrm W}\!\inparen{d,\nfrac{1}{18},\nfrac13}
    \cdot
    O\!\inparen{
        \log\abs{\hyH_d}
    }
\]
example-oracle calls and at most the same number of membership queries.
Since $\hyH_d$ is a class over a domain of size $2^d$, Sauer--Shelah's lemma implies $\log\abs{\hyH_d}
    =
    O\!\inparen{
        d\,\vc\!\inparen{\hyH_d}
    }$
(see, \eg{}, \citet{shalev2014understanding}).
Choose $q_d\geq 1$ satisfying \eqref{eq:valiant:usefulQueries:size} and large enough to absorb the factor $2$ incurred below when example-oracle calls are compiled into membership queries.

\paragraph{A Single Successful ``Run'' for All Targets in $\V(S)$.}
Let $\rho$ denote the internal randomness of $\cL_{\mathrm{unif}}$ together with an i.i.d.\ example stream from $\mathcal U_S$.
For each target $h\in \V\!\inparen{S}$, let $\evG_h$ be the event that, when $\cL_{\mathrm{unif}}$ is run with target $h$, distribution $\mathcal U_S$, and realization $\rho$, it satisfies the conclusion of \cref{def:valiant:hpAllPotentialTargets} (\ie{}, it succeeds for all other potential targets $h'$ which remain consistent with the transcript and samples). 
Then for every such $h$, $\Pr\!\inparen{\evG_h}
    \geq
    1-\frac{1}{2\abs{\hyH_d}}.$
Hence, by the union bound,
\[
    \Pr\!\inparen{
        \bigcap_{h\in \V\!\inparen{S}} \evG_h
    }
    \geq
    1-
    \frac{\abs{\V\!\inparen{S}}}{2\abs{\hyH_d}}
    \geq \frac{1}{2}.
\]
Therefore there exists a realization $\rho^\star$ such that every event $\evG_h$, for $h\in \V\!\inparen{S}$, holds simultaneously.
Fix such a realization $\rho^\star$.

\paragraph{A Deterministic and ``Successful'' Query Strategy.}
Using the fixed realization $\rho^\star$, define a deterministic adaptive query strategy $\sigma_S$ by simulating $\cL_{\mathrm{unif}}$.
Whenever the simulated learner requests its next example, the strategy reads the next point from the hardwired example stream encoded in $\rho^\star$, queries that point to the membership oracle, and then feeds the point to the simulated learner as its next example.
Whenever the simulated learner requests an ordinary membership query $x$, the strategy asks the same query $x$ and feeds back the answer.
Because every hardwired example lies in $S$, every target $h\in \V\!\inparen{S}$ answers $1$ on those compiled example queries.
Therefore, for every such $h$, the realized run of $\sigma_S$ against the membership oracle for $h$ is \textit{exactly} the run of $\cL_{\mathrm{unif}}$ with realization $\rho^\star$ against target $h$ and distribution $\mathcal U_S$.
By the choice of $q_d$, the total number of membership queries made by $\sigma_S$ is at most $q_d$.

\paragraph{A Weak Compression Scheme.}
We claim that $\sigma_S$ is a weak query-based compression scheme.
To see this, fix any response transcript $r$ of $\sigma_S$ that is realizable with respect to $\V\!\inparen{S}$, and let $h^\star\in \V\!\inparen{S}$ be a target realizing $r$.
Let $P_r\subseteq S$ be the set of hardwired example points consumed along this branch, let $q_r,a_r$ be the ordinary membership-query transcript seen by the simulated learner, and let $\wh g_r$ be the output of the simulation.
Then
\[
    \V\!\inparen{\sigma_S,r}
    =
    \V\!\inparen{P_r}\cap \V\!\inparen{q_r,a_r}.
\]
Now $\rho^\star\in \evG_{h^\star}$, so the corresponding run of $\cL_{\mathrm{unif}}$ satisfies the uniform guarantee over all potential targets.
Since $h^\star\in \V\!\inparen{P_r}\cap \V\!\inparen{q_r,a_r}$ and $S\subseteq \supp(h^\star)$, conditioning $\mathcal U_S$ on $\supp(h^\star)$ leaves $\mathcal U_S$ unchanged.
Therefore
\[
    \Pr_{x\sim \mathcal U_S}\inparen{\wh g_r(x)=0}
    \leq
    \frac23,
    \qquadand
    \supp\!\inparen{\wh g_r}
    \subseteq
    \bigcap_{h\in \V\!\inparen{P_r}\cap \V\!\inparen{q_r,a_r}} \supp(h).
\]
The first inequality implies $\abs{
        S\cap \supp\!\inparen{\wh g_r}
    }
    \geq
    \frac13\abs{S}.$
Further, the second inclusion yields
\[
    \bigcap_{h\in \V\!\inparen{\sigma_S,r}} \supp(h)
    =
    \bigcap_{h\in \V\!\inparen{P_r}\cap \V\!\inparen{q_r,a_r}} \supp(h)
    \supseteq
    \supp\!\inparen{\wh g_r}.
\]
Hence
\[
    \abs{
        S\cap \bigcap_{h\in \V\!\inparen{\sigma_S,r}} \supp(h)
    }
    \geq
    \abs{
        S\cap \supp\!\inparen{\wh g_r}
    }
    \geq
    \frac13\abs{S}.
\]
Thus, $\sigma_S$ is a weak query-based compression scheme for $S$.
\end{proof}
Now we are ready to prove \cref{thm:valiant:compression}.
\begin{proof}[Proof of \cref{thm:valiant:compression}]
Let $q_d$ be as in \cref{thm:valiant:weak-compression}.
Since the bound \eqref{eq:valiant:usefulQueries:size} is asymptotic, we may and do assume $q_d\geq 1$.
For each finite set $S\subseteq \zo^d$ with $\V\!\inparen{S}\neq\emptyset$, we define a deterministic adaptive query strategy $\Sigma_S$ recursively.

\paragraph{Recursive Construction.}
If $\abs{S}\leq q_d$, then $\Sigma_S$ queries every point of $S$ and stops.
If $\abs{S}>q_d$, then $\Sigma_S$ first runs the weak query-based compression scheme $\sigma_S$ from \cref{thm:valiant:weak-compression}.
For each response transcript $r$ of this first stage that is realizable with respect to $\V\!\inparen{S}$, define
\[
    I_{S,r}
    \coloneqq
    S\cap \bigcap_{h\in \V\!\inparen{\sigma_S,r}} \supp(h),
    \qquadand
    R_{S,r}
    \coloneqq
    S\setminus I_{S,r}.
\]
For such an $r$, any hypothesis in $\V\!\inparen{S}$ realizing $r$ also witnesses $\V\!\inparen{R_{S,r}}\neq\emptyset$.
On every such branch $r$, the strategy continues with the recursively defined strategy $\Sigma_{R_{S,r}}$.
On first-stage branches that are not realizable with respect to $\V\!\inparen{S}$, define the continuation arbitrarily, say by stopping.
This is well-defined because whenever $r$ is realizable with respect to $\V\!\inparen{S}$, the weak-compression guarantee gives $\abs{R_{S,r}}
    \leq
    \frac23\abs{S}
    <
    \abs{S},$
so the recursion strictly decreases the size of the unresolved set.

\paragraph{Correctness.}
We prove by induction on $\abs{S}$ that for every finite set $S$ with $\V\!\inparen{S}\neq\emptyset$, and every response transcript $r$ of $\Sigma_S$ that is realizable with respect to $\V\!\inparen{S}$, $S
    \subseteq
    \bigcap_{h\in \V\!\inparen{\Sigma_S,r}} \supp(h).$
\begin{itemize}[leftmargin=15pt]
    \item \textbf{Case A ($\abs{S}\leq q_d$):} If $\abs{S}\leq q_d$, then $\Sigma_S$ queries every point of $S$, so every hypothesis in $\V\!\inparen{\Sigma_S,r}$ answers $1$ on every point of $S$.
Thus the claim is immediate.
    \item \textbf{Case B ($\abs{S}> q_d$):}
        Now suppose $\abs{S}>q_d$, and assume the claim has already been proved for all smaller realizable sets.
Fix a realizable transcript $r$ of $\Sigma_S$.
Write
\[
    r = r_0\circ r_1,
\]
where $r_0$ is the first-stage transcript produced by $\sigma_S$.
Since $r$ is realizable with respect to $\V\!\inparen{S}$, there exists some $h^\star\in \V\!\inparen{S}$ realizing $r$.
In particular, $h^\star$ realizes $r_0$, so $r_0$ is realizable with respect to $\V\!\inparen{S}$.
Let
\[
    I\coloneqq I_{S,r_0},
    \qquadand
    R\coloneqq R_{S,r_0}.
\]
Because $h^\star\in \V\!\inparen{S}$, we have $S\subseteq \supp(h^\star)$ and therefore $R\subseteq \supp(h^\star)$.
Thus $\V\!\inparen{R}\neq\emptyset$, and by construction, the suffix $r_1$ is a realizable transcript of the recursive strategy $\Sigma_R$ with respect to $\V\!\inparen{R}$ and $\abs{R}<\abs{S}$ by the definition of the weak compression scheme $\sigma_S$.
Thus, by the induction hypothesis,
\[
    R
    \subseteq
    \bigcap_{h\in \V\!\inparen{\Sigma_R,r_1}} \supp(h).
\]
Now let $h$ be any hypothesis in $\V\!\inparen{\Sigma_S,r}$.
Since $h$ is already consistent with the first-stage transcript $r_0$, it lies in $\V\!\inparen{\sigma_S,r_0}$ and therefore contains $I$ by definition of $I$.
Also, since the continuation on branch $r_0$ is $\Sigma_R$, the same hypothesis $h$ lies in $\V\!\inparen{\Sigma_R,r_1}$, and therefore contains $R$ by the induction hypothesis.
Hence $h$ contains $I\cup R = S.$
Since this holds for every $h\in \V\!\inparen{\Sigma_S,r}$, the induction is complete.
\end{itemize}

\paragraph{Compression Rate.}
Each nonterminal stage uses at most $q_d$ queries and reduces the unresolved set by a factor of at most $\nfrac23$.
Starting from a nonempty set of size $n$, after $t
    \coloneqq
    \left\lceil \log_{3/2} n \right\rceil$
such stages, the unresolved set has size at most $n\inparen{\nfrac23}^{t}
    \leq
    1
    \leq
    q_d,$
so the next stage is terminal.
Hence the total number of queries made by $\Sigma_S$ is at most $q_d\cdot \inparen{1+\left\lceil \log_{3/2}\abs{S}\right\rceil}.$
Finally, \cref{thm:valiant:weak-compression} provides the desired upper bound on $q_d$, completing the proof.
\end{proof}

\subsubsection{Query-Based Compression Implies Learnability}
In this section, we prove the converse direction: query-based compression schemes yield learners in Valiant's model.
This mirrors the classical fact that sample compression implies PAC learnability \citep{LittlestoneWarmuth1986,floyd1995sample}. %

As we have seen, a natural hypothesis to output after compression is the closure of the realized transcript, and so the key object to control is the class of all sets that can arise as such closures.
\begin{definition}[$k$-Query Closure Class]
\label{def:valiant:kQueryClosureClass}
Fix $d,k\in\N$.
The \emph{$k$-query closure class} of $\hyH_d$ is
\[
    \operatorname{Cl}_{k}\!\inparen{\hyH_d}
    \coloneqq
    \inbrace{
        \bigcap_{h\in \V\!\inparen{q,r}} \supp(h)
        :
        q,r \text{ have the same length } t\leq k
        \text{ and }
        \V\!\inparen{q,r}\neq\emptyset
    }.
\]
\end{definition}
Each set in $\operatorname{Cl}_{k}\!\inparen{\hyH_d}$ is obtained by intersecting all hypotheses in $\hyH_d$ that remain consistent with a given transcript.
However, this is quite different from the usual $k$-fold intersection class
\[
    \inbrace{
        \supp(h_1)\cap\cdots\cap\supp(h_k)
        :
        h_1,\dots,h_k\in \hyH_d
    }.
\]
For ordinary $k$-fold intersections, the VC dimension is known to be $O\!\inparen{\vc{}\!\inparen{\hyH_d}k\log k}$, and this order is tight in general \citep{blumer1989learnability,eisenstat2007vcdimension}.
In contrast, a query closure intersects the entire version space after a transcript: the number of intersected concepts is not fixed in advance, can vary from branch to branch, and over infinite domains can even be infinite.

To gain some intuition, consider monotone conjunctions on $\zo^d$. If a transcript consists only of positive answers on points $x^{(1)},\dots,x^{(t)}$, then the corresponding closure is the upward-closed set determined by the coordinates that are $1$ in every $x^{(i)}$.
So even a short transcript can encode a closure obtained by intersecting many hypotheses at once.
Despite this, the situation remains manageable on the finite domain $\zo^d$ because the entire closure class is indexed by transcripts, and we can bound the number of distinct transcripts of length at most $k$.
Later we will see that over general domains the behavior can be much wilder; in particular, even one-point query closures can have infinite VC dimension despite the base class having VC dimension at most two; see \cref{thm:onePointClosureInfiniteVC}.

\begin{lemma}
\label{lem:valiant:kQueryClosureClass:VC}
For every $d,k\in\N$,
    $\abs{\operatorname{Cl}_{k}\!\inparen{\hyH_d}}
    \leq
    2^{(d+1)(k+1)}.$
Consequently,
\[
    \vc\!\inparen{\operatorname{Cl}_{k}\!\inparen{\hyH_d}}
    \leq
    (d+1)(k+1).
\]
\end{lemma}
\begin{proof}[Proof of \cref{lem:valiant:kQueryClosureClass:VC}]
For each $t\in\inbrace{0,1,\dots,k}$, a labeled query transcript of length $t$ is a sequence of $t$ pairs $(q_i,r_i)\in \zo^d\times \zo$.
Hence there are at most
    $\abs{\zo^d\times\zo}^{\,t}
    =
    \inparen{2^d\cdot 2}^{t}
    =
    2^{(d+1)t}$
possible transcripts of length $t$.
Each transcript determines at most one element of $\operatorname{Cl}_{k}\!\inparen{\hyH_d}$, so
\[
    \abs{\operatorname{Cl}_{k}\!\inparen{\hyH_d}}
    \leq
    \sum_{t=0}^{k} 2^{(d+1)t}
    \leq
    2^{(d+1)(k+1)}.\qedhere
\] 
\end{proof}
Before stating our result, we need to define the compression rate of a query-based compression scheme:
\begin{definition}[Compression Rate]
    \label{def:compression-rate}
    Fix a concept class $\hyH=(\hyH_1,\hyH_2,\dots)$.
    If for each $S$ with $\V(S)\neq \emptyset$, there is a query-based compression scheme $\Sigma$ for $S$ that makes at most $k=\kappa(\abs{S})$ membership queries, we say that the compression rate for $\hyH$ is $\kappa$.
\end{definition}
\vspace{-7mm}
\begin{theorem}[Query-Based Compression Implies Learnability]
\label{thm:valiant:learningFromCompression}
\label{thm:valiant:compression-implies-learnability}
Suppose that for each $d\in\N$, the class $\hyH_d$ has a nondecreasing query-based compression rate $\kappa_d\colon \N\to\N$ in the sense of \cref{def:compression-rate}.
Then there is a learner with the following guarantee.
For every $\eps,\delta\in(0,1)$, every target $h^\star\in\hyH_d$, and every distribution $\cD$ compatible with $h^\star$, for any
\[
    n
    \geq
    \Omega\inparen{
    \frac{
        (d+1)\inparen{\kappa_d\!\inparen{n}+1}
        +
        \log\nfrac{1}{\delta}
    }{\eps}
    }
\]
the learner uses exactly $n$ example-oracle calls and at most $\kappa_d\!\inparen{n}$ membership queries and, with probability at least $1-\delta$, outputs a hypothesis satisfying \cref{def:valiant}.
In particular,
\[
    m\!\inparen{d,\eps,\delta}
    =
    n+\kappa_d\!\inparen{n}
\]
is a valid sample/query complexity bound.
In particular, if $\kappa_d\!\inparen{n}
    =
    q_d\cdot \inparen{1+\left\lceil \log_{3/2} n \right\rceil},$ then
\[
    m\!\inparen{d,\eps,\delta}
    =
    O\!\inparen{
        \frac{d q_d+\log\nfrac1\delta}{\eps}
        \cdot
        \log\inparen{
            \frac{d q_d+\log\nfrac1\delta}{\eps}
        }
    }.
\]
\end{theorem}
The learner in the above result first compresses the sample $S$ and outputs the closure of the realized transcript.
Then two properties are immediate: the closure has no false positives (because the true target is consistent with its own transcript), and it has zero empirical false-negative error (because the compression scheme certifies that every sampled point lies in the closure).
The proof therefore reduces to a uniform convergence argument over the closure class $\operatorname{Cl}_{k}\!\inparen{\hyH_d}$, which holds due to \cref{lem:valiant:kQueryClosureClass:VC}.

\begin{proof}[Proof of \cref{thm:valiant:learningFromCompression}]
Fix $d\in\N$, $\eps,\delta\in(0,1)$, a target $h^\star\in\hyH_d$, and a distribution $\cD$ compatible with $h^\star$.
Let $n$ be any integer satisfying the displayed lower bound in the theorem, draw an ordered sample $x_1,\dots,x_n\sim \cD,$
and let $S\coloneqq \inbrace{x_1,\dots,x_n}$ be the underlying set of distinct sample points.
Since $\cD$ is compatible with $h^\star$ (\cref{def:compatible}), every sample point is positive for $h^\star$, and hence $h^\star\in \V\!\inparen{S}$.

\paragraph{Learning Algorithm.}
The learning algorithm runs a query-based compression scheme $\Sigma_S$ for $S$ of size at most $\kappa_d\!\inparen{\abs{S}}\leq \kappa_d\!\inparen{n}$ against the membership oracle for $h^\star$.
Let $r_S$ be the realized response transcript.
It outputs the indicator of the set
\[
    C_S
    \coloneqq
    \bigcap_{h\in \V\!\inparen{\Sigma_S,r_S}} \supp(h).
\]

\paragraph{One-Sided Error.}
Since $h^\star$ generated the transcript $r_S$, we have $h^\star\in \V\!\inparen{\Sigma_S,r_S}$ and hence
    $C_S
    \subseteq
    \supp\!\inparen{h^\star}.$
So the output has no false positives.
Also, because $\Sigma_S$ is a query-based compression scheme for $S$,
    $S
    \subseteq
    C_S.$
Therefore every sample point lies in both $C_S$ and $\supp\!\inparen{h^\star}$, and so the empirical symmetric-difference error of $C_S$ relative to $\supp\!\inparen{h^\star}$ is zero.

\paragraph{Bound on False-Negative Rate.}
Because $\Sigma_S$ makes at most $\kappa_d\!\inparen{n}$ queries, we have
$C_S\in \operatorname{Cl}_{\kappa_d(n)}\!\inparen{\hyH_d}.$
Let
\[
    k_n\coloneqq \kappa_d\!\inparen{n}
    \qquadand
    \hyF_{d,n}
    \coloneqq
    \inbrace{
        C\triangle \supp\!\inparen{h^\star}
        :
        C\in \operatorname{Cl}_{k_n}\!\inparen{\hyH_d}
    }.
\]
Then $\abs{\hyF_{d,n}}=\abs{\operatorname{Cl}_{k_n}\!\inparen{\hyH_d}}$.
By the second-order uniform convergence bound from \cite[Theorem~5.1 and Eq.~(15)]{Boucheron_Bousquet_Lugosi_2005}, with probability at least $1-\delta$, every set $F\in \hyF_{d,n}$ with zero empirical measure on $x_1,\dots,x_n$ satisfies
\[
    \cD\!\inparen{F}
    \leq
    \frac{O\!\inparen{\log\abs{\hyF_{d,n}}+\log\nfrac{1}{\delta}}}{n}.
\]
Applying this to
\[
    F_S\coloneqq C_S\triangle \supp\!\inparen{h^\star}
\]
and using \cref{lem:valiant:kQueryClosureClass:VC}, we obtain with probability at least $1-\delta$ that
\[
    \cD\!\inparen{C_S\triangle \supp\!\inparen{h^\star}}
    \leq
    \frac{O\!\inparen{\log\abs{\operatorname{Cl}_{k_n}\!\inparen{\hyH_d}}+\log\nfrac{1}{\delta}}}{n}
    \leq
    \frac{O\!\inparen{(d+1)(k_n+1)\log 2+\log\nfrac{1}{\delta}}}{n}
    \leq
    \eps,
\]
where the last inequality holds by the choice of $n$.
Since $C_S\subseteq \supp\!\inparen{h^\star}$ and $\supp\!\inparen{\cD}\subseteq \supp\!\inparen{h^\star}$,
\[
    \Pr_{x\sim \cD}\inparen{C_S(x)=0 \text{ and } h^\star(x)=1}
    =
    \cD\!\inparen{\supp\!\inparen{h^\star}\setminus C_S}
    =
    \cD\!\inparen{C_S\triangle \supp\!\inparen{h^\star}}
    \leq
    \eps.
\]
Thus, with probability at least $1-\delta$, the output satisfies the guarantees in \cref{def:valiant}.
The learner uses exactly $n$ example-oracle calls and at most $\kappa_d\!\inparen{n}$ membership queries, proving the first part of the theorem.

\paragraph{Logarithmic Compression Rate.}
Suppose now $\kappa_d\!\inparen{n}
    =
    q_d\cdot \inparen{1+\left\lceil \log_{3/2} n \right\rceil}.$
Let
\[
    A
    \coloneqq
    \frac{d\inparen{q_d+1}+\log\nfrac1\delta}{\eps},
\]
and choose $n
    \coloneqq
    \left\lceil c' A\log\inparen{c'A} \right\rceil$
for a sufficiently large absolute constant $c'>0$.
Then
\[
    \kappa_d\!\inparen{n}+1
    \leq
    c''\inparen{q_d+1}\log\!\inparen{n+1}
\]
for an absolute constant $c''$, and hence
\[
    \frac{(d+1)\inparen{\kappa_d(n)+1}+\log\nfrac1\delta}{\eps}
    \leq
    c'''\inparen{A\log\!\inparen{n+1}+A}
\]
for another absolute constant $c'''$.
Since $\log\!\inparen{n+1}=O\!\inparen{\log A}$ for this choice of $n$, the right-hand side is at most $n/c$ once $c'$ is chosen large enough relative to $c,c'',c'''$.
Thus the displayed condition of the theorem holds.
Moreover,
\[
    \kappa_d\!\inparen{n}
    =
    O\!\inparen{\inparen{q_d+1}\log A}
    \leq
    O\!\inparen{A\log A}
    =
    O\!\inparen{n},
\]
so both the example-oracle and membership-query complexities are at most
\[
    O\!\inparen{n}
    =
    O\!\inparen{
        \frac{d\inparen{q_d+1}+\log\nfrac1\delta}{\eps}
        \cdot
        \log\inparen{
            \frac{d\inparen{q_d+1}+\log\nfrac1\delta}{\eps}
        }
    }.\qedhere
\]
\end{proof}

\subsubsection{A VC Lower Bound for Learnability in Valiant's Model}
In this section, we show that learnability in Valiant's model requires sample/query complexity at least linear in the VC dimension.

\begin{theorem}[VC Lower Bound for Valiant's Model]
\label{thm:valiant:VCLowerBound}
Suppose $\hyH=\inparen{\hyH_1,\hyH_2,\dots}$ is learnable in Valiant's model with sample/query complexity
$m\!\inparen{d,\eps,\delta}$.
Then for every $d\in\N$ and every $\eps,\delta\in\inparen{0,\nfrac12}$,
\[
    m\!\inparen{d,\eps,\delta}
    \geq
    \frac{\vc\!\inparen{\hyH_d}-1}{9}.
\] 
In particular, if $\hyH$ is learnable in Valiant's model, then $\vc\!\inparen{\hyH_d}=\poly(d)$.
\end{theorem}

\begin{proof}[Proof of \cref{thm:valiant:VCLowerBound}]
Fix $d\in\N$, $\eps\in\inparen{0,\nfrac12}$, and $\delta\in\inparen{0,\nfrac12}$.
Set
\[
    M \coloneqq m\!\inparen{d,\eps,\delta},
    \qquad
    k \coloneqq 4M+1,
    \qquadand
    n \coloneqq 2k+M = 9M+2.
\]
Suppose toward a contradiction that
\[
    \vc\!\inparen{\hyH_d}\geq n.
\]
Then there exists a set $X\subseteq \zo^d$ of size $n$ shattered by $\hyH_d$.

\paragraph{Hard Prior.}
For each subset $U\subseteq X$ of size $k$, let $h_U\in \hyH_d$ be a hypothesis satisfying
\[
    h_U(x)=\one\!\inbrace{x\in U}
    \qquad\text{for every }x\in X.
\]
(This hypothesis exists since $X$ is VC-shattered.)
Let $\cD_U$ be the uniform distribution on $U$; then $\cD_U$ is compatible with $h_U$.
Now choose $U$ uniformly at random among all $k$-subsets of $X$, and let the target/distribution pair be $\inparen{h_U,\cD_U}$.
By Yao's minimax principle, it suffices to show that every deterministic learner that makes at most $M$ example-oracle calls and at most $M$ membership queries fails with probability at least $\nfrac12$ against this prior.

Fix such a deterministic learner $\cL$.
Padding with ignored queries if necessary, we may assume that $\cL$ makes exactly $M$ example-oracle calls and exactly $M$ membership queries.
We may also assume that all $M$ example-oracle samples are revealed before any membership query is asked, since this can only help the learner.

Let $P=\inparen{x_1,\dots,x_M}\in X^M$
be the ordered sample sequence, and let $\overline{P}\coloneqq \inbrace{x_1,\dots,x_M}$
be the corresponding set of distinct sampled points.
Fix an arbitrary realization $P_0\in X^M$, and let $r\coloneqq \abs{\overline{P_0}}.$
Conditioned on $P=P_0$, the hidden set $U$ is uniform over
$\inbrace{
        V\subseteq X : \abs{V}=k \text{ and } \overline{P_0}\subseteq V
    }.$
Indeed, for every $V\subseteq X$ of size $k$,
$\Pr\inparen{P=P_0\mid U=V}
    =
    k^{-M}\cdot \one\!\inbrace{\overline{P_0}\subseteq V},$
and the prior on $U$ is uniform.
Hence, conditioned on $P=P_0$, there are exactly
    $\binom{n-r}{k-r}$
possible hidden sets given $P=P_0$.
Since $r\leq M$ and
\[
    \frac{\binom{n-r-1}{k-r-1}}{\binom{n-r}{k-r}}
    =
    \frac{k-r}{n-r}
    <1,
\]
this quantity is minimized when $r=M$.
Therefore, for every $P_0$,
\[
    \#\inbrace{\text{possible }U\text{'s given }P=P_0}
    \geq
    \binom{n-M}{k-M}
    =
    \binom{8M+2}{3M+1}.
    \yesnum\label{eq:valiant:VCLB:candidate-count}
\]

\paragraph{Upper Bound on the Number of Hidden Sets per Transcript.}
Now fix the sample sequence $P_0$.
Once the sample is fixed, the query phase produces one of at most $2^M$ response transcripts.
For each such response transcript $r\in\zo^M$, let
    $H_{P_0,r}\subseteq X$
denote the set of points in $X$ that the learner's output labels as positive after seeing $P_0$ and the response string $r$.

We claim that for any fixed response transcript $r$, the learner can succeed on at most
    $\binom{7M+1}{2M}$
different hidden sets $U$.
Indeed, suppose the learner succeeds on some hidden set $U$ of size $k$ after receiving response transcript $r$.
Since the output has no false positives, we must have
    $H_{P_0,r}\subseteq U.$
Since $\cD_U$ is uniform on $U$ and $\eps\leq \nfrac12$, the coverage guarantee implies
\[
    \abs{U\setminus H_{P_0,r}}
    \leq
    \eps k
    \leq
    \frac{k}{2}
    <
    2M+1.
\]
Since $\abs{U}=k=4M+1$, it follows that
\[
    \abs{H_{P_0,r}}
    \geq
    k-2M
    =
    2M+1.
\]
Therefore every successful hidden set $U$ must contain $H_{P_0,r}$, which has size at least $2M+1$. 
The number of $k$-subsets of $X$ that contain a fixed $(2M+1)$-element subset is at most
$\binom{n-(2M+1)}{k-(2M+1)}
    =
    \binom{7M+1}{2M}.$
This proves the claim.

Summing over all $2^M$ response transcripts, the total number of hidden sets on which the learner can succeed given $P=P_0$ is at most
\[
    2^M\binom{7M+1}{2M}
    \yesnum\label{eq:valiant:VCLB:success-count}
\] 
It remains to compare the candidate count \eqref{eq:valiant:VCLB:candidate-count} with the success count \eqref{eq:valiant:VCLB:success-count}.
Direct calculation shows:
\[
    \frac{\binom{8M+2}{3M+1}}{\binom{7M+1}{2M}}
    =
    \frac{(8M+2)!\,(2M)!}{(3M+1)!\,(7M+1)!}
    =
    \prod_{j=0}^{M}\frac{7M+2+j}{2M+1+j}.
\]
Each factor in the product is at least $2$, so $\frac{\binom{8M+2}{3M+1}}{\binom{7M+1}{2M}}
    \geq
    2^{M+1}.$
Therefore, $2^M\binom{7M+1}{2M}
    \leq
    \frac12\binom{8M+2}{3M+1}.$
    
That is, conditioned on any realization $P_0$, the learner succeeds on at most half of the candidate hidden sets.
Therefore
    $\Pr\inparen{\text{$\cL$ succeeds}\mid P=P_0}\leq \nfrac12$
for every $P_0$, and averaging over $P$ gives
    $\Pr\inparen{\text{$\cL$ succeeds}}\leq \nfrac12.$
So every deterministic learner with budget $M$ fails with probability at least $\nfrac12$ under the hard prior.
This contradicts learnability in Valiant's model, since $\delta<\nfrac12$.
Therefore our assumption $\vc\!\inparen{\hyH_d}\geq 9M+2$ was false, and so
    $\vc\!\inparen{\hyH_d}
    \leq
    9M+1
    =
    9\,m\!\inparen{d,\eps,\delta}+1.$
Applying this with $\eps=\delta=\nfrac13$: if $\hyH$ is learnable in Valiant's model, then
$m\!\inparen{d,\nfrac13,\nfrac13}=\poly(d)$,
and hence
    $\vc\!\inparen{\hyH_d}=\poly(d).$
\end{proof}

\subsubsection{Completing the Characterization of Learnability in Valiant's Model}
We now combine the previous sections to obtain a complete characterization of learnability in Valiant's model.
In particular, we prove \cref{thm:characterization}, which we restate below.
\thmCharacterization*
\begin{proof}[Proof of \cref{thm:characterization}]
We prove the two implications separately.

\paragraph{{(1)} $\Rightarrow$ {(2)}:}
Assume that $\hyH$ is learnable, and let
$m_{\mathrm W}\!\inparen{d,\eps,\delta}$
denote the sample/query complexity of some learner witnessing this.
By \cref{thm:valiant:VCLowerBound}, applied with $\eps=\delta=\nfrac13$, we have $\vc\!\inparen{\hyH_d}=\poly(d).$
Moreover, \cref{thm:valiant:compression} implies that for every nonempty finite set
$S\subseteq \zo^d$ with $\V\!\inparen{S}\neq\emptyset$, there is a query-based compression scheme for $S$ of size
\[
    m_{\mathrm W}\!\inparen{d,\nfrac{1}{18},\nfrac13}
    \cdot
    O\!\inparen{
        d\,\vc\!\inparen{\hyH_d}\,\log\!\inparen{\abs{S}+1}
    }.
\]
Since $\abs{S}\leq 2^d$, the factor $\log\!\inparen{\abs{S}+1}$ is $O(d)$.
Since both
$m_{\mathrm W}\!\inparen{d,\nfrac{1}{18},\nfrac13}$
and
$\vc\!\inparen{\hyH_d}$
are polynomial in $d$, the displayed bound is $\poly(d)$; this completes the proof of (2).

\paragraph{{(2)} $\Rightarrow$ {(1)}:}
Assume now that Condition \emph{(2)} holds, witnessed by a polynomial $q$.
For each $d\in\N$, define $\kappa_d(n)\coloneqq q(d)$ (for all $n\in \N$).
Then $\kappa_d$ is a nondecreasing query-based compression rate for $\hyH_d$.
By \cref{thm:valiant:learningFromCompression}, $\hyH$ is learnable in Valiant's model with sample/query complexity
\[
    O\!\inparen{
        \frac{dq(d)+\log\nfrac1\delta}{\eps}
        \cdot
        \log\!\inparen{
            \frac{dq(d)+\log\nfrac1\delta}{\eps}
        }
    }.\qedhere
\]
\end{proof}

\section{Proofs of Remaining Results}
 In this section, we present the proofs of our remaining results.
\subsection{Proof of \cref{cor:intro:positiveOnlySpecialCase} (Positive-Only Learning as the Non-Interactive Special Case)}
    \label{sec:proofof:cor:intro:positiveOnlySpecialCase}
    In this section, we prove \cref{cor:intro:positiveOnlySpecialCase}, which we restate below.
    \positiveOnlySpecialCase*
\begin{proof}[Proof of \cref{cor:intro:positiveOnlySpecialCase}] 
Let $\hyC_d$ denote the intersection closure of $\hyH_d$.
For $R\subseteq S$, write
$I(R)\coloneqq \bigcap_{h\in \V\!\inparen{R}} \supp(h).$
Choose an inclusion-minimal set $T\subseteq S$ such that
$ I(T)=I(S).$

We first construct a query compression scheme.
To this end, query every point of $T$, nonadaptively.
Since $\V\!\inparen{S}\neq\emptyset$, there exists $h^\star\in \V\!\inparen{S}$, and because $T\subseteq S$, every query in $T$ is answered positively.
After this transcript, the consistent hypotheses are exactly $\V\!\inparen{T}$, so the certified set is
\[
    \bigcap_{h\in \V\!\inparen{T}} \supp(h)
    =
    I(T)
    =
    I(S).
\]
Moreover, since every $h\in \V\!\inparen{S}$ labels every point of $S$ positively, $S\subseteq I(S)=I(T).$
Hence these queries indeed certify $S$.
Thus $T$ defines a deterministic query compression scheme for $S$ whose queries are nonadaptive and all belong to $S$.

Now fix any center hypothesis $c\in \V\!\inparen{S},$
which exists since $\V\!\inparen{S}\neq\emptyset$.
We claim that $T$ is a $1$-centered star set for $\hyH_d$, and hence also for $\hyC_d$.

To prove this, fix any $t\in T$.
By minimality of $T$, we have $I\!\inparen{T\setminus\inbrace{t}} \neq I(T).$
In particular, $\V\!\inparen{T\setminus\inbrace{t}} \neq \V\!\inparen{T},$
since equality of these two version spaces would force equality of the corresponding intersections.
Therefore we may choose $h_t \in \V\!\inparen{T\setminus\inbrace{t}} \setminus \V\!\inparen{T}.$
Then $h_t(t')=1$ for all $t'\in T\setminus\inbrace{t},$
while necessarily $h_t(t)=0,$
because otherwise $h_t$ would belong to $\V\!\inparen{T}.$
Since also $c(t')=1$ for every $t'\in T$, this proves that $T$ is a $1$-centered star set for $\hyH_d$.

Now $\hyH_d\subseteq \hyC_d$, so the same witnesses show that $T$ is a $1$-centered star set for $\hyC_d$.
Because $\hyC_d$ is closed under intersections, \cref{thm:oneStarDimension} (Theorem~19 of \citet{hanneke2024Star}) implies $\abs{T}
    \leq
    \oneStar\!\inparen{\hyC_d}
    \leq
    \vc\!\inparen{\hyC_d}.$
Therefore the above query compression scheme has size at most $\vc\!\inparen{\hyC_d}$.
\end{proof}

\subsection{Proof of \cref{thm:implications} (Sandwich between PAC and Positive-Only Learning)}
    \label{sec:proofof:thm:implications}
    In this section, we prove \cref{thm:implications}, which we restate below.
    \implicationscor*
    \noindent Below, we state the equivalent version of this theorem using the notation of one-centered star dimension, which we prove below.
    \begin{restatable}[]{corollary}{corollaryValiantBasicConsequences}
\label{cor:valiant:basicConsequences:old}
The following hold for a sequence of concept classes $\hyH=\inparen{\hyH_1,\hyH_2,\dots}$ in Valiant's model.
\begin{enumerate}[itemsep=0pt]
    \item If $\vc{}\!\inparen{\hyH_d}=d^{\omega(1)}$ for infinitely many $d$, then $\hyH$ is not learnable in Valiant's model.
    In particular, if $\vc{}\!\inparen{\hyH_d}\geq 2^{\Omega(d)}$ for infinitely many $d$, then $\hyH$ is not learnable.
    \item If $\oneStar\!\inparen{\hyH_d}=2^d$ for infinitely many $d$, then $\hyH$ is not learnable in Valiant's model.
    In particular, there exists $\hyH$ with $\vc{}\!\inparen{\hyH_d}=1$ for each $d\in\N$ that is not learnable.
    \item There exists a sequence of classes that is learnable even though, for infinitely many $d$, $\oneStar\!\inparen{\hyH_d}=d^{\omega(1)}$.
\end{enumerate}
\end{restatable}
\begin{proof}[Proof of \cref{cor:valiant:basicConsequences:old}]
We prove Items~1 to 3 in order.

\smallskip

\noindent\textit{Proof of Item~1.}
Suppose $\hyH$ is learnable in Valiant's model.
Now fix $d$, and let $B\subseteq \zo^d$ be shattered by $\hyH_d$, with $\abs{B}=m$.
For each $T\subseteq B$, let $h_T\in \hyH_d$ satisfy $h_T(x)=\mathds{1}\inbrace{x\in T}$ for $x\in B$.
Since $T$ is realizable, it admits a query compression scheme of size at most $q(d)$; let $\inparen{q_T,r_T}$ be the query-response transcript realized when this scheme is run against $h_T$.
Then, on the one hand, as the scheme certifies $T$, $T
    \subseteq
    B\cap \bigcap_{h\in \V\!\inparen{q_T,r_T}} \supp(h).$
On the other hand, $h_T\in \V\!\inparen{q_T,r_T}$, and hence
$B\cap \bigcap_{h\in \V\!\inparen{q_T,r_T}} \supp(h)
    \subseteq
    B\cap \supp(h_T)
    =
    T.$
Thus $T
    =
    B\cap \bigcap_{h\in \V\!\inparen{q_T,r_T}} \supp(h),$
so distinct subsets $T\subseteq B$ yield distinct query-response transcripts.

Finally, the number of query-response transcripts of length at most $q(d)$ is at most $2^{O(d\cdot q(d))}.$
Since there are $2^m$ choices of $T\subseteq B$, it follows that $2^m\leq 2^{\poly(d)}$, and hence $m\leq \poly(d)$.
Therefore $\vc\!\inparen{\hyH_d}\leq \poly(d)$ for all $d$.
This proves Item~1 by contraposition.

\medskip
\noindent
\textit{Proof of Item~2.}
Suppose for contradiction that $\hyH$ is learnable in Valiant's model.
By \cref{thm:characterization}, there exists a polynomial $q$ such that every realizable set $S\subseteq \zo^d$ admits a query compression scheme of size at most $q(d)$.
Since $\oneStar\!\inparen{\hyH_d}=2^d$ for infinitely many $d$, we may fix such a $d$ large enough that $q(d)<2^d$.
Then $\hyH_d$ $1$-star shatters $\cX=\zo^d$.
Let $g\in \hyH_d$ be the all-one center of this $1$-star, and set $S=\zo^d$.
Consider any query compression scheme $\Sigma$ for $S$, and let $r$ be the transcript realized when $\Sigma$ is run against $g$.
If some point $x\in \zo^d$ is not queried, then the corresponding star witness $h_x$ lies in $\V\!\inparen{\Sigma,r}$, since $h_x$ agrees with $g$ on every queried point and differs only at $x$.
But $h_x(x)=0$, so $x\notin \bigcap_{h\in \V\!\inparen{\Sigma,r}} \supp(h),$ contradicting that $\Sigma$ certifies $S=\zo^d$.
Thus every query compression scheme for $S$ must query every point of $\zo^d$ and so have size at least $2^d$, contradicting $q(d)<2^d$.
This proves Item~2.

\medskip
\noindent
\textit{Proof of the final sentence in Item~2.}
For each $d\in\N$, define $g_d(x)\coloneqq 1$ for each $x\in \zo^d$
and for each $v\in \zo^d$, define $h^{(v)}(x)\coloneqq \one\!\inbrace{x\neq v}.$
Let
\[
    \hyH_d
    \coloneqq
    \inbrace{g_d}\cup \inbrace{h^{(v)}: v\in \zo^d}.
\]
Then every singleton $\inbrace{v}$ is shattered by $\hyH_d$, since $g_d(v)=1$ and $h^{(v)}(v)=0$, so $\vc\!\inparen{\hyH_d}\geq 1$.
On the other hand, no two-point set $\inbrace{u,v}$ can be shattered, since every hypothesis in $\hyH_d$ is negative on at most one point of $\zo^d$.
Thus $\vc\!\inparen{\hyH_d}=1$ for every $d$.

Moreover, $\zo^d$ is a $1$-star for $\hyH_d$ centered at $g_d$, with witnesses $\sinbrace{h^{(v)}:v\in\zo^d}$, so $\oneStar(\hyH_d)=2^d$.
Moreover, for every $A\subseteq \zo^d$,
\[
    A
    =
    \bigcap_{v\in \zo^d\setminus A} \supp\!\inparen{h^{(v)}},
\]
so the intersection closure $\overline{\hyH}_d$ contains every subset of $\zo^d$.
Therefore Item~2 applies, and the sequence $\hyH=\inparen{\hyH_1,\hyH_2,\dots}$ is not learnable in Valiant's model.

\medskip
\noindent
\textit{Proof of Item~3.}
For each $d\geq 1$, let $N_d\coloneqq 2^{d-1}$, and identify $\zo^d$ with $\insquare{N_d}\times \zo$ via the first $d-1$ bits and the last bit.
For each $i\in \insquare{N_d}$, define $h_i^{(d)}\colon \zo^d\to \zo$ by
\[
    h_i^{(d)}(x,0)\coloneqq \one\!\inbrace{x\neq i}
    \qquadand\qquad
    h_i^{(d)}(x,1)\coloneqq \one\!\inbrace{x\geq i}.
\]
Set $\hyH_d\coloneqq \inbrace{h_i^{(d)}: i\in \insquare{N_d}}.$
We claim that the sequence $\hyH=\inparen{\hyH_1,\hyH_2,\dots}$ witnesses Item~3.

We first show that $\hyH$ is learnable.
Fix $d$, a target $h_{i^\star}^{(d)}\in \hyH_d$, and a distribution $\cD$ compatible with $h_{i^\star}^{(d)}$.
Let $\cD_0$ denote the marginal of $\cD$ on the first coordinate.
Given $\eps,\delta\in (0,1)$, the learner draws
\[
    m = O\!\inparen{\frac{\log\!\nfrac1\eps+\log\!\nfrac1\delta}{\eps}}
\]
independent examples $(x_1,y_1),\dots,(x_m,y_m)\sim \cD$, lets $Z\coloneqq \inbrace{x_1,\dots,x_m}\subseteq \insquare{N_d}$, writes the distinct values of $Z$ as $z_{(1)}<\cdots<z_{(r)}$, and adds sentinels $z_{(0)}\coloneqq 0$ and $z_{(r+1)}\coloneqq N_d+1$.
By the $\eps$-net theorem for intervals, with probability at least $1-\delta$ every interval $J\subseteq \insquare{N_d}$ with $\cD_0(J)\geq \eps$ contains at least one sampled coordinate; fix a realization for which this holds.
The learner then uses membership queries on the slice $\insquare{N_d}\times\inbrace{1}$.
Since $h_{i^\star}^{(d)}(x,1)=\one\!\inbrace{x\geq i^\star}$, a binary search over $z_{(1)},\dots,z_{(r)}$ finds the smallest index $t\in\inbrace{1,\dots,r+1}$ such that $h_{i^\star}^{(d)}\!\sinparen{z_{(t)},1}=1$, with the convention $t=r+1$ if all answers are $0$.
Setting $L\coloneqq z_{(t-1)}$ and $R\coloneqq z_{(t)}$, we have $i^\star\in (L,R]$.
If $R\leq N_d$, the learner asks one additional membership query at $(R,0)$.
If the answer is $0$, then $R=i^\star$, so it outputs the exact target hypothesis $h_R^{(d)}$.
Otherwise, or if $R=N_d+1$, let $J\coloneqq (L,R)\cap \insquare{N_d}$ and output the hypothesis $\wh h$ defined by
\[
    \wh h(x,0)=\one\!\inbrace{x\notin J}
    \qquadand\qquad
    \wh h(x,1)=\one\!\inbrace{x\geq R},
\]
where $\one\!\inbrace{x\geq N_d+1}=0$ by convention.
A direct verification shows that $\wh h$ has no false positives and that every false negative has first coordinate in $J$.
Since $J$ contains no sampled coordinate, the $\eps$-net property gives $\cD_0(J)<\eps$, so the false-negative probability is less than $\eps$.
Thus $\hyH$ is learnable in Valiant's model with sample complexity $O\!\inparen{(\log(1/\eps)+\log(1/\delta))/\eps}$ and $O\!\inparen{\log(m+1)}$ membership queries.

We now show that $\oneStar\!\inparen{\hyH_d}$ is superpolynomial.
Fix any $k\leq N_d-1$ and let $S_k\coloneqq \inbrace{(1,0),\dots,(k,0)}$.
Then $h_{k+1}^{(d)}$ labels every point of $S_k$ by $1$, while for each $t\in\insquare{k}$ the hypothesis $h_t^{(d)}$ agrees with $h_{k+1}^{(d)}$ on $S_k\setminus \inbrace{(t,0)}$ and labels $(t,0)$ by $0$.
Thus $S_k$ is a $1$-star for $\hyH_d$.
Taking $k=N_d-1$, we obtain
\[
    \oneStar\!\inparen{\hyH_d}\geq N_d-1 = 2^{d-1}-1.
\]
Hence $\oneStar\!\inparen{\hyH_d}=d^{\omega(1)}$ for infinitely many $d$, proving Item~3.
\end{proof}

\subsection{Proof of \cref{thm:intro:necessary-sufficient} (Sandwich between PAC and Positive-Only on General Domain)}
\label{sec:proofof:thm:intro:necessary-sufficient} 
In this section, we prove \cref{thm:intro:necessary-sufficient}, which we restate below.
\theoremNecessarySufficient* 
\noindent Below, we state the equivalent version of this theorem using the notation of one-centered star dimension, which we prove below.
\begin{restatable}[]{theorem}{theoremSeparationTwo}
    \label{thm:necessary-sufficient}
    The following hold in Valiant's model over a general domain (\cref{def:valiant:general}):
    \begin{enumerate}[itemsep=0pt,leftmargin=14pt]
        \item If $\hyH$ is learnable, then $\vc\!\inparen{\hyH}
            \leq
            9\,m\!\inparen{\nfrac13,\nfrac13}+1$. %
        \item 
        If $\abs{\cX}=\infty$ and $\cX$ is a $1$-star for $\hyH$, then $\hyH$ is not learnable in Valiant's model.
        \item 
        If $\oneStar\!\inparen{\hyH}<\infty$, then $\hyH$ is learnable with $m(\eps,\delta)\leq O\!\inparen{
                \frac{
                    \oneStar\!\inparen{\hyH}\log\!\nfrac{1}{\eps}
                    +
                    \log\!\nfrac{1}{\delta}
                }{\eps}
            }.$
    \end{enumerate}
\end{restatable} 

\begin{proof}[Proof sketch of Part I of \cref{thm:necessary-sufficient}]
    This is an immediate corollary of the hard-prior argument as in \cref{thm:valiant:VCLowerBound}.
\end{proof} 
\vspace{-7mm}
\begin{remark}
    This, in particular, implies that finite VC dimension ($\vc\!\inparen{\hyH}<\infty$) is necessary for learning $\hyH$.
    This special result also follows from Theorem~3 of \citet{kulkarni1993active}, who show that infinite VC dimension rules out learnability even in a stronger model where the learner knows the feature distribution exactly and may ask arbitrarily many binary queries.
    For the quantitative bound above, we required the new proof presented above.
\end{remark}

\vspace{-5mm}

\begin{proof}[Proof of Part II of \cref{thm:necessary-sufficient}] 
This is the same conditioning-on-the-sample/counting argument as in \cref{thm:valiant:VCLowerBound}, except that the hidden object is now a single index rather than a $k$-subset.
Because $\cX$ is a $1$-star for $\hyH$, for every $x\in \cX$, there exists a concept
\[
    h^{(x)}(y)=\one\!\inbrace{y\neq x}
    \qquad\text{for every }y\in \cX.
\]
Fix $\eps,\delta\in\inparen{0,\nfrac12}$ and suppose, toward a contradiction, that $\hyH$ is learnable with sample/query complexity $M\coloneqq m\!\inparen{\eps,\delta}.$
Choose a large enough finite set $X_n=\inbrace{x_1,\dots,x_n}\subseteq \cX$ such that
\[
    \frac{M+\eps(n-1)+1}{n-M}<\frac12.
    \yesnum\label{eq:valiant:necessary-sufficient:fullStar:nChoice}
\]
This is possible since $\eps<\nfrac12$ and $\cX$ is infinite.
For each $i\in\insquare{n}$, let $\cD_i$ be the uniform distribution on $X_n\setminus\inbrace{x_i}$, so $\cD_i$ is compatible with $h^{(x_i)}$.
Choose $I$ uniformly from $\insquare{n}$ and use the pair $\inparen{h^{(x_I)},\cD_I}$
as the hard prior.
By Yao's minimax principle, it suffices to analyze deterministic learners with budget $M$ for each oracle.

Fix such a deterministic learner $\cL$.
Padding with ignored queries if necessary, we may assume that $\cL$ makes exactly $M$ example-oracle calls and exactly $M$ membership queries, and that all membership queries are asked after the samples are seen.
Let
    $P=\inparen{z_1,\dots,z_M}\in X_n^M$
be the ordered sample sequence, and let
    $\overline{P}\coloneqq \inbrace{z_1,\dots,z_M}$
be the set of distinct sampled points.

Now fix any sample sequence $P_0\in X_n^M$, and write
$ r\coloneqq \abs{\overline{P_0}}.$
For each $i\in\insquare{n}$,
\[
    \Pr\inparen{P=P_0\mid I=i}
    =
    (n-1)^{-M}\cdot \one\!\inbrace{x_i\notin \overline{P_0}}.
\]
Since the prior on $I$ is uniform, conditioning on $P=P_0$ makes $I$ uniform over
    $\inbrace{i\in\insquare{n}: x_i\notin \overline{P_0}},$
so there are exactly $n-r\geq n-M$ candidate hidden indices after the sample.

Ignoring queries outside $X_n$ (which always return $1$), the learner's behavior before it hits the hidden negative point is completely determined by the all-ones transcript.
Let $Q^{\uparrow}(P_0)\subseteq X_n$
be the set of points of $X_n$ queried along this branch.
Since $\abs{Q^{\uparrow}(P_0)}\leq M$, at most $M$ candidate indices can ever be identified by a negative answer, namely those with
\[
    x_i\in Q^{\uparrow}(P_0)\setminus \overline{P_0}.
\]
For every remaining candidate index $i$ with
$x_i\notin \overline{P_0}\cup Q^{\uparrow}(P_0),$
all membership answers are $1$, so the learner reaches the same all-ones leaf and outputs the same positive region
$H_{P_0}^{\uparrow}\subseteq \cX,$
independent of $i$.
Success on target $h^{(x_i)}$ then requires
\[
    x_i\notin H_{P_0}^{\uparrow}
    \qquad\text{and}\qquad
    \abs{
        \sinbrace{x_j\in X_n\setminus\inbrace{x_i}: x_j\notin H_{P_0}^{\uparrow}}
    }
    \leq
    \eps(n-1).
\]
Writing $T_{P_0}\coloneqq X_n\setminus H_{P_0}^{\uparrow},$
this says that
\[
    x_i\in T_{P_0}
    \qquad\text{and}\qquad
    \abs{T_{P_0}}\leq \eps(n-1)+1.
\]
Hence among the no-hit candidates, the learner can succeed on at most $\eps(n-1)+1$ indices.

Combining the hit and no-hit cases, conditioned on $P=P_0$ the learner can succeed on at most $M+\eps(n-1)+1$
of the $n-r$ candidate hidden indices.
Therefore
\[
    \Pr\inparen{\text{$\cL$ succeeds}\mid P=P_0}
    \leq
    \frac{M+\eps(n-1)+1}{n-r}
    \leq
    \frac{M+\eps(n-1)+1}{n-M}
    <
    \frac12,
\]
where the last inequality is exactly \eqref{eq:valiant:necessary-sufficient:fullStar:nChoice}.
Averaging over $P$, we obtain
    $\Pr\inparen{\text{$\cL$ succeeds}}<\frac12.$
Thus every deterministic learner with budget $M$ for each oracle fails with probability greater than $\nfrac12$ under the hard prior.
By Yao's minimax principle, the same holds for randomized learners, contradicting learnability at confidence $1-\delta>\nfrac12$.
Hence $\hyH$ is not learnable in \cref{def:valiant:general}.
\end{proof}

\vspace{-5mm}
\begin{proof}[Proof of Part III of \cref{thm:necessary-sufficient}]
By \cref{thm:oneStarDimension}, there exists a concept class
$\overline{\hyH}\supseteq \hyH$
that is closed under intersections and satisfies
   $ \vc\!\inparen{\overline{\hyH}}=\oneStar(\hyH).$
The claim now follows from the standard closure-learner theorem for intersection-closed classes, applied to $\overline{\hyH}$:
using no membership queries and
    $O\!\inparen{
        (\nfrac{1}{\eps})\cdot \inparen{\vc\inparen{\overline{\hyH}}\log\!\nfrac1\eps+\log\!\nfrac1\delta}
    }
$
positive examples, one obtains a hypothesis with no false positives and $\eps$-false negatives.
Since every target $h^\star\in \hyH$ also belongs to $\overline{\hyH}$, the same learner is valid for $\hyH$.
Substituting $\vc\!\inparen{\overline{\hyH}}=\oneStar(\hyH)$ gives the stated bound.
\end{proof}

\subsection{Proof That Conditions in \cref{thm:intro:necessary-sufficient} Are Not Tight for Learnability}
\label{sec:proofof:thm:nonTightness} 
In this section, we prove the following result, which shows that the conditions in \cref{thm:intro:necessary-sufficient} are not tight.
\begin{restatable}[]{theorem}{theoremNonTightness}
    \label{thm:nonTightness}
    The following hold in Valiant's model over a general domain (\cref{def:valiant:general}).
    \begin{enumerate}[itemsep=0pt,leftmargin=14pt]
        \item \textbf{(Finite VC dimension is not sufficient)}~~
        There exists a concept class $\hyH$ with $\vc\!\inparen{\hyH}<\infty$ that is \underline{not} learnable in Valiant's model.
        \item \textbf{(Finite $1$-star number is not necessary)}~~
        There exists a concept class $\hyH$ with $\oneStar\!\inparen{\hyH}=\infty$ that is learnable in Valiant's model.
        \item \textbf{(The whole-domain $1$-star condition is not necessary)}~~
        There exists a concept class $\hyH$ that is not learnable in Valiant's model, even though the whole domain is not a $1$-star for $\hyH$.
    \end{enumerate}
\end{restatable}
  
\begin{proof}[Proof of Part I of \cref{thm:nonTightness}]
The first witness is the simplest possible infinite $1$-star.
Let $\cX=\Z$, $h^\circ\equiv 1$, and, for each $i\in \Z$, $h_i(x)\coloneqq \one\!\inbrace{x\neq i}.$
Define
\[
    \hyH_{\mathrm{sing}}
    \coloneqq
    \inbrace{h^\circ}\cup \inbrace{h_i : i\in \Z}.
\]
We first verify that $\vc\!\inparen{\hyH_{\mathrm{sing}}}=1$.
To see this, note that: 
(1) Any singleton $\inbrace{a}$ is shattered, since $h^\circ(a)=1$ and $h_a(a)=0$, so $\vc\!\inparen{\hyH_{\mathrm{sing}}}\geq 1$.
(2) Further, no two-point set is shattered, as every hypothesis in $\hyH_{\mathrm{sing}}$ labels at most one point by $0$.
Next, by construction, the domain $\Z$ is a $1$-star for $\hyH_{\mathrm{sing}}$ centered at $h^\circ$.
Therefore Part II of \cref{thm:necessary-sufficient} implies that $\hyH_{\mathrm{sing}}$ is not learnable.

\end{proof}
  
\vspace{-5mm}

\begin{proof}[Proof of Part II of \cref{thm:nonTightness}]
Let $\cX\coloneqq \Z\times \zo$, and for each $i\in \Z$ define $h_i\colon \cX\to\zo$ by
\[
    h_i(x,0)\coloneqq \one\!\inbrace{x\neq i}
    \qquadand\qquad
    h_i(x,1)\coloneqq \one\!\inbrace{x\geq i}.
\]
Set $\hyH_{\mathrm{pre}}\coloneqq \inbrace{h_i : i\in \Z}$.
Intuitively, on the slice $\Z\times\inbrace{0}$ each hypothesis acts as the complement of a singleton, and on the slice $\Z\times\inbrace{1}$ it acts as a threshold.

We first verify that $\oneStar(\hyH_{\mathrm{pre}})=\infty$. Fix any $k\in \N$ and let $S_k\coloneqq \inbrace{(1,0),\dots,(k,0)}$. Then $h_{k+1}$ labels every point of $S_k$ by $1$, while for each $t\in\insquare{k}$ the hypothesis $h_t$ agrees with $h_{k+1}$ on $S_k\setminus \inbrace{(t,0)}$ and labels $(t,0)$ by $0$. Thus $S_k$ is a $1$-star for $\hyH_{\mathrm{pre}}$, and since $k$ is arbitrary, $\oneStar(\hyH_{\mathrm{pre}})=\infty$.

We now show that $\hyH_{\mathrm{pre}}$ is learnable. Fix a target $h_{i^\star}\in \hyH_{\mathrm{pre}}$ and a distribution $\cD$ compatible with $h_{i^\star}$, and let $\cD_0$ denote the marginal of $\cD$ on the first coordinate. Given $\eps,\delta\in(0,1)$, the learner draws
\[
    m = O\!\inparen{\frac{\log\!\nfrac1\eps+\log\!\nfrac1\delta}{\eps}}
\]
independent examples $(x_1,y_1),\dots,(x_m,y_m)\sim \cD$, lets $Z\coloneqq \inbrace{x_1,\dots,x_m}\subseteq \Z$, writes the distinct values of $Z$ as $z_{(1)}<\cdots<z_{(r)}$, and adds sentinels $z_{(0)}\coloneqq -\infty$ and $z_{(r+1)}\coloneqq \infty$.

By the $\eps$-net theorem for intervals (which have VC dimension $2$), with probability at least $1-\delta$ every interval $J\subseteq \R$ with $\cD_0(J)\geq \eps$ contains at least one sampled coordinate. Fix a realization for which this holds. The learner then uses membership queries on the slice $\Z\times\inbrace{1}$. Since $h_{i^\star}(x,1)=\one\!\inbrace{x\geq i^\star}$, a binary search over $z_{(1)},\dots,z_{(r)}$ finds the smallest index $t\in\inbrace{1,\dots,r+1}$ such that $h_{i^\star}\sinparen{z_{(t)},1}=1$ (with the convention $t=r+1$ if all answers are $0$). Setting $L\coloneqq z_{(t-1)}$ and $R\coloneqq z_{(t)}$, we have $i^\star\in (L,R]$.

If $R<\infty$, the learner asks one additional membership query at $(R,0)$. If the answer is $0$, then $R=i^\star$, so it outputs the exact target hypothesis $h_R$. Otherwise, or if $R=\infty$, let $J\coloneqq (L,R)\cap \Z$ (with the convention $(L,\infty)\cap \Z$ when $R=\infty$), and output the hypothesis $\wh h$ defined by
\[
    \wh h(x,0)=\one\!\inbrace{x\notin J}
    \qquadand
    \wh h(x,1)=\one\!\inbrace{x\geq R},
\]
where $\one\!\inbrace{x\geq \infty}=0$.

We now verify correctness. If $R=i^\star$, then $\wh h=h_{i^\star}$ and there is nothing to prove. Suppose we are in the remaining case. First, $\wh h$ has no false positives. Indeed, since $i^\star\in (L,R]$ and we are not in the exact case, we actually have $i^\star\in J=(L,R)\cap \Z$. Therefore, if $\wh h(x,0)=1$, then $x\notin J$, so $x\neq i^\star$ and hence $h_{i^\star}(x,0)=1$. Also, if $\wh h(x,1)=1$, then necessarily $R<\infty$ and $x\geq R>i^\star$, so again $h_{i^\star}(x,1)=1$.

Next, every false negative of $\wh h$ has first coordinate in $J$. For points of the form $(x,0)$, this is immediate from the definition of $\wh h$. For points of the form $(x,1)$, if $\wh h(x,1)=0$ while $h_{i^\star}(x,1)=1$, then $i^\star\leq x<R$, and since $i^\star\in J$, this implies $x\in J$. Thus the total false-negative probability is at most $\cD_0(J)$.

Finally, the interval $J$ contains no sampled coordinate by construction, so the $\eps$-net property gives $\cD_0(J)<\eps$, and the learner succeeds with probability at least $1-\delta$. The total cost is $O\!\inparen{(\nfrac{1}{\eps})\cdot \inparen{\log\!\nfrac1\eps+\log\!\nfrac1\delta}}$ example-oracle calls and $O\!\inparen{\log(m+1)}$ membership queries.
Therefore $\hyH_{\mathrm{pre}}$ is learnable in Valiant's model and satisfies $\oneStar(\hyH_{\mathrm{pre}})=\infty$.
\end{proof}  
\vspace{-5mm}

\begin{proof}[Proof of Part III of \cref{thm:nonTightness}]
For each finite set $A\subseteq \N$, let $h_A(x)\coloneqq \one\!\inbrace{x\in A}$, and set
\[
    \hyH_{\mathrm{fin}}
    \coloneqq
    \inbrace{
        h_A : A\subseteq \N \text{ finite}
    }.
\]
We first show that $\vc\!\inparen{\hyH_{\mathrm{fin}}}=\infty$. Fix any $n\in\N$ and let $S_n\coloneqq \inbrace{1,\dots,n}$. For every subset $T\subseteq S_n$, the set $T$ is finite, so $h_T\in \hyH_{\mathrm{fin}}$ and $h_T$ realizes exactly that labeling on $S_n$. Thus $S_n$ is shattered, and since $n$ is arbitrary, $\vc\!\inparen{\hyH_{\mathrm{fin}}}=\infty$. By Part~(1) of \cref{thm:necessary-sufficient}, $\hyH_{\mathrm{fin}}$ is not learnable.

It remains to show that the whole domain $\N$ is not a $1$-star for $\hyH_{\mathrm{fin}}$. If $\N$ were a $1$-star, then in particular there would exist a center hypothesis $h^\circ\in \hyH_{\mathrm{fin}}$ such that $h^\circ(x)\equiv 1$ (for each $x\in \N$) and, further, for each $z$, there would need to be a hypothesis $h_z\in \hyH_{\mathrm{fin}}$ such that $h_z(x)=\mathds{1}\sinbrace{x\neq z}$ (for each $x\in \N$). But every hypothesis in $\hyH_{\mathrm{fin}}$ has finite support, so neither $h^\circ$ nor $h_z$ (for any $z\in \N$) exists. Hence, $\N$ is not a $1$-star for $\hyH_{\mathrm{fin}}$.
\end{proof} 

\subsection{Proof of \cref{thm:halfspace} (Halfspaces Are Learnable on \texorpdfstring{$\R^d$}{R^d})} 
    \label{sec:proofof:thm:halfspace}
 In this section, we prove \cref{thm:halfspace}, which we restate below.
\thmhalfspace*
\begin{remark}[Comparison with the Boolean Hypercube]
    It is useful to contrast \cref{thm:halfspace} with the Boolean setting.
    For each $d\in\N$, let $\hyH^{\mathrm{cube}}_d$ denote the restriction of halfspaces in $\R^d$ to the Boolean cube $\zo^d$.
    Perhaps surprisingly, the sequence $\inparen{\hyH^{\mathrm{cube}}_1,\hyH^{\mathrm{cube}}_2,\dots}$ is not learnable in Valiant's model, even though the unrestricted class over $\R^d$ is learnable by \cref{thm:halfspace} (see \cref{thm:booleanHalfspace}).
    The key difference is that over $\R^d$, the learner may query any point in the ambient space, whereas on the Boolean cube the learner is confined to queries inside $\zo^d$.
    This additional geometric freedom is exactly what makes \cref{thm:halfspace} possible.
\end{remark}

\paragraph{Setup.}
To prove \cref{thm:halfspace}, fix $\eps,\delta\in(0,1)$, a target halfspace
\[
    \Hstar=\inbrace{x\in \R^d : \langle n^\star,x\rangle \geq c^\star},
\]
where $n^\star\neq 0$ is a normal vector of $\Hstar$, and fix a distribution $\cD$ compatible with $\Hstar$ (\cref{def:compatible}).
Let $m$ be the sample size as specified later, and draw $m$ independent examples $x_1,\dots,x_m\sim \cD.$
Define 
\[
    P\coloneqq \inbrace{x_1,\dots,x_m},
    \qquad
    C\coloneqq \conv(P),
    \qquadand 
    A\coloneqq \aff(C)\,.
\]
Here, $C$ is the convex hull of the observed positive sample and $A$ is the smallest affine subspace containing $C$.
Let $L$ be the linear subspace of directions parallel to \mbox{$A$, and let $r$ denote its dimension:}
\[
    L\coloneqq \inbrace{x-y : x,y\in A}
    \quadand
    r\coloneqq \dim(L)\,.
\]

\subsection*{Facets, Cones, and Simplexes}
Next, we introduce facets of $C$ and define certain cones and simplexes derived from them, which the learning algorithm will use.
In the degenerate case where $r=0$ and, hence, $C$ is a single sampled point, we do not need these additional definitions; the learner will simply return $C$.
Hence, for the remainder of this part, assume $r\geq 1$.

Consider a vertex $v$ of $C$, and let $\cF(v)$ denote the set of facets of $C$ (within $A$) that are incident to $v$.
For each $F\in \cF(v)$, let $n_F\in L$ be the outward unit normal of $F$.
The supporting halfspace of $F$ is $H_F
    \coloneqq
    \inbrace{x\in A : \langle n_F,x-v\rangle\leq 0}$, which contains $C$ by construction.
\mbox{We first recall a standard fact:}
\vspace{-7mm}
\begin{lemma}[Incident Normals Span $L$]
\label{lem:halfspace:incident-span}
The outward normals of the facets in $\cF(v)$ span $L$.
\end{lemma}
For completeness, we prove \cref{lem:halfspace:incident-span} at the end of this section.
Next, consider any set of $r$ facets with linearly independent normals:
\[
    J=\inbrace{F_1,\dots,F_r}\subseteq \cF(v)
\]
For each such choice, $W_{v,J}$ denotes the cone cut out by these facets:
\[
    W_{v,J}
    \coloneqq
    \bigcap_{F\in J} H_{F}
    =
    \inbrace{x\in A : \langle n_{F},x-v\rangle\leq 0 \text{ for all } F\in J}.
\]
Since each $H_F$ contains $C$, we have $W_{v,J}\supseteq C$.
Next, define
\[
    a_{v,J}\coloneqq -\sum_{F\in J} n_{F}
    \qquadand
    T_{v,J}\coloneqq \max_{x\in C}\langle a_{v,J},x-v\rangle.
    \yesnum\label{eq:halfspace:definitionOfT}
\]
The vector $a_{v,J}$ aggregates the inward-pointing normals, and $T_{v,J}$ measures the extent of $C$ from $v$ in direction $a_{v,J}$.
Next, we use $a_{v,J}$ and $T_{v,J}$ to ``close off'' the cone $W_{v,J}$: for each $v$ and $J$, define
\[
    S_{v,J}
    \coloneqq
    \inbrace{x\in W_{v,J} : \langle a_{v,J},x-v\rangle\leq T_{v,J}}.
\]
The learning algorithm outputs $S_{v,J}$ for a suitable choice of $v$ and $J$.
We now establish the key properties of these sets.
First, every candidate contains the sample hull:
\begin{lemma}
\label{lem:halfspace:contains-hull}
For every vertex $v$ of $C$ and size-$r$ subset of facets $J\subseteq \cF(v)$ with linearly independent normals, $C\subseteq S_{v,J}.$
\end{lemma}
\vspace{-5mm}
\begin{proof}
We already saw $C\subseteq W_{v,J}$ and, by $T_{v,J}$'s definition, each $x{\in} C$ satisfies $\langle a_{v,J},x-v\rangle\leq T_{v,J}.$
\end{proof}
Next, we show that there always exists a choice of $v^\star$ and $J^\star$ for which the candidate simplex lies inside the target halfspace $\Hstar$.
\begin{lemma}[Safe Candidate Exists]
\label{lem:halfspace:safe-candidate}
    There exist a vertex $v^\star$ of $C$ and a size-$r$ subset $J^\star\subseteq \cF(v^\star)$ with linearly independent normals such that $S_{v^\star,J^\star}\subseteq \Hstar.$
\end{lemma}
\vspace{-5mm} 
\begin{proof}[Proof of \cref{lem:halfspace:safe-candidate}]
Let $\Pi_L(n^\star)$ denote the orthogonal projection of $n^\star$ onto $L$.
The restriction of the affine functional $x\mapsto \langle n^\star,x\rangle$ to $A$ differs from $x\mapsto \langle \Pi_L(n^\star),x\rangle$ by an additive constant, so the latter is minimized on $C$ at some vertex.
Fix a vertex $v^\star$ minimizing $x\mapsto \langle \Pi_L(n^\star),x\rangle$ over $C$.

By the normal-cone characterization of minimizing vertices, the vector $-\Pi_L(n^\star)$ lies in the normal cone of $C$ at $v^\star$.
For a polytope, this normal cone is exactly the conic hull of the outward normals of the facets incident to $v^\star$.
Therefore $-\Pi_L(n^\star)$ is a nonnegative combination of the outward normals of the facets in $\cF(v^\star)$.
By conic Carath{\'e}odory in the $r$-dimensional space $L$, there exists a set $J_0\subseteq \cF(v^\star)$ of size $\abs{J_0}\leq r,$
whose outward normals already generate $-\Pi_L(n^\star)$.
Choose such a representation with minimal support, so that the corresponding outward normals are linearly independent.
Now, using \cref{lem:halfspace:incident-span}, enlarge $J_0$ if necessary to an $r$-element set
$J^\star=\inbrace{F_1,\dots,F_r}\subseteq \cF(v^\star)$
whose outward normals are linearly independent.

Let $x\in W_{v^\star,J^\star}$.
Then for every $F\in J^\star$, $ \langle n_F,x-v^\star\rangle\leq 0.$
Since $-\Pi_L(n^\star)$ is a nonnegative combination of the outward normals in $J_0\subseteq J^\star$, $\langle \Pi_L(n^\star),x-v^\star\rangle\geq 0.$
But $x-v^\star\in L$, so
\[
    \langle n^\star,x-v^\star\rangle
    =
    \langle \Pi_L(n^\star),x-v^\star\rangle
    \geq 0.
\]
Because $v^\star\in C\subseteq \Hstar$, we have $\langle n^\star,v^\star\rangle\geq c^\star$, and therefore
\[
    \langle n^\star,x\rangle
    =
    \langle n^\star,v^\star\rangle+\langle n^\star,x-v^\star\rangle
    \geq c^\star.
\]
Thus $x\in \Hstar$, so $W_{v^\star,J^\star}\subseteq \Hstar.$
Since $S_{v^\star,J^\star}\subseteq W_{v^\star,J^\star}$ by definition, we conclude that $S_{v^\star,J^\star}\subseteq \Hstar.$
\end{proof}
Finally, to bound the false-negative rate, we use that the candidate outputs are all simplexes, a class of bounded VC dimension. The following lemma confirms that each $S_{v,J}$ is indeed a simplex:
\vspace{-2mm}\vspace{-2mm}
\begin{lemma}
\label{lem:halfspace:simplex}
    The set $S_{v,J}$ is an $r$-simplex in $A$. (Recall that an $r$-simplex is a polytope which is the convex hull of its $r + 1$ vertices.)
\end{lemma}
The proof is straightforward and appears at the end of this section.
This is useful for two reasons.
First, the class of simplexes has bounded VC dimension:
\begin{fact}[\citep{blumer1989learnability}]
    \label{lem:halfspace:vc}
    Let $\hyT_d$ be the set of simplexes in $d$ dimensions.
    Then, $\vc(\hyT_d)=O\!\inparen{d^2\log d}.$
\end{fact}
Second, since $S_{v,J}$ is a simplex with at most $r+1\leq d+1$ vertices, the learner can check whether $S_{v,J}\subseteq \Hstar$ by querying the membership oracle on these vertices alone: if all vertices are positive (\ie{}, they lie in $\Hstar$), convexity implies that the entire simplex lies in $\Hstar$.
\vspace{-2mm} 
\subsection*{Learning Algorithm}
\vspace{-2mm}
We now describe the learning algorithm (see \cref{alg:halfspace}).
If $r=0$, the algorithm returns the hypothesis $\wh h(x)=\one\!\inbrace{x\in C}$.
Otherwise, it enumerates every vertex $v$ of $C$ and every size-$r$ subset $J\subseteq \cF(v)$ of facets with linearly independent normals, constructs the corresponding candidate simplex $S_{v,J}$ for each pair $(v,J)$, and collects the vertices of all these simplexes into a finite set $\cV$.
It then runs the point-location procedure of \citet[Theorem~1.5 and Section~7]{hopkins2020pointlocation} on $\cV$ with failure parameter $\delta/2$, and returns the first candidate simplex whose vertices are all labeled $1$ by this procedure.
If no such candidate exists, it returns the all-zero hypothesis.
When the algorithm returns a candidate simplex, we denote it by $\wh S$ and write $\wh h(x)\coloneqq \one\sinbrace{x\in \wh S}$ for the resulting hypothesis.
\begin{algorithm}[htb!]
\caption{\textsc{Hull-Simplex Learner}$(d)$ \,\, (Learning halfspaces in Valiant's model)}
\label{alg:halfspace}
\begin{algorithmic}[1]
\Procedure{\textsc{Hull-Simplex Learner}}{$\eps,\delta,d$}
    \State Set $m
        \gets
        O\!\inparen{
            (\nfrac{1
            }{\eps})\cdot \inparen{
                d^2\log d\cdot \log\!\nfrac{1}{\eps}
                +
                \log\!\nfrac{1}{\delta}}
        }$ and $\eta\gets \delta/2$
    \State Obtain $m$ i.i.d.\ examples $x_1,\dots,x_m$ from the example oracle
    \State Set $P\gets \inbrace{x_1,\dots,x_m},$ 
    $C\gets \conv(P),$
    $A\gets \aff(C),$ and 
    $r\gets \dim(A)$
    \vspace{4mm}
    
    \State \textbf{If} $r=0$ (\ie{}, all samples are identical), \textbf{then} \Return hypothesis $\widehat h(x)=\one\!\inbrace{x\in C}$
    \vspace{4mm}
    \item[] \phantom{..}\quad \textbf{$\#$~Phase A:}~~ \textit{Construct list of simplexes $\inbrace{S_{v,J}: v,J}$ containing $C$}
    \vspace{0mm}
    \State Initialize an ordered list of candidate simplexes: $\cS\gets \emptyset$
    \For{each vertex $v$ of $C$}
        \State Let $\mathcal F(v)$ be the set of facets of $C$ in $A$ that are incident to $v$
        \For{each size-$r$ subset $J\subseteq \mathcal F(v)$ of facets with linearly-independent normals}

            \State For each $F\in J$, let $n_{F}$ be the outward unit normal of $F$, and define\vspace{-2mm}
            \[
                H_{F}
                =
                \inbrace{
                    x\in A : \langle n_{F},x-v\rangle \leq 0
                }\vspace{-2mm}\vspace{-1mm}
            \]
            \State Set $a_{v,J}\gets -\sum_{F\in J} n_{F}$ and $T_{v,J}\gets \max_{x\in C}\langle a_{v,J},x-v\rangle$
            \State Use $a_{v,J}$ and $T_{v,J}$ to define the candidate simplex as follows:\vspace{-2mm}
            \begin{align*}
                W_{v,J}
                \gets
                \bigcap\nolimits_{F\in J} H_F
                \qquadand
                S_{v,J}
                \gets
                W_{v,J}
                \cap 
                \inbrace{
                    x\in A:
                    \langle a_{v,J},x-v\rangle\leq T_{v,J}
                }\vspace{-2mm}\vspace{-1mm}
            \end{align*}  
            \State Append $S_{v,J}$ to $\cS$

        \EndFor
    \EndFor
    \vspace{4mm}
    \item[] \phantom{..}\quad \textbf{$\#$~Phase B:}~~ \textit{Find labels of vertices of all simplexes}
    \vspace{0mm}
    \State Let $\cV$ be the set of all vertices of all simplexes in $\cS$
    \State Run the point-location procedure of \cite{hopkins2020pointlocation} on $\cV$ with failure parameter $\eta$
    \State Let $\lambda(z)\in\inbrace{0,1}$ denote the returned label for each $z\in \cV$
    \vspace{4mm}
    \item[] \phantom{..}\quad \textbf{$\#$~Phase C:}~~ \textit{Select output hypothesis}
    \vspace{0mm}
    \State \textbf{if} there exists $S\in \cS$ all of whose vertices $z$ satisfy $\lambda(z)=1$ \textbf{then} \Return  $\widehat h(x)=\one\sinbrace{x\in S}$
    \State \textbf{otherwise} \Return all-zero hypothesis  
    \vspace{3mm}
\EndProcedure
\end{algorithmic}
\end{algorithm}

\vspace{-2mm}
\subsection*{Analysis of Correctness}
\vspace{-2mm}
Next, we prove correctness of the learning algorithm.
\begin{lemma}[Correctness on the Point-Location Event]
\label{lem:halfspace:one-sided}
Let $\evE_{\mathrm{PL}}$ denote the event that the point-location subroutine labels every point of $\cV$ correctly.
On $\evE_{\mathrm{PL}}$, the learner returns a candidate simplex $\wh S$ that contains the entire sample $P$ and satisfies $\wh S\subseteq \Hstar$.
\mbox{Hence on $\evE_{\mathrm{PL}}$, the output $\wh h$ has no false positives.}
\end{lemma}
\vspace{-5mm}
\begin{proof}[Proof of \cref{lem:halfspace:one-sided}]
If $r=0$, then the learner returns $\wh S=C\subseteq \Hstar$, so the claim is immediate.
Assume therefore that $r\geq 1$.
By \cref{lem:halfspace:safe-candidate}, there exists a candidate simplex $S_{v^\star,J^\star}$ contained in $\Hstar$.
All of its vertices belong to $\cV$ and are positively labeled (\ie{}, they lie in $\Hstar$).
Therefore, on $\evE_{\mathrm{PL}}$, the point-location subroutine labels all vertices of $S_{v^\star,J^\star}$ by $1$, so the learner returns some candidate simplex $\wh S$.
By \cref{lem:halfspace:contains-hull}, every candidate simplex contains $C$, and hence $P\subseteq C\subseteq \wh S$.
Again on $\evE_{\mathrm{PL}}$, all vertices of $\wh S$ are labeled correctly, so they all lie in $\Hstar$.
Since $\Hstar$ is convex and $\wh S$ is the convex hull of its vertices, we obtain $\wh S\subseteq \Hstar$.
Thus the output has no false positives.
\end{proof}
On the event $\evE_{\mathrm{PL}}$, the learner returns a simplex of dimension at most $d$ that contains the entire sample.
To control its false-negative error, we use the fact that the class of simplexes has VC dimension $O(d^2\log d)$ (\cref{lem:halfspace:vc}).
\begin{lemma}[False-Negative Error]
\label{lem:halfspace:false-negative}
If
$m=O\!\inparen{
        \inparen{\nfrac{1}{\eps}}\cdot \inparen{
            d^2\log d\cdot \log\!\nfrac1\eps+\log\!\nfrac1\delta
        }
    }$
then with probability $1-\delta$ over the sample draw, every simplex $T\in \hyT_d$ containing $P$ satisfies
    $\Pr_{x\sim \cD}\sinparen{x\notin T}\leq \eps.$
Consequently, on the intersection of this event with $\evE_{\mathrm{PL}}$, \mbox{the returned $\wh S$ satisfies $\Pr_{x\sim \cD}\sinparen{x\notin \wh S}\leq \eps,$
and therefore}
\[
    \Pr_{x\sim \cD}\inparen{\wh h(x)=0 \text{ and } \Hstar(x)=1}\leq \eps.
\]
\end{lemma}
\vspace{-5mm}
\begin{proof}[Proof of \cref{lem:halfspace:false-negative}]
Let $\hyT_d$ be the class of simplexes of dimension at most $d$.
Define $\hyG_d$ to be the class of complements of hypotheses in $\hyT_d$, \ie{},
$\hyG_d
    \coloneqq
    \inbrace{
        x\mapsto \one\!\inbrace{x\notin T} :
        T\in \hyT_d
    }.$
Since complementing a class does not change its VC dimension, $\vc(\hyG_d)=\vc(\hyT_d)$.
By \cref{lem:halfspace:vc} and the standard VC bound for zero empirical error, if
\[
    m
    =
    O\!\inparen{
        \frac{\vc(\hyG_d)\log\!\nfrac1\eps+\log\!\nfrac1\delta}{\eps}
    }
    =
    O\!\inparen{
        \frac{d^2\log d\cdot \log\!\nfrac1\eps+\log\!\nfrac1\delta}{\eps}
    },
\]
then with probability at least $1-\delta$ every function in $\hyG_d$ with empirical mean zero has true mean at most $\eps$.
Now fix any simplex $T\in \hyT_d$ containing $P$.
Its miss indicator $g_T(x)\coloneqq \one\!\inbrace{x\notin T}$ belongs to $\hyG_d$ and has empirical mean zero on the sample, because $P\subseteq T$.
Hence $\Pr_{x\sim \cD}\sinparen{x\notin T}\leq \eps$, proving the first claim.
On the event $\evE_{\mathrm{PL}}$, \cref{lem:halfspace:one-sided} shows that the returned simplex $\wh S$ contains $P$ and satisfies $\wh S\subseteq \Hstar$.
Applying the first claim with $T=\wh S$ gives
   $\Pr_{x\sim \cD}\sinparen{x\notin \wh S}\leq \eps.$
Since $\cD$ is compatible with $\Hstar$ (\cref{def:compatible}), it is supported on $\Hstar$.
Therefore $\Pr_{x\sim \cD}\sinparen{\wh h(x)=0 \text{ and } \Hstar(x)=1}
    =
    \Pr_{x\sim \cD}\sinparen{x\notin \wh S}
    \leq
    \eps.\qedhere{}$
\end{proof}
\vspace{-4mm}
\begin{lemma}[Query Complexity]
\label{lem:halfspace:query}
The learner uses
$O\sinparen{
        d^3\log^2 d\cdot \log{\nfrac{d}{\eps\delta}}
    }$
membership queries.
\end{lemma}
\vspace{-5mm}
\begin{proof}[Proof of \cref{lem:halfspace:query}]
If $r=0$, the learner uses no membership queries.
Assume therefore that $r\geq 1$.
Let $\cV$ denote the set of all vertices of all candidate simplexes $S_{v,J}$.

\begin{definition}[Point-Location Problem]
    Given a finite set $X\subseteq \R^d$ and an unknown closed halfspace $H=\inbrace{x\in \R^d : \langle n,x\rangle \geq c},$
the point-location problem asks for the labels $\one\!\inbrace{x\in H}$ of all points in $X$ using binary queries at adaptively chosen points of $\R^d$.
\end{definition}
\citet{hopkins2020pointlocation} prove the following result.\footnote{While \citet{hopkins2020pointlocation} focus on ternary queries and homogeneous halfspaces for most of their paper, Section~7 of their work extends their result to non-homogeneous halfspaces with binary queries.}
\begin{theorem}[Theorem~1.5 and Section~7 in \citep{hopkins2020pointlocation}]
    For every finite $X\subseteq \R^d$ and $\eta\in(0,1)$, there is a randomized query procedure which, with probability at least $1-\eta$ over its internal randomness, correctly labels every point of $X$ with respect to the unknown closed halfspace, using $q$ binary queries for 
\[
    q=O\!\inparen{
        d\log^2 d\cdot \log\!\nfrac{\abs{X}}{\eta}
    }\,.
\]
\end{theorem}
\mbox{We apply this result to the set $\cV$ and the target halfspace $\Hstar$, to which we have membership access.}

\paragraph{Upper Bound on $\abs{\cV}$.}
It remains to bound $\abs{\cV}$.
The polytope $C$ has at most $m$ vertices.
Let $f(C)$ denote the number of facets of $C$ in $A$.
Each facet of the $r$-dimensional polytope $C$ contains $r$ affinely independent vertices, and any such $r$-tuple determines the affine hull of that facet uniquely.
Choosing one such $r$-tuple for each facet gives an injection from the set of facets into the set of $r$-subsets of the vertex set of $C$.
Hence $f(C)\leq \binom{m}{r}\leq m^r.$
For each vertex $v$, the learner considers at most
    $\binom{f(C)}{r}\leq \inparen{m^r}^{\,r}=m^{r^2}$
choices of $J$.
Since there are at most $m$ vertices, the total number of candidate pairs $(v,J)$ is at most $m^{r^2+1}\leq m^{d^2+1}$.
Each candidate simplex has at most $r+1\leq d+1$ vertices, so
   $ \abs{\cV}\leq (d+1)m^{d^2+1}.$
Because $m
    =
    O\!\inparen{
        (\nfrac{1}{\eps})\cdot\inparen{d^2\log d\cdot \log\!\nfrac1\eps+\log\!\nfrac1\delta}
    }$
is polynomial in $d$, $\nfrac1\eps$, and $\nfrac1\delta$, we have $\log m = O\!\inparen{\log{\nfrac{d}{\eps\delta}}}.$
Hence
\[
    \log\abs{\cV}
    \leq
    \log(d+1)+(d^2+1)\log m
    =
    O\!\inparen{
        d^2\log{\nfrac{d}{\eps\delta}}
    }.
\]
The learner runs the point-location procedure with failure parameter $\eta=\delta/2$.
Therefore, the number of membership queries is
$O\sinparen{
        d\log^2 d\cdot \log\!\nfrac{\abs{\cV}}{\eta}
    }
    =
    O\sinparen{
        d^3\log^3 d\cdot \log{\nfrac{1}{\eps\delta}}
    }.\qedhere$
\end{proof}
\vspace{-6mm}
\subsection*{Completing the Proof of \cref{thm:halfspace}}
\begin{proof}[Proof of \cref{thm:halfspace}]
The learner uses exactly $m$ calls to the example oracle.
Let $\evE_{\mathrm{PL}}$ denote the event from \cref{lem:halfspace:one-sided}.
Because the point-location subroutine is run with failure parameter $\delta/2$, for every fixed sample we have $\Pr\!\inparen{\evE_{\mathrm{PL}} \mid x_1,\dots,x_m}\geq 1-\delta/2.$
Hence, $\Pr\sinparen{\evE_{\mathrm{PL}}}\geq 1-\delta/2$.

Next, apply \cref{lem:halfspace:false-negative} with $\delta/2$ in place of $\delta$.
If $m
    =
    O\!\inparen{
        \inparen{\nfrac{1}{\eps}}\cdot \inparen{
            d^2\log d\cdot \log\!\nfrac1\eps+\log\!\nfrac1\delta
        }
    },$
then with probability at least $1-\delta/2$ over the sample draw, every simplex containing $P$ has false-negative mass at most $\eps$.

On the intersection of these two events, \cref{lem:halfspace:one-sided} shows that the learner returns a simplex $\wh S$ with $P\subseteq \wh S\subseteq \Hstar$, and \cref{lem:halfspace:false-negative} gives
$\Pr_{x\sim \cD}\sinparen{\wh h(x)=0 \text{ and } \Hstar(x)=1}\leq \eps.$
Moreover, \cref{lem:halfspace:one-sided} gives that the output has no false positives.
A union bound therefore shows that the learner succeeds with probability at least $1-\delta$.
Finally, \cref{lem:halfspace:query} gives the membership-query bound $O\sinparen{
        d^3\log^3 d\cdot \log{\nfrac{1}{\eps\delta}}
    }.$
This completes the proof of the theorem.
\end{proof} 

\vspace{-4mm}
\subsection*{Proofs Deferred from Earlier in this Section}
\vspace{-2mm}
\begin{proof}[Proof of \cref{lem:halfspace:incident-span}]
Suppose toward a contradiction that the normals of $\cF(v)$ do not span $L$.
Then there exists a nonzero vector $y\in L$ orthogonal to the outward normals of all facets in $\cF(v)$.
For every facet of $C$ that is incident to $v$, the defining inequality of its supporting halfspace remains tight along the line $v+\tau y$.
For every facet of $C$ that is not incident to $v$, the defining inequality is strict at $v$, and therefore remains satisfied for all sufficiently small $\abs{\tau}$.
Hence both $v+\tau y$ and $v-\tau y$ lie in $C$ for all sufficiently small $\tau>0$.
This contradicts that $v$ is a vertex of $C$.
\end{proof}
\vspace{-7mm}
\begin{proof}[Proof of \cref{lem:halfspace:simplex}]
     Recall that an $r$-simplex is an $r$-dimensional polytope that is the convex hull of its $r + 1$ vertices.
To show $S_{v,J}$ is an $r$-simplex, introduce the affine map $\Phi_{v,J}\colon A\to \R^r$ defined as $\Phi_{v,J}(x)
    \coloneqq
    \inparen{
        -\langle n_{F_1},x-v\rangle,
        \dots,
        -\langle n_{F_r},x-v\rangle
    }.$
This coordinate map is useful because it converts the facet inequalities defining $W_{v,J}$ into the standard nonnegativity constraints in $\R^r$.
In particular, by the definitions of $W_{v,J}$, $a_{v,J}$, and $\Phi_{v,J}$, $\Phi_{v,J}\inparen{S_{v,J}}
    =
    \inbrace{
        y\in \R^r_{\geq 0} :
        y_1+\cdots+y_r\leq T_{v,J}
    }.$
Assume $T_{v,J}$ is positive.
Then $\Phi_{v,J}\inparen{S_{v,J}}$ is the standard $r$-simplex in $\R^r$.
Further, since the vectors $n_{F_1},\dots,n_{F_r}$ are linearly independent and $\dim(A)=\dim(L)=r$, $\Phi_{v,J}$ is an isomorphism from $A$ to $\R^r$.
Putting these two observations together implies that $S_{v,J}$ also is an $r$-simplex in $A$.

It remains to show that $T_{v,J}$ is positive.
To see this, observe that, for every $1\leq \ell \leq r$, $C\subseteq H_{F_\ell}$ (by construction), so for every $x\in C$, $\Phi_{v,J}(x)\in \R_{\geq 0}^r.$
Since $r\geq 1$, the set $C$ is not equal to $\inbrace{v}$.
Choose $x\in C\setminus \inbrace{v}$.
Because $\Phi_{v,J}$ is injective, $\Phi_{v,J}(x)\neq \Phi_{v,J}(v)=0$, and therefore $\sum_{\ell=1}^{r}\Phi_{v,J}(x)_\ell
    =
    \langle a_{v,J},x-v\rangle
    >
    0.$
Thus, $T_{v,J}$ (see \eqref{eq:halfspace:definitionOfT}) is positive.
\end{proof} 

\vspace{-4mm}
\subsection{Proof of \cref{thm:booleanHalfspace} (Halfspaces Are Not Learnable on the Boolean Cube)}
    \label{sec:proofof:thm:booleanHalfspace}
    In this section, we prove \cref{thm:booleanHalfspace}, which we restate below.
    \thmbooleanhalfspace*
    \vspace{-3mm}
\begin{proof}[Proof of \cref{thm:booleanHalfspace}]
For every $d$, the class $\hyH^{\rm H}_d$ contains the constant-one hypothesis.
Further, for each $v\in \zo^d$, it contains a hypothesis that is positive on every point of $\zo^d\setminus\inbrace{v}$ and negative on $v$ alone.
Indeed, if we define $D_v(x)
    \coloneqq
    \sum_{i:v_i=0} x_i + \sum_{i:v_i=1} \inparen{1-x_i},$
then $D_v$ is an affine function of $x$, $D_v(v)=0$, and $D_v(x)\geq 1$ for every $x\in \zo^d\setminus\inbrace{v}$.
Hence the halfspace $h_v(x)
    \coloneqq
    \one\!\inbrace{D_v(x)\geq \nfrac12}$
belongs to $\hyH^{\rm H}_d$ and excludes exactly the point $v$.
Therefore, for every $A\subseteq \zo^d$, $A
    =
    \bigcap_{v\in \zo^d\setminus A} \supp(h_v),$
so the intersection closure of $\hyH^{\rm H}_d$ contains every subset of $\zo^d$.
Thus Item~2 of \cref{thm:implications} applies, and $\hyH^{\rm H}$ is not learnable in Valiant's model.
\end{proof}

\newpage

\newpage

\newpage

\newpage
\printbibliography
\newpage

\appendix

\section{Additional Results}
    \label{sec:appendix:additional}

\subsection{Learning Intersections of $s$ Halfspaces}

In this section, we show that intersections of $s$ halfspaces are learnable in Valiant's model.
\begin{theorem}[Intersections of Halfspaces Are Learnable in Valiant's Model]
\label{thm:s-halfspaces}
Fix $d,s\geq 1$.
The class of intersections of $s$ halfspaces in $\R^d$ is learnable in the extension of Valiant's model to $\R^d$ (\cref{def:valiant:general})
with sample complexity $m(\eps,\delta)$ and query complexity $q(\eps,\delta)$ for
\[
    m(\eps,\delta)
    =
    O\!\inparen{
        \frac{
            s d^2 \log\!\inparen{s d}\cdot \log\!\nfrac1\eps
            +
            \log\!\nfrac1\delta
        }{\eps}
    }
    \qquadand
    q(\eps,\delta)=(2es)^d\,m^{s(d^2+1)}\,.
\]
\end{theorem}
The proof builds directly on the geometric machinery developed for the single-halfspace case.
The reader may find it helpful to review the proof of \cref{thm:halfspace} before proceeding.
At a high level, in this proof, we reuse the candidate simplices $S_{v,J}$ from the proof of \cref{thm:halfspace} and intersect $s$ of them, one for each halfspace defining the target.

\begin{proof}[Proof of \cref{thm:s-halfspaces}]
Fix $\eps,\delta\in(0,1)$, a target set
\[
    P^\star
    =
    \bigcap_{t=1}^{s} H_t^\star,
    \qquad
    H_t^\star
    =
    \inbrace{
        x\in \R^d :
        \langle n_t^\star,x\rangle \geq c_t^\star
    },
\]
and a distribution $\cD$ compatible with $P^\star$ (\cref{def:compatible}).
We allow repetitions among the halfspaces $H_t^\star$, so this representation exists even when $P^\star$ is defined by fewer than $s$ irredundant halfspaces.
As in the proof of \cref{thm:halfspace}, draw $m$ independent examples $x_1,\dots,x_m\sim \cD$ (with $m$ to be specified), and define the sample set $P\coloneqq \inbrace{x_1,\dots,x_m}$, its convex hull $C\coloneqq \conv(P)$, the affine hull $A\coloneqq \aff(C)$, the linear space $L\coloneqq \inbrace{x-y:x,y\in A}$, and the intrinsic dimension $r\coloneqq \dim(L)$.

If $r=0$, then $C$ is a singleton and the learner returns $\widehat h(x)\coloneqq \one\!\inbrace{x\in C}$.
We henceforth assume $r\geq 1$.
We reuse the family of candidate simplices from \cref{thm:halfspace}'s proof:
\[
    \mathcal{S}(C)
    \coloneqq
    \inbrace{
        S_{v,J} :
        v \text{ is a vertex of } C,\;
        J\subseteq \cF(v),\;
        \abs{J}=r,\;
        \text{the normals in } J \text{ are linearly independent}
    }.
\]
Recall from \cref{lem:halfspace:contains-hull} that every $S_{v,J}\in\mathcal{S}(C)$ contains $C$, and from \cref{lem:halfspace:simplex} that each $S_{v,J}$ is an $r$-simplex in $A$.
For each ordered $s$-tuple $\sigma
    =
    \inparen{S_1,\dots,S_s}
    \in
    \mathcal{S}(C)^s,$
define the candidate polytope
\[
    P_\sigma
    \coloneqq
    \bigcap_{t=1}^{s} S_t.
\]
Since each $S_t\supseteq C$, every candidate satisfies $P_\sigma\supseteq C$.
\begin{lemma}[A Safe Candidate Polytope Exists]
\label{lem:s-halfspaces:safe-candidate}
There exists $\sigma^\star\in \mathcal{S}(C)^s$ such that $C\subseteq P_{\sigma^\star}\subseteq P^\star.$
\end{lemma}
\begin{proof}[Proof of \cref{lem:s-halfspaces:safe-candidate}]
For each $t\in\insquare{s}$, the distribution $\cD$ is compatible with $H_t^\star$ (since $P^\star\subseteq H_t^\star$).
By \cref{lem:halfspace:safe-candidate}, there exists $S_t^\star\in \mathcal{S}(C)$ with $S_t^\star\subseteq H_t^\star$.
Setting $\sigma^\star\coloneqq \inparen{S_1^\star,\dots,S_s^\star}$, we have
\[
    P_{\sigma^\star}
    =
    \bigcap_{t=1}^{s} S_t^\star
    \subseteq
    \bigcap_{t=1}^{s} H_t^\star
    =
    P^\star.
\]
Since each $S_t^\star\supseteq C$, we also have $P_{\sigma^\star}\supseteq C$.
\end{proof} 

\begin{lemma}[Candidate Polytopes Have Few Vertices]
\label{lem:s-halfspaces:vertex-bound}
Every candidate polytope $P_\sigma$ has at most $\binom{s(r+1)}{r}$
vertices.
\end{lemma}
\begin{proof}
    Fix $\sigma=\inparen{S_1,\dots,S_s}\in \mathcal{S}(C)^s$.
    Each $S_t$ is an $r$-simplex in $A$ (by \cref{lem:halfspace:simplex}), so it has exactly $r+1$ facets.
    Therefore $P_\sigma=\bigcap_{t=1}^{s} S_t$ is an $r$-dimensional polytope in $A$ cut out by at most $s(r+1)$ facet-defining inequalities (it is $r$-dimensional because $C\subseteq P_\sigma\subseteq A$ and $\dim(C)=r$).
    At each vertex $u$ of $P_\sigma$, choose $r$ active facet inequalities whose normals are linearly independent.
    The corresponding $r$ hyperplanes meet in the unique point $u$.
    Hence distinct vertices can be assigned distinct $r$-subsets of the at most $s(r+1)$ facets, so the number of vertices is at most $\binom{s(r+1)}{r}$.
\end{proof} 

\paragraph{Learning Algorithm.}
If $r=0$, return the hypothesis $\widehat h(x)=\one\!\inbrace{x\in C}$.
Otherwise, enumerate all ordered $s$-tuples $\sigma\in \mathcal{S}(C)^s.$
For each $\sigma$, compute $P_\sigma$, query the membership oracle on every vertex of $P_\sigma$, and return the first candidate all of whose vertices receive label $1$.
Write $\widehat P$ for the returned candidate and $\widehat h(x)\coloneqq \one\!\sinbrace{x\in \widehat P}$
for the output hypothesis.

\begin{lemma}[One-Sided Error]
\label{lem:s-halfspaces:one-sided}
The learner always halts, and its output hypothesis $\widehat h$ has no false positives.
Moreover, $P\subseteq C\subseteq \widehat P.$
\end{lemma}
\begin{proof}
By \cref{lem:s-halfspaces:safe-candidate}, there exists $\sigma^\star\in \mathcal{S}(C)^s$ with $P_{\sigma^\star}\subseteq P^\star$.
All vertices of $P_{\sigma^\star}$ lie in $P^\star$ and therefore receive label $1$, so the learner eventually halts.
Now let $\widehat P$ be the returned candidate.
Every vertex of $\widehat P$ received label $1$ and hence lies in $P^\star$.
Since $P^\star$ is convex and $\widehat P$ is the convex hull of its vertices, $\widehat P\subseteq P^\star$, and so $\widehat h$ has no false positives.
Finally, $P\subseteq C\subseteq \widehat P$ because every candidate polytope contains $C$.
\end{proof} 
To control the false-negative error, we bound the VC dimension of the class of sets the learner can output. Define
\[
    \hyT_{d,s}
    \coloneqq
    \inbrace{
        T_1\cap \cdots \cap T_s :
        T_1,\dots,T_s \text{ are affine simplices of dimension at most } d
    }.
\]
Every output $\widehat P$ belongs to $\hyT_{d,s}$.

\begin{lemma}[The Output Class Has Small VC Dimension]
\label{lem:s-halfspaces:vc}
If $\hyG_{d,s}
    \coloneqq
    \inbrace{
        x\mapsto \one\!\inbrace{x\notin T} :
        T\in \hyT_{d,s}
    },$
then $\vc(\hyG_{d,s})
    =
    \vc(\hyT_{d,s})
    =
    O\!\inparen{s d^2\log\!\inparen{s d}}.$
\end{lemma}
\begin{proof}
Complementation preserves VC dimension, so it suffices to bound $\vc(\hyT_{d,s})$.
Each simplex in $\R^d$ can be written as the intersection of at most $2d+1$ halfspaces, so every $T\in \hyT_{d,s}$ is the intersection of at most $s(2d+1)$ halfspaces.
By the standard VC bound for intersections of halfspaces in $\R^d$,
\[
    \vc(\hyT_{d,s})
    =
    O\!\inparen{
        d\cdot s(2d+1)\cdot \log\!\inparen{s(2d+1)}
    }
    =
    O\!\inparen{
        s d^2\log\!\inparen{s d}
    }.\qedhere
\]
\end{proof}

\begin{lemma}[False-Negative Error]
\label{lem:s-halfspaces:false-negative}
If $m
    =
    O\!\inparen{
        (\nfrac{
            1
        }{\eps})\cdot \inparen{s d^2 \log\!\inparen{s d}\cdot \log\!\nfrac1\eps
            +
            \log\!\nfrac1\delta}
    },$
then with probability at least $1-\delta$, $\Pr_{x\sim \cD}\inparen{x\notin \widehat P}\leq \eps.$
Consequently, $\Pr_{x\sim \cD}\inparen{\widehat h(x)=0 \text{ and } P^\star(x)=1}\leq \eps.$
\end{lemma}

\begin{proof}
By \cref{lem:s-halfspaces:one-sided}, $P\subseteq \widehat P$, so the function $g_{\widehat P}(x)\coloneqq \one\!\sinbrace{x\notin \widehat P}\in \hyG_{d,s}$ has empirical mean zero on the sample.
By \cref{lem:s-halfspaces:vc} and the standard VC bound for zero empirical error, if
\[
    m
    =
    O\!\inparen{
        \frac{
            \vc(\hyG_{d,s})\log\!\nfrac1\eps
            +
            \log\!\nfrac1\delta
        }{\eps}
    }
    =
    O\!\inparen{
        \frac{
            s d^2 \log\!\inparen{s d}\cdot \log\!\nfrac1\eps
            +
            \log\!\nfrac1\delta
        }{\eps}
    },
\]
then with probability at least $1-\delta$, $\Pr_{x\sim \cD}\inparen{x\notin \widehat P}\leq \eps.$
Since $\widehat P\subseteq P^\star$ (\cref{lem:s-halfspaces:one-sided}) and $\cD$ is supported on $P^\star$,
\[
    \Pr_{x\sim \cD}\inparen{\widehat h(x)=0 \text{ and } P^\star(x)=1}
    =
    \Pr_{x\sim \cD}\inparen{x\notin \widehat P}
    \leq
    \eps.\qedhere
\]
\end{proof}

\begin{lemma}[Query Complexity]
\label{lem:s-halfspaces:query}
The learner uses at most $(2es)^d\,m^{s(d^2+1)}$ membership queries.
\end{lemma}
\begin{proof}
If $r=0$, no membership queries are made.
Assume $r\geq 1$.
By the counting in the proof of \cref{lem:halfspace:query}, $\abs{\mathcal{S}(C)}\leq m^{d^2+1}$, so the number of $s$-tuples is at most $m^{s(d^2+1)}$.
By \cref{lem:s-halfspaces:vertex-bound}, each $P_\sigma$ has at most $\binom{s(r+1)}{r}$ vertices.
Using $\binom{n}{k}\leq \inparen{\frac{en}{k}}^k$,
\[
\binom{s(r+1)}{r}\leq \inparen{\frac{es(r+1)}{r}}^r \leq (2es)^r \leq (2es)^d,
\]
since $r\geq 1$ implies $(r+1)/r\leq 2$.
The learner queries at most $(2es)^d$ vertices per candidate across at most $m^{s(d^2+1)}$ candidates, giving at most $(2es)^d\,m^{s(d^2+1)}$ membership queries in total.
\end{proof} 
The theorem follows by combining \cref{lem:s-halfspaces:one-sided,lem:s-halfspaces:false-negative,lem:s-halfspaces:query}. 
\end{proof}

\subsection{Learnability in Valiant's Model Is Closed under Unions} 
In this section, we show that learnability is closed under unions. 
\begin{theorem}[Learnability Is Closed under Unions]
\label{thm:valiant:unionClosedness}
Let $\hyH=\inparen{\hyH_1,\hyH_2,\dots}$ and $\hyG=\inparen{\hyG_1,\hyG_2,\dots}$
be sequences of concept classes that are learnable in Valiant's model (\cref{def:valiant}).
Then their union $\hyH\cup \hyG
    \coloneqq
    \inparen{\hyH_1\cup \hyG_1,\hyH_2\cup \hyG_2,\dots}$
is also learnable in Valiant's model.
More precisely, if $\hyH_d$ and $\hyG_d$ admit query compression schemes of size $q_{\hyH}(d)$ and $q_{\hyG}(d)$ respectively, then there is a learner for $\hyH\cup \hyG$ that, for every $d,\eps,\delta$, uses at most 
\[
    O\!\inparen{
        \frac{
            d\inparen{q_{\hyH}(d)+q_{\hyG}(d)}
            +
            \log\nfrac{1}{\delta}
        }{\eps}
    } \text{examples}
    \qquadand
    q_{\hyH}(d)+q_{\hyG}(d)
    \text{ membership queries}\,.
\]
\end{theorem}

\paragraph{Setup.}
For $\mathcal K\in\inbrace{\hyH,\hyG}$ and a transcript $\inparen{q,r}$, write
$\V_{\mathcal K}\!\inparen{q,r}
    \coloneqq
    \inbrace{f\in \mathcal K : f(q)=r},$
and similarly for $\V_{\mathcal K}\!\inparen{S}$ and $\V_{\mathcal K}\!\inparen{\sigma,r}$.
Fix a finite set $S\subseteq \zo^d$ and a target $h^\star\in \hyH\cup \hyG$.
For each $\mathcal K\in\inbrace{\hyH,\hyG}$, if $\V_{\mathcal K}\!\inparen{S}\neq\emptyset$, let $\Sigma_{S,\mathcal K}$ be a query compression scheme for $S$ of size at most $k_{\mathcal K}$, as guaranteed by \cref{thm:characterization}; otherwise, let $\Sigma_{S,\mathcal K}$ be an arbitrary depth-$0$ strategy.
Run $\Sigma_{S,\mathcal K}$ against the membership oracle for $h^\star$ and let $r_{S,\mathcal K}$ denote the realized transcript.
Define
\[
    D_{S,\mathcal K}
    \coloneqq
    \begin{cases}
        \displaystyle\bigcap_{f\in \V_{\mathcal K}\!\inparen{\Sigma_{S,\mathcal K},r_{S,\mathcal K}}}\supp(f)
        & \text{if } \V_{\mathcal K}\!\inparen{\Sigma_{S,\mathcal K},r_{S,\mathcal K}}\neq\emptyset,\\[6pt]
        \zo^d & \text{otherwise}.
    \end{cases}
\]
First, we run the query compression scheme for each of the two classes $\hyH$ and $\hyG$ separately, and show that the resulting version-space closure either certifies the sample or rules out that class as containing the target.
\begin{lemma}
\label{lem:valiant:unionClosure:classSpecific}
For $\mathcal K\in\inbrace{\hyH,\hyG}$, the following hold:
\begin{enumerate}[itemsep=0pt]
    \item if $h^\star\in \mathcal K$, then $S\subseteq D_{S,\mathcal K}\subseteq \supp\!\inparen{h^\star}$;
    \item if $S\not\subseteq D_{S,\mathcal K}$, then $h^\star\notin \mathcal K$.
\end{enumerate}
\end{lemma}

\begin{proof}
Suppose $h^\star\in \mathcal K$.
Every point in $S$ is positive for $h^\star$, so $h^\star\in \V_{\mathcal K}\!\inparen{S}$.
Hence $\Sigma_{S,\mathcal K}$ is a valid query compression scheme for $S$, and the realized transcript $r_{S,\mathcal K}$ is realizable with respect to $\V_{\mathcal K}\!\inparen{S}$.
The defining property of query compression gives $S\subseteq D_{S,\mathcal K}$.
Moreover, since $h^\star$ itself generated the transcript $r_{S,\mathcal K}$, we have $h^\star\in \V_{\mathcal K}\!\inparen{\Sigma_{S,\mathcal K},r_{S,\mathcal K}}$, and therefore $D_{S,\mathcal K}\subseteq \supp\!\inparen{h^\star}$.
This proves the first item; the second is its contrapositive.
\end{proof}
Next, we discard any closure that fails to contain the sample and take the intersection of the survivors, and show that the resulting set $C_S$ still satisfies $S\subseteq C_S\subseteq \supp\!\inparen{h^\star}$.
\cref{lem:valiant:unionClosure:classSpecific} provides exactly the dichotomy we need: for each class, the corresponding closure either certifies the sample or certifies that the target does not belong to that class.
We now combine the two candidates.
For $\mathcal K\in\inbrace{\hyH,\hyG}$, set
\[
    C_{S,\mathcal K}
    \coloneqq
    \begin{cases}
        D_{S,\mathcal K} & \text{if } S\subseteq D_{S,\mathcal K},\\
        \zo^d & \text{otherwise},
    \end{cases}
\]
and define $C_S\coloneqq C_{S,\hyH}\cap C_{S,\hyG}$.

\begin{lemma}
\label{lem:valiant:unionClosure:combine}
The set $C_S$ satisfies $S\subseteq C_S\subseteq \supp\!\inparen{h^\star}$.
\end{lemma}

\begin{proof}
Without loss of generality, suppose $h^\star\in \hyH$.
By \cref{lem:valiant:unionClosure:classSpecific}, $S\subseteq D_{S,\hyH}\subseteq \supp\!\inparen{h^\star}$, so $C_{S,\hyH}=D_{S,\hyH}$.
For the $\hyG$-side, either $S\subseteq D_{S,\hyG}$, in which case $C_{S,\hyG}=D_{S,\hyG}$, or $S\not\subseteq D_{S,\hyG}$, in which case $C_{S,\hyG}=\zo^d$.
In either case, $S\subseteq C_{S,\hyG}$.
Combining gives
\[
    S\subseteq C_{S,\hyH}\cap C_{S,\hyG}=C_S
    \qquad\text{and}\qquad
    C_S\subseteq C_{S,\hyH}\subseteq \supp\!\inparen{h^\star}.\qedhere
\]
\end{proof}
Finally, we transfer this empirical guarantee to the population level via a finite-class uniform convergence argument.
\cref{lem:valiant:unionClosure:combine} shows that after running both class-specific procedures and discarding any candidate that fails to contain the sample, the intersection $C_S$ is a one-sided consistent hypothesis with zero empirical false-negative error.
It remains to show that this guarantee transfers to the population level.

\begin{proof}[Proof of \cref{thm:valiant:unionClosedness}]
Fix $\eps,\delta\in(0,1)$, a target $h^\star\in \hyH\cup \hyG$, and a distribution $\cD$ compatible with $h^\star$.
Draw an ordered sample $x_1,\dots,x_n\sim \cD$ with
\[
    n
    \geq
    \Omega\!\inparen{
        \frac{
            d\inparen{k_{\hyH}+k_{\hyG}}
            +
            \log\nfrac{1}{\delta}
        }{\eps}
    },
\]
and let $S\coloneqq \inbrace{x_1,\dots,x_n}$ be the underlying set of distinct sample points.
Construct $D_{S,\hyH},D_{S,\hyG},C_{S,\hyH},C_{S,\hyG}$, and $C_S$ as above, and output the indicator of $C_S$.
By construction, the learner uses exactly $n$ examples and at most $k_{\hyH}+k_{\hyG}$ membership queries.
By \cref{lem:valiant:unionClosure:combine}, the output satisfies $S\subseteq C_S\subseteq \supp\!\inparen{h^\star}$, so it has no false positives and zero empirical false-negative error.

It remains to bound the population false-negative rate.
Consider the finite class
\[
    \cC
    \coloneqq
    \inbrace{
        A\cap B
        :
        A\in \operatorname{Cl}_{k_{\hyH}}\!\inparen{\hyH_d}\cup \inbrace{\zo^d},
        \;
        B\in \operatorname{Cl}_{k_{\hyG}}\!\inparen{\hyG_d}\cup \inbrace{\zo^d}
    }.
\]
By construction, $C_S\in \cC$, and \cref{lem:valiant:kQueryClosureClass:VC} gives
$\log\abs{\cC}
    \leq
    O\!\inparen{d\inparen{k_{\hyH}+k_{\hyG}}}.$
The rest of the proof is identical to the final step in the proof of \cref{thm:valiant:learningFromCompression}.
Applying the finite-class realizable-case uniform convergence argument to
$\inbrace{
        C\triangle \supp\!\inparen{h^\star}
        :
        C\in \cC
    }$
shows that, with probability at least $1-\delta$,
$\cD\!\inparen{C_S\triangle \supp\!\inparen{h^\star}}
    \leq
    \eps.$
Since $C_S\subseteq \supp\!\inparen{h^\star}$ and $\supp\!\inparen{\cD}\subseteq \supp\!\inparen{h^\star}$, this implies
$ \Pr_{x\sim \cD}\inparen{C_S(x)=0 \text{~and~} h^\star(x)=1}\leq \eps.$
The output therefore satisfies the guarantees in \cref{def:valiant} with probability at least $1-\delta$.
\end{proof}
 
    \vspace{-5mm}
    \subsection{A Finite-VC Class whose 1-Point Closure has Infinite VC Dimension}
        In this section, we prove that when the domain is infinite, the VC dimension of the $k$-query closure class (\cref{def:valiant:kQueryClosureClass}) can be much larger than that of the base class itself, even for $k=1$.

        For each nonempty subfamily $\cG\subseteq \hyH$ and for a point $p\in \cX$, define
    \[
        \closure\!\inparen{\cG}
        \coloneqq
        \bigcap\nolimits_{h\in \cG}\supp(h)
        \qquadand
        \V\!\inparen{p,\hyH}
        \coloneqq
        \inbrace{h\in \hyH : p\in \supp(h)}.
    \]
We prove the following result.
\begin{theorem}
    \label{thm:onePointClosureInfiniteVC}
    There exists a countably infinite domain $\cX$ and a hypothesis class $\hyH\subseteq \zo^\cX$ such that $\vc\!\inparen{\hyH}=2$, but the $1$-point closure class $\operatorname{Cl}_1(\hyH)
        \coloneqq
        \inbrace{\closure\!\inparen{\V\!\inparen{p,\hyH}} : p\in\cX}$ (\cref{def:valiant:kQueryClosureClass}) has infinite VC dimension.
\end{theorem}
\begin{proof}
Let 
\[
    Q\coloneqq \inbrace{q_i : i\in \N}
    \qquad\text{and}\qquad
    R\coloneqq \inbrace{r_A : A\subseteq \N \text{ finite}},
\]
and set $\cX\coloneqq Q\cup R$.
Since the family of finite subsets of $\N$ is countable, both $Q$ and $R$ are countable, and therefore $\cX$ is countably infinite.
For every finite set $A\subseteq \N$ and every $i\in \N\setminus A$, define a hypothesis $h_{A,i}\colon \cX\to \zo$ by
\[
    h_{A,i}(r_B)
    \coloneqq
    \one\!\inbrace{B=A}
    \qquad\text{and}\qquad
    h_{A,i}(q_j)
    \coloneqq
    \one\!\inbrace{j\neq i}.
\]
Let
\[
    \hyH
    \coloneqq
    \inbrace{h_{A,i} : A\subseteq \N \text{ finite and } i\in \N\setminus A}.
\]
We first show that $\vc\!\inparen{\hyH}=2$.
\begin{claim}
    It holds that $\vc\!\inparen{\hyH}=2$.
\end{claim}
\begin{proof}
    The set $\inbrace{r_{\emptyset},q_1}$ is shattered, since
$h_{\emptyset,2}, h_{\emptyset,1}, h_{\inbrace{1},2}, h_{\inbrace{2},1}$
realize the four labelings $(1,1)$, $(1,0)$, $(0,1)$, and $(0,0)$ on $\inbrace{r_{\emptyset},q_1}$, respectively.
Thus $\vc\!\inparen{\hyH}\geq 2$.
Conversely, no $3$-point subset of $\cX$ is shattered by $\hyH$.
Indeed, if $T\subseteq \cX$ has size $3$, then $T$ contains either at least two points of $R$ or at least two points of $Q$.
In the first case, no hypothesis in $\hyH$ labels both of those $R$-points by $1$, since every $h_{A,i}$ is positive on exactly one point of $R$.
In the second case, no hypothesis in $\hyH$ labels both of those $Q$-points by $0$, since every $h_{A,i}$ is negative on exactly one point of $Q$.
So $T$ is not shattered, and therefore $\vc\!\inparen{\hyH}\leq 2$.
Hence $\vc\!\inparen{\hyH}=2$.
\end{proof}
Now fix a finite set $A\subseteq \N$.
We claim that
\[
    \closure\!\inparen{\V\!\inparen{r_A,\hyH}}
    =
    \inbrace{r_A}\cup \inbrace{q_i : i\in A}.
\]
Indeed, $\V\!\inparen{r_A,\hyH}=\inbrace{h_{A,i} : i\in \N\setminus A}$, since a hypothesis contains $r_A$ if and only if its first index is exactly $A$.
Every hypothesis in this family contains $r_A$, so $r_A\in \closure\!\inparen{\V\!\inparen{r_A,\hyH}}$.
If $j\in A$, then every $h_{A,i}$ with $i\in \N\setminus A$ satisfies $i\neq j$, and hence $h_{A,i}(q_j)=1$; thus $q_j\in \closure\!\inparen{\V\!\inparen{r_A,\hyH}}$.
If $j\notin A$, then $h_{A,j}\in \V\!\inparen{r_A,\hyH}$ and $h_{A,j}(q_j)=0$, so $q_j\notin \closure\!\inparen{\V\!\inparen{r_A,\hyH}}$.
Finally, if $B\neq A$, then every hypothesis in $\V\!\inparen{r_A,\hyH}$ labels $r_B$ by $0$, so $r_B\notin \closure\!\inparen{\V\!\inparen{r_A,\hyH}}$.
This proves the claim.

To conclude, let $T_n\coloneqq \inbrace{q_1,\dots,q_n}$.
We show that $T_n$ is shattered by $\operatorname{Cl}_1(\hyH)$ for every $n\in \N$.
Fix any subset $B\subseteq \inbrace{1,\dots,n}$.
Since $B$ is finite, the point $r_B$ belongs to $\cX$, and by the previous paragraph,
\[
    \closure\!\inparen{\V\!\inparen{r_B,\hyH}}
    =
    \inbrace{r_B}\cup \inbrace{q_i : i\in B}.
\]
Therefore
\[
    \closure\!\inparen{\V\!\inparen{r_B,\hyH}}\cap T_n
    =
    \inbrace{q_i : i\in B}.
\]
Thus every subset of $T_n$ is realized by some member of $\operatorname{Cl}_1(\hyH)$, so $T_n$ is shattered by $\operatorname{Cl}_1(\hyH)$.
Since this holds for every $n$, we conclude that $\vc\!\inparen{\operatorname{Cl}_1(\hyH)}=\infty$.
\end{proof} 
\vspace{-7mm}

\section{Proof of \cref{thm:valiant:boosting} (Boosting; Deferred from \cref{sec:proofof:thm:characterization})}
        \label{sec:proofof:thm:valiant:boosting}
        In this section, we prove \cref{thm:valiant:boosting}.
        \mbox{Toward this, we first prove the following intermediate result:}
        \vspace{-7mm}
\begin{lemma}[Confidence Amplification]
    \label{lem:valiant:confidenceBoost:boolean}
    Suppose $\hyH$ is learnable in Valiant's model (\cref{def:valiant}), and let
    $m_{\mathrm W}\inparen{d,\eps,\delta}$ denote the sample/query complexity of some learner witnessing this.
    Fix
    \[
        m
        \coloneqq
        m_{\mathrm W}\!\inparen{d,\nfrac{1}{18},\nfrac13}.
        \yesnum\label{eq:valiant:weakComplexity}
    \]
	    Then for every $\eta\in (0,\nfrac12]$, there is a learner $\cL_{\eta}$, s.t., for every target $\hstar\in\hyH_d$ and every distribution $\mu$ compatible with $\hstar$,
    with probability at least $1-\eta$ the learner $\cL_{\eta}$ outputs a hypothesis satisfying the Valiant conditions with error parameter $\nfrac13$, while using
        $O\!\inparen{m\log\nfrac1\eta}$
    example-oracle calls and membership queries.
\end{lemma}

\begin{proof}
	    Fix $d\in\N$, $\eta\in(0,\nfrac12]$, a target $\hstar\in\hyH_d$, and a distribution $\mu$ compatible with $\hstar$.
    Let $\cL_0$ be a learner witnessing the learnability of $\hyH$ in Valiant's model with sample/query complexity $m_{\mathrm W}$.
    Thus, when run with parameters $\inparen{d,\nfrac{1}{18},\nfrac13}$, the learner $\cL_0$ uses at most $m$ example-oracle calls and at most $m$ membership queries, and with probability at least $\nfrac23$ outputs a hypothesis with no false positives and false-negative mass at most $\nfrac{1}{18}$ under $\mu$.
    Let
    \[
        k
        \coloneqq
        C\log\frac1\eta,
    \]
    where $C>0$ is a sufficiently large absolute constant.
    Run $\cL_0$ independently $k$ times on fresh data and let $h_1,\dots,h_k$ be the resulting hypotheses.
    Output their majority:
        $h(x)
        \coloneqq
        \one\!\inbrace{\sum_{j=1}^k h_j(x)\geq \nfrac{k}{2}}.$
    Call run $j$ \emph{successful} if: 
    \[
        \text{$h_j$ has no false positives}
        \qquadand
        \Pr_{x\sim\mu}\inparen{h_j(x)=0}\leq \frac{1}{18}.
        \tag{Successful run}
    \]
    Each run is successful with probability at least $\nfrac23$, and these events are independent across $j$.
    Hence, by a Chernoff bound, with probability at least $1-\eta$ at least $\nfrac{3k}{5}$ runs are successful.
    Let $\evE$ denote this event.

    \paragraph{No false positives.}
    Assume $\evE$ holds and fix any $x\in\zo^d$ with $\hstar(x)=0$.
    Every successful $h_j$ satisfies $h_j(x)=0$, so at most the unsuccessful runs can output $1$ on $x$.
    Since there are at most $\nfrac{2k}{5}$ unsuccessful runs, fewer than half of the $h_j$'s output $1$ on $x$.
    Therefore $h(x)=0$.
    Thus $h$ has no false positives.

    \paragraph{Upper bound on false negatives.}
    Assume $\evE$ and let $G\subseteq \insquare{k}$ be the set of successful indices.
    Then $\abs{G}\geq \nfrac{3k}{5}$.
    Fix any $x\in\zo^d$ with $h(x)=0$.
    At least $\nfrac{k}{2}$ of the values $h_j(x)$ are $0$.
    Even if every unsuccessful run outputs $0$, at least
    $\frac{k}{2}-\inparen{k-\abs{G}}
        =
        \abs{G}-\frac{k}{2}
        \geq
        \frac{k}{10}$
	    successful runs must also output $0$.
	    Since $\abs{G}\geq \nfrac{3k}{5}$, this quantity is also at least $\nfrac{\abs{G}}{6}$.
	    Therefore, $\one\!\inbrace{h(x)=0}
	        \leq
	        6\cdot \frac{1}{\abs{G}}\sum_{j\in G} \one\!\inbrace{h_j(x)=0}.$
    Taking expectation over $x\sim\mu$ gives
    \[
        \Pr_{x\sim\mu}\inparen{h(x)=0}
        \leq
        6\cdot \frac{1}{\abs{G}}\sum_{j\in G} \Pr_{x\sim\mu}\inparen{h_j(x)=0}
        \leq
        6\cdot \frac{1}{18}
        =
        \frac13.
    \] 
    The learner makes $k$ independent runs of $\cL_0$, each using at most $m$ example-oracle calls and at most $m$ membership queries.
    Hence the total sample/query complexity is
    $O\!\inparen{m\log\nfrac1\eta}$
    for each type of oracle.
    By the preceding no-false-positive and false-negative arguments, on the event $\evE$ the output $h$ satisfies the Valiant conditions with error parameter $\nfrac13$.
    This suffices to complete the proof since $\Pr\inparen{\evE}\geq 1-\eta$.
\end{proof}

\begin{proof}[Proof of \cref{thm:valiant:boosting}]
	    Fix $\delta\in(0,\nfrac12]$.
	    Fix $d\in\N$, $\eps\in (0,1)$, a target $\hstar\in\hyH_d$, and a compatible distribution $\cD$.
    Recall that $m
        \coloneqq
        m_{\mathrm W}\!\inparen{d,\nfrac{1}{18},\nfrac13}$ from \eqref{eq:valiant:weakComplexity}.
	    By \cref{lem:valiant:confidenceBoost:boolean}, for every $\eta\in(0,\nfrac12]$ there is a learner $\cB_{\eta}$ such that on every compatible distribution $\mu$,
    with probability at least $1-\eta$, the learner $\cB_{\eta}$ outputs a hypothesis $g$ satisfying the Valiant conditions with error parameter $\nfrac13$, while using $O\!\inparen{m\log\nfrac1\eta}$
    example-oracle calls and at most the same number of membership queries.
    By the structural observation preceding this section, we may assume that $\cB_{\eta}$ first draws all its example samples and only then makes its membership queries.
    Let
    \[
        T\coloneqq \left\lceil \log_3\nfrac1\eps\right\rceil
        \qquadand
        \tau_s\coloneqq 3^{-s}
        \quad \text{for } s\in\inbrace{0,1,\dots,T},
    \]
    and define weights
    \[
        w_s\coloneqq \frac{1}{\tau_s}=3^s
        \qquadand 
        W\coloneqq \sum_{s=0}^T w_s.
    \]
    Then $W=O\!\inparen{\nfrac1\eps}$.
    For each scale $s\in\inbrace{0,1,\dots,T}$, let
    \[
        \eta_s\coloneqq \frac{\delta w_s}{100W}.
    \]
    Since $\sum_{s=0}^T w_s=W$, we have $\sum_{s=0}^T \eta_s\leq \frac{\delta}{100}.$
    We now describe the learner.
    Initialize $H_0\equiv 0$.
    For each scale $s=0,1,\dots,T$, given $H_s$, let
    \[
        p_s\coloneqq \Pr_{x\sim\cD}\inparen{H_s(x)=0}.
    \]
    Because $\cD$ is compatible with $\hstar$, this is exactly the $\cD$-mass of the current residual positive region.

    \paragraph{Testing the residual mass.}
    Draw $A_s\coloneqq C_1\frac{\log{\nfrac{1}{\eta_s}}}{\tau_s}$
    fresh samples from $\cD$, where $C_1>0$ is a sufficiently large absolute constant.
    Let $\widehat{p}_s$ be the fraction of these samples on which $H_s(x)=0$.
    By a Chernoff bound, with probability at least $1-\eta_s$ the following two implications hold simultaneously:
    \begin{align*}
        p_s\geq \tau_s
        &\Longrightarrow
        \widehat{p}_s\geq \frac{2\tau_s}{3}
        \qquadand
        p_s\leq \frac{\tau_s}{3}
        \Longrightarrow
        \widehat{p}_s<\frac{2\tau_s}{3}.
        \yesnum\label{eq:valiant:boosting:residual-test}
    \end{align*}
    If $\widehat{p}_s<\nfrac{2\tau_s}{3}$, set $H_{s+1}\coloneqq H_s$ and continue to the next scale.

    \paragraph{Simulating the residual oracle.}
    Otherwise, draw fresh samples from $\cD$ until either
    $O\!\inparen{m\log\nfrac1{\eta_s}}$
    of them satisfy $H_s(x)=0$, or until a budget of
    $B_s
        \coloneqq
        C_2\frac{m\log{\nfrac{1}{\eta_s}}+\log{\nfrac{1}{\eta_s}}}{\tau_s}$
    original example-oracle calls has been exhausted, where $C_2>0$ is a sufficiently large absolute constant.
    If the budget is exhausted first, declare failure.
    Otherwise, let $S_s$ be the first
    $O\!\inparen{m\log\nfrac1{\eta_s}}$
    retained points.
    These are i.i.d. samples from the conditional distribution of $\cD$ given $H_s(x)=0$.
	    Moreover, on the residual-test event in \eqref{eq:valiant:boosting:residual-test}, whenever we simulate the residual oracle we have $p_s>\tau_s/3$; hence, conditioned on the past, by another Chernoff bound and for $C_2$ large enough, the probability that the residual-oracle simulation fails is at most $\eta_s$.

    \paragraph{Learning on the residual distribution.}
    Run $\cB_{\eta_s}$ on the simulated residual examples $S_s$ and the membership oracle for $\hstar$, and let $h_s$ be the output.
    Set
    \[
        H_{s+1}(x)\coloneqq H_s(x)\vee h_s(x).
    \] 
    Let $\evE$ be the event that every residual-mass test, every residual-oracle simulation, and every call to $\cB_{\eta_s}$ succeeds.
    By a union bound,
        $\Pr\inparen{\evE}
        \geq
        1-3\sum_{s=0}^T \eta_s
        \geq
        1-\delta,$
    after adjusting the absolute constants.

    \paragraph{An Invariant.}
    On $\evE$, for every $s\in\inbrace{0,1,\dots,T+1}$ the hypothesis $H_s$ has no false positives.
    Moreover, for every $s\in\inbrace{1,2,\dots,T+1}$,
    \[
        p_s\leq \tau_{s-1}.
        \yesnum\label{eq:valiant:boosting:invariant}
    \]
    We prove this by induction on $s$.
    The statement for $s=0$ is trivial.
    Assume it holds for some $s\in\inbrace{0,1,\dots,T}$.
    If the residual-mass test skips scale $s$, then on $\evE$ either $p_s\leq \tau_s/3$ or $p_s\inparen{\tau_s/3,\tau_s}$.
    In either case, $p_{s+1}=p_s\leq \tau_s,$
    and $H_{s+1}=H_s$ still has no false positives.

    Suppose instead that we simulate the residual oracle.
    On $\evE$, the residual-test guarantee \eqref{eq:valiant:boosting:residual-test} guarantees that $p_s>\tau_s/3$.
    Let $\mu_s$ denote $\cD$ conditioned on $H_s(x)=0$.
    Since $\cD$ is compatible with $\hstar$ and $H_s$ has no false positives, the distribution $\mu_s$ is also compatible with $\hstar$.
    Because the call to $\cB_{\eta_s}$ succeeds on $\evE$, the hypothesis $h_s$ has no false positives and
    $\Pr_{x\sim\mu_s}\inparen{h_s(x)=0}\leq \frac13.$
    Therefore $H_{s+1}=H_s\vee h_s$ has no false positives, and
    \[
        p_{s+1}
        =
        \Pr_{x\sim\cD}\inparen{H_s(x)=0 \text{ and } h_s(x)=0}
        =
        p_s\Pr_{x\sim\mu_s}\inparen{h_s(x)=0}
        \leq
        \frac{p_s}{3}.
    \]
    If $s=0$, then $p_0\leq 1$, so
    $p_1\leq \frac13=\tau_0.$
    If $s\geq 1$, then the induction hypothesis gives $p_s\leq \tau_{s-1}$, and hence
    $ p_{s+1}
        \leq
        \frac{\tau_{s-1}}{3}
        =
        \tau_s.$
    This completes the induction.

    \paragraph{Correctness.}
    By the no-false-positive invariant and \eqref{eq:valiant:boosting:invariant}, on $\evE$ the final hypothesis $H_{T+1}$ has no false positives and
    \[
        \Pr_{x\sim\cD}\inparen{H_{T+1}(x)=0}
        =
        p_{T+1}
        \leq
        \tau_T
        =
        3^{-T}
        \leq
        \eps.
    \]
    Since $\cD$ is compatible with $\hstar$, this is exactly the required false-negative guarantee.

    \paragraph{Membership-query complexity.}
    Each invocation of $\cB_{\eta_s}$ uses
        $O\!\inparen{m\log\nfrac1{\eta_s}}$
    membership queries.
    Hence the total number of membership queries is at most
        $O\!\inparen{m\sum_{s=0}^T \log\nfrac1{\eta_s}}.$
    Since $\eta_s\geq \delta/(100W)$ for every $s$, this is
        $m\cdot O\!\inparen{T\log\nfrac{100W}{\delta}}.$
    Now $T=O\!\inparen{\log\nfrac1\eps}$ and $W=O\!\inparen{\nfrac1\eps}$, so
    \[
        T\log\frac{100W}{\delta}
        =
        O\!\inparen{\log^2\frac1\eps+\log\frac1\eps\log\frac1\delta}
        =
        O\!\inparen{\frac{1}{\eps}\log\frac1\delta},
    \]
	    because $\delta\leq\nfrac12$ implies $\log\nfrac1\delta=\Omega(1)$ and $\log\nfrac1\eps\leq \nfrac1\eps$.
    Therefore the membership-query complexity is $m\cdot O\!\inparen{(\nfrac{1}{\eps})\log\nfrac1\delta}.$
        
    \paragraph{Example-oracle complexity.}
    At scale $s$, the residual-mass test uses
        $A_s
        =
        O\!\inparen{(\nfrac{1}{\tau_s})\, \log{\nfrac{1}{\eta_s}}}$
    samples from the original example oracle.
    The residual-oracle simulation uses at most
        $B_s
        =
        O\!\inparen{(\nfrac{m}{\tau_s})\, \log{\nfrac{1}{\eta_s}}}$
    additional samples from the original example oracle.
    Since $m\geq 1$, the total number of original example-oracle calls is
        $O\!\inparen{m\sum_{s=0}^T (\nfrac{1}{\tau_s})\cdot \log{\nfrac{1}{\eta_s}}}.$
    Now
    \[
        \sum_{s=0}^T \frac{\log{\nfrac{1}{\eta_s}}}{\tau_s}
        =
        \sum_{s=0}^T w_s\log\frac{100W}{\delta w_s}
        =
        W\log\frac1\delta
        +
        \sum_{s=0}^T w_s\log\frac{100W}{w_s}.
    \]
    Also, $\frac{W}{w_s}
	        =
	        \sum_{j=0}^{T} 3^{j-s}
	        =
	        O\!\inparen{3^{T-s}+1},$
    so
    \[
        \sum_{s=0}^T w_s\log\frac{100W}{w_s}
        =
        O\!\inparen{\sum_{s=0}^T 3^s\inparen{T-s+1}}
        =
        O\!\inparen{W}.
    \]
	    Because $\delta\leq\nfrac12$, this is $O\!\inparen{W\log\nfrac1\delta}.$
    Therefore,
    \[
        \sum_{s=0}^T \frac{\log{\nfrac{1}{\eta_s}}}{\tau_s}
        =
        O\!\inparen{W\log\frac1\delta}
        =
        O\!\inparen{\frac{1}{\eps}\log\frac1\delta}.
    \]
    It follows that the total number of original example-oracle calls is $O\!\inparen{(\nfrac{m}{\eps})\log\nfrac1\delta}.$
    Combining this with the membership-query complexity proves the claimed sample/query complexity.
\end{proof} 

\end{document}